\documentclass[nohyperref]{article}

\usepackage{microtype}
\usepackage{graphicx}
\usepackage{booktabs} %
\usepackage{subcaption}

\usepackage{hyperref}

 \usepackage[accepted]{icml2022}

\usepackage{amsmath}
\usepackage{amssymb}
\usepackage{mathtools}
\usepackage{amsthm}

\usepackage[capitalize,noabbrev]{cleveref}

\theoremstyle{plain}
\newtheorem{theorem}{Theorem}[section]

\newtheorem{lemma}[theorem]{Lemma}
\newtheorem{corollary}[theorem]{Corollary}
\theoremstyle{definition}
\newtheorem{definition}[theorem]{Definition}

\newtheorem{remark}[theorem]{Remark}

\usepackage{mkolar_definitions}
\usepackage{dsfont}
\usepackage{array}
\usepackage{xspace}
\usepackage{algorithmic}
\usepackage{enumitem}
\usepackage{commath}

\newcommand{\robustaggTS}{\ensuremath{\textsc{RobustAgg-TS}}\xspace}
\newcommand{\robustagg}{\ensuremath{\textsc{RobustAgg}}\xspace}
\newcommand{\inducb}{\ensuremath{\textsc{Ind-UCB}}\xspace}
\newcommand{\indts}{\ensuremath{\textsc{Ind-TS}}\xspace}
\newcommand{\robustaggTSV}{\ensuremath{\textsc{RobustAgg-TS-V}}\xspace}

\newcommand{\aggmu}{\ensuremath{\mathrm{agg}\text{-}\hat{\mu}\xspace}}
\newcommand{\aggvar}{\ensuremath{\mathrm{agg}\text{-}\mathrm{var}\xspace}}
\newcommand{\indmu}{\ensuremath{\mathrm{ind}\text{-}\hat{\mu}\xspace}}
\newcommand{\indvar}{\ensuremath{\mathrm{ind}\text{-}\mathrm{var}\xspace}}

\newcommand{\Reg}{\mathrm{Reg}}

\newcommand{\wbar}[1]{ {\ensuremath{\overline{#1}}} }

\newcommand{\indic}{\mathds{1}}

\newcommand{\UCB}{\mathrm{UCB}}
\newcommand{\var}{\mathrm{var}}
\newcommand{\rounds}{P}
\newcommand{\bPhi}{\wbar{\Phi}}
\newcommand{\optarm}{\diamond} %

\renewcommand{\algorithmiccomment}[1]{\hfill \textcolor{blue}{$\triangleright$ #1}}

\usepackage[textsize=tiny]{todonotes}

\icmltitlerunning{Thompson Sampling for Robust Transfer in Multi-Task Bandits}

\makeatletter
\newcommand{\svast}{\bBigg@{3}}
\newcommand{\vast}{\bBigg@{4}}
\newcommand{\Vast}{\bBigg@{5}}
\makeatother

\begin{document}

\twocolumn[
\icmltitle{Thompson Sampling for Robust Transfer in Multi-Task Bandits}

\icmlsetsymbol{equal}{*}

\begin{icmlauthorlist}
\icmlauthor{Zhi Wang}{ucsd}
\icmlauthor{Chicheng Zhang}{ua}
\icmlauthor{Kamalika Chaudhuri}{ucsd,fair}
\end{icmlauthorlist}

\icmlaffiliation{ucsd}{University of California San Diego}
\icmlaffiliation{ua}{University of Arizona}
\icmlaffiliation{fair}{Facebook AI Research}

\icmlcorrespondingauthor{Zhi Wang}{zhiwang@ucsd.edu}

\icmlkeywords{Machine Learning, ICML}

\vskip 0.3in
]

\printAffiliationsAndNotice{}  %

\begin{abstract}

We study the problem of online multi-task learning where the tasks are performed within similar but not necessarily identical multi-armed bandit environments. In particular, we study how a learner can improve its overall performance across multiple related tasks through robust transfer of knowledge. While an upper confidence bound (UCB)-based algorithm has recently been shown to achieve nearly-optimal performance guarantees in a setting where all tasks are solved concurrently, it remains unclear whether Thompson sampling (TS) algorithms, which have superior empirical performance in general, share similar theoretical properties. In this work, we present a TS-type algorithm for a more general online multi-task learning protocol, which extends the concurrent setting. We provide its frequentist analysis and prove that it is also nearly-optimal using a novel concentration inequality for multi-task data aggregation at random stopping times. Finally, we evaluate the algorithm on synthetic data and show that the TS-type algorithm enjoys superior empirical performance in comparison with the UCB-based algorithm and a baseline algorithm that performs TS for each individual task without transfer.

\end{abstract}

\section{Introduction}
\label{introduction}
We study multi-task transfer learning in a multi-armed bandit (MAB) setting.
In practice, auxiliary data from different but related sources are often available, although it is also often less clear how they should be utilized. 
If properly managed, such data can serve an important role in accelerating learning; in particular, in online learning, auxiliary data may be used to avoid costs associated with unnecessary exploration.
In this work, we study how data collected from similar sources can be robustly aggregated and utilized.

We consider a generalization of the $\epsilon$-multi-player multi-armed bandit ($\epsilon$-MPMAB) problem recently proposed by \citet{wzsrc21}, which can be used to model multi-task bandits. In the $\epsilon$-MPMAB problem\footnote{We shall still refer to the generalized problem as the $\epsilon$-MPMAB problem.}, a set of players sequentially and potentially concurrently interact with a common set of arms that have player-dependent reward distributions.
Each player and its associated reward distributions (data sources) are thereby regarded as a task.
Furthermore, we consider the reward distributions that the players face for each arm to be {\em similar} but not necessarily identical, and the level of (dis)similarity is specified by a parameter $\epsilon \in [0,1]$.

The $\epsilon$-MPMAB problem can be used to model important real-world applications. 
For example, in healthcare robotics, a set of robots, which correspond to players, can be paired with people with dementia to provide personalized cognitive training and wellness activities \citep{kubota2020jessie}.
Each training/wellness activity corresponds to an arm in the $\epsilon$-MPMAB problem, and people with similar preferences or symptoms may exhibit similar interests or needs---this is modeled via similarity in reward
distributions of each arm \cite{wzsrc21}. 
Another example can be seen in recommendation systems where learning agents are assigned to people within a social network, who may have similar interests due to inter-network influence \citep{qian2013personalized}.

Despite the similarity in its reward distributions, the $\epsilon$-MPMAB problem is still challenging for two reasons:
on the one hand, misusing auxiliary data can lead to {negative transfer} and substantially impair a player's performance \citep{rosenstein2005transfer};
on the other hand, while auxiliary data are often immediately accessible in their entirety in offline transfer learning settings, in the $\epsilon$-MPMAB problem, the available auxiliary data grow in time and depend on the interactions between the players and the environments.

An upper confidence bound (UCB)-based algorithm, $\robustagg(\epsilon)$, has been proposed for the $\epsilon$-MPMAB problem \citep{wzsrc21}. It achieves strong, near-optimal theoretical guarantees through robust data aggregation. Nevertheless, $\robustagg(\epsilon)$'s empirical performance can, unfortunately, be underwhelming.

Meanwhile, Thompson sampling (TS) algorithms \citep{thompson1933likelihood}, another family of bandit algorithms, have been shown superior empirically in comparison with UCB-based algorithms in standard single-task settings ~\citep[e.g.,][]{chapelle2011empirical}.
In fact, we show in Section~\ref{sec:experiments} that, for the $\epsilon$-MPMAB problem, a baseline algorithm which employs TS for each task individually without transfer learning can outperform $\robustagg(\epsilon)$ in many cases.

In spite of the encouraging signs from the empirical evaluations, the theoretical study of TS have lagged behind, especially in terms of {\em frequentist} analyses~\cite{an17,kaufmann2012thompson} for data aggregation and transfer learning in the multi-task setting\footnote{See Section~\ref{sec:related_work} for a discussion on related work.}.
It is therefore imperative to design multi-task TS-type algorithms that have superior empirical performance {\em and} strong theoretical guarantees. Our contributions in this work are:

\begin{enumerate}
    \item Inspired by prior works \citep{cgz13,glz14,hong2021hierarchical}, we generalize the $\epsilon$-MPMAB problem \citep{wzsrc21} to model a wider class of multi-task bandit learning scenarios so that it covers sequential and concurrent multi-task learning as special cases.

    \item We design a TS-type algorithm, $\robustaggTS(\epsilon)$, for the $\epsilon$-MPMAB problem and provide a frequentist analysis with near-optimal  performance guarantees.

    \item We empirically evaluate $\robustaggTS(\epsilon)$ on synthetic data and show that it outperforms the UCB-based $\robustagg(\epsilon)$ and a baseline algorithm that runs TS for each individual task without data sharing.

    \item Technical highlight: frequentist analyses of Thompson sampling can be much harder to conduct than those of UCB-based algorithms (see Remark~\ref{rem:compare-ucb}); a concentration inequality loose in logarithmic factors can result in a polynomial increase in regret guarantee (see Remark~\ref{rem:novel_concentration}). 
    To cope with this challenge, we prove a novel concentration inequality for multi-task data aggregation at random stopping times (Lemma~\ref{lem:novel_concentration_main_paper}), which leads to tight performance guarantees for $\robustaggTS(\epsilon)$. Our technique may be of independent interest for analyzing other multi-task sequential learning problems. 

\end{enumerate}

\section{Preliminaries}
\label{sec:background}

In this section, we first present the problem formulation and some important known results.
We then introduce a new baseline algorithm based on TS.
 
\paragraph{Notations.} Throughout, we use $[n]$ to denote the set $\cbr{1, 2, \ldots, n}$. Let $\Ncal(\mu, \sigma^2)$ denote the Gaussian distribution with mean $\mu$ and variance $\sigma^2$.
Let $a \vee b = \max(a,b)$. For a set $A \subseteq U$, denote by $A^C = U \setminus A$ the complement of $A$ in the universe $U$.
We use $\tilde{\Ocal}$ to hide logarithmic factors.

\subsection{Problem Formulation}
We consider and generalize the $\epsilon$-MPMAB problem introduced by \citet{wzsrc21}. 
An $\epsilon$-MPMAB problem instance comprises $M$ players, $K$ arms, and a dissimilarity parameter $\epsilon \in [0,1]$.
Let $[M]$ denote the set of players and $[K]$ the set of arms.
For each player $p \in [M]$ and each arm $i \in [K]$, there is an initially-unknown reward distribution $\Dcal_i^p$, which has support $[0,1]$ and has mean $\mu_i^p$. 

\paragraph{Reward dissimilarity.}
The reward distributions for each arm are assumed to be {\em similar but not necessarily identical} for different players; specifically,
\begin{align}
    \forall i \in [K], \ p,q \in [M], \quad \abr{\mu_i^p - \mu_i^q} \le \epsilon.
\end{align}

\paragraph{Protocol.} In the work of \citet{wzsrc21}, the players interact with the arms in rounds, and within each round, all players take an action concurrently. In this paper, inspired by the problem setup  of~\citet{hong2021hierarchical}, we generalize the interaction protocol such that it allows any subset of the players to take an action. In each round $t \in [T]$, where $T > \max(K,M)$ is the time horizon of learning, a subset of players $\Pcal_t \subseteq [M]$ is chosen (called the {\em active player set} at round $t$) by an oblivious adversary; each active player $p \in \Pcal_t$ then pulls an arm $i_t^p \in [K]$ and observes an independently-drawn reward $r_t^p \sim \Dcal_{i_t^p}^p$. At the end of round $t$, the active players communicate their decisions, $\cbr{i_t^p: p \in \Pcal_t}$, as well as their observed rewards, $\cbr{r_t^p: p \in \Pcal_t}$, to all players. Note that, when $\abr{\Pcal_t} = 1$ for all $t$, the problem setting resembles the one in \cite{cgz13} and captures a sequential transfer bandit learning setting \citep[e.g.,][]{alb13}; when $\Pcal_t = [M]$ for all $t$, we recover the setting in the work of \citet{wzsrc21}.

\paragraph{Performance metric.} The goal of the players is to minimize their expected collective regret, which we define shortly. 
For each player $p \in [M]$, let $\mu_*^p = \max_{j \in [K]} \mu_j^p$ denote the mean reward of an optimal arm for $p$; then, for each arm $i \in [K]$, let $\Delta_i^p = \mu_*^p - \mu_i^p \ge 0$ denote the (suboptimality) gap of arm $i$ for player $p$. 
In addition, let $n_i^p(t) = \sum_{s \leq t} \indic \cbr{p \in \Pcal_s, i_s^p = i}$ denote the number of pulls of arm $i$ by player $p$ after $t$ rounds. 
Then, the individual expected regret of any player $p$ is defined as 
\begin{align*}
\Reg^p(T)
= \EE \Bigg[ \sum_{  \substack{ t \in [T]: \\ p \in \Pcal_t}} 
 \mu_*^p - \mu_{i_t^p}^p  \Bigg]
= \sum_{i \in [K]} \EE \sbr{n_i^p(T)} \Delta_i^p. %
\end{align*}
Finally, the {\em expected collective regret} is defined as the sum of individual expected regret over all the players, i.e.,
\begin{align}
   \Reg(T) = \sum_{p \in [M]} \Reg^p(T) = \sum_{i \in [K]} \sum_{p \in [M]}  \EE \sbr{n_i^p(T)} \Delta_i^p. \label{eqn:regret_definition}
\end{align}

\paragraph{Does one need to know $\epsilon$?}
In this work, we focus on the case where $\epsilon$ is {\em known} to the players in the $\epsilon$-MPMAB problem. This is because \citet{wzsrc21} prove that, unfortunately, not much can be done when $\epsilon$ is unknown to the players---a lower bound (Theorem 11 therein) shows that no sublinear-regret algorithms can effectively take advantage of inter-task data aggregation for every $\epsilon \in [0,1]$ to achieve improved regret upper bounds.

\subsection{Existing Results}
\label{sec:existing-results}

In the concurrent setting ($\Pcal_t = [M]$ for all $t$), \citet{wzsrc21} show that, whether data aggregation can be provably beneficial for an arm $i$ depends on how its associated suboptimality gaps, $\Delta_i^p$'s, compare with the dissimilarity parameter, $\epsilon$. 

\paragraph{Subpar arms.} Specifically, the problem complexity is captured by a notion called {\em subpar arms}. The set of $\alpha$-subpar arms is defined as:
\begin{align}
    \Ical_{\alpha} = \cbr{i \in [K]:\ \exists p,\ \Delta_i^p > \alpha}.
    \label{eqn:alpha-subpar}
\end{align}

\paragraph{Regret guarantees.} The upper and lower bounds provided in \citep{wzsrc21} characterize that, informally, the collective performance of the players  can be improved by a factor of $M$ (resp. $\sqrt{M}$) for each $\Ocal(\epsilon)$-subpar arm in the (suboptimality) gap-dependent (resp. gap-independent) bounds, where we recall that $M$ is the number of players. 

This improvement is in comparison with baseline algorithms in which each player runs their own instance of a bandit algorithm individually. 
Let $\inducb$ be a baseline in which  each player runs the UCB-$1$ algorithm \citep{acf02}. Its collective regret guarantees are obtained by simply summing over individual gap-dependent and gap-independent regret bounds, respectively:
$\Ocal \rbr{
\sum_{p \in [M]}
\sum_{i \in [K]: \Delta_i^p > 0}
\frac{\ln T}{\Delta_i^p}}$ and $\tilde{\Ocal} \rbr{ M \sqrt{K T} }$.

In contrast, through leveraging auxiliary data from inter-player communication, the UCB-based algorithm, $\robustagg(\epsilon)$, proposed by \citet{wzsrc21} has  gap-dependent and gap-independent regret bounds of 
\[
\Ocal \svast(
\underbrace{\frac{1}{M} \sum_{i \in \Ical_{5\epsilon}} \sum_{ \substack{ p \in [M] \\ \Delta_i^p > 0}} \frac{\ln T}{\Delta_i^p}}_{(*)} +
\sum_{i \in \Ical^C_{5\epsilon}}
\sum_{\substack{ p \in [M] \\ \Delta_i^p > 0}}
\frac{\ln T}{\Delta_i^p} + MK \svast) \text{ and}
\]
\[ \tilde{\Ocal} \bigg( \underbrace{ \sqrt{ M |\Ical_{5\epsilon}| T } }_{(*)} + M \sqrt{ \rbr{ | \Ical_{5\epsilon}^C | - 1} T } + MK \bigg),
\]
respectively\footnote{The results may appear different from \citep{wzsrc21} at a glance because we use a slightly notation for subpar arms.}. 
These guarantees exhibit a factor of $\frac1M$ and $\frac1{\sqrt{M}}$ improvement in the respective $(*)$ terms, for the set of $\Ocal (\epsilon)$-subpar arms, $\Ical_{5\epsilon}$, and is nearly optimal.

In Appendix~\ref{sec:appendix_baselines}, we give a brief recap of $\robustagg(\epsilon)$---we show that with a few small modifications, it can be extended to work in the generalized $\epsilon$-MPMAB setting, and achieve generalized regret guarantees (see Theorem~\ref{thm:robustagg-reg}).

\paragraph{Lower bounds.} In the setting where the dissimilarity parameter $\epsilon$ is known, 
a lower bound in~\cite{wzsrc21} shows that, for any algorithm that has a sublinear-regret guarantee, when facing a large class of $\epsilon$-MPMAB problem instances, it must have regret at least
\begin{align*}
\Omega & \left(
\sum_{i \in \Ical_{\epsilon / {4}}: \Delta_i^{\min} > 0}
\frac{\ln T }{ \Delta_i^{\min}} 
+
\sum_{i \in \Ical_{\epsilon/4}^C}
\sum_{p \in [M]: \Delta_i^p > 0} 
\frac{\ln T }{ \Delta_i^p} 
\right),
\end{align*}
where $\Delta_i^{\min} = \min_{p \in [M]} \Delta_i^p$.
This lower bound shows that, 
data aggregation cannot be effective for the arm set $\Ical^C_{\epsilon/4} \subseteq \Ical^C_{\epsilon}$. 

In addition,~\citet{wzsrc21} also show a gap-independent lower bound:
for any algorithm, there exists an $\epsilon$-MPMAB instance, in which the algorithm has regret at least
\[ \Omega \bigg( \sqrt{ M |\Ical_{5\epsilon}| T }  + M \sqrt{ \rbr{ | \Ical_{5\epsilon}^C | - 1} T } \bigg),
\]
in the setting where $\Pcal_t = [M]$ for all $t \in [T]$.

\subsection{Baseline: $\indts$}
In this work, we consider another baseline algorithm, $\indts$, in which each player runs the standard TS algorithm with Gaussian priors.
We now describe the TS algorithm. 
At a high level, every learner (player) $p$ begins with some prior belief on the mean reward of each arm, and through interactions with the environment, the learner updates its posterior belief. Specifically, we consider TS with Gaussian product priors---a learner maintains one Gaussian posterior distribution for each arm, beginning with $\Ncal\rbr{0,1}$. In each round $t$, the learner draws an independent sample $\theta^p_i(t)$ for each arm $i$ from its corresponding posterior distribution, which is of form $\Ncal \rbr{\bar{\mu}^p_i, \frac{1}{n^p_i(t-1) 
\vee 1}}$, where $\bar{\mu}^p_i = \frac{1}{n^p_i(t-1) \vee 1} \sum_{s<t: p \in \Pcal_s, i_s^p = i} r_s^p$ is the empirical mean reward of player $p$ pulling arm $i$.
The learner then pulls the arm $i_t^p = \argmax_{i} \theta^p_i(t)$, receives a reward $r_t^p \sim \Dcal^p_{i_t^p}$, and updates the posterior distribution for arm $i$.

Based on the results of~\citet{an17}, we obtain the regret guarantees of $\indts$ by summing over individual bounds:
$\Ocal \rbr{
\sum_{p \in [M]}
\sum_{i \in [K]: \Delta_i^p > 0}
\frac{\ln T}{\Delta_i^p}}$
and
$\tilde{\Ocal} \rbr{ M \sqrt{K T} }$.

In Appendix~\ref{sec:appendix_baselines}, we briefly recap the guarantees of $\inducb$ and $\indts$ in the generalized $\epsilon$-MPMAB setting, where $\Pcal_t$'s are not necessarily $[M]$ in every round.

\begin{algorithm}[b!]
\begin{algorithmic}[1]
    \STATE {\bfseries Input:} Dissimilarity parameter $\epsilon \in [0,1]$, universal constants $c_1, c_2 > 0$.
    \STATE {\bfseries Initialization:} For every $i \in [K]$ and $p \in [M]$, set $n_i^p = 0$, $\indmu_i^p = 0$, $\indvar_i^p = c_2$, $\aggmu_i^p = 0$, and $\aggvar_i^p = c_2$; for every $i \in [K]$, set $n_i = 0$.
   
    \FOR{round $t \in [T]$}
    
        \STATE Receive active set of players $\Pcal_t$.
    
        \FOR{active player $p \in \Pcal_t$}
            \FOR{arm $i \in [K]$}

                \IF{$n_i^p \ge \frac{c_1 \ln T}{\epsilon^2} + 2 M$}
                \label{line:choose-posterior-start}
                
                \STATE 
                $\hat{\mu}_i^p \gets \indmu_i^p$, $\var_i^p \gets \indvar_i^p$;
                \label{line:posterior_ind} 
                \\ \COMMENT{Use the individual posterior}
                
                \ELSE
                \STATE $\hat{\mu}_i^p \gets \aggmu_i^p$, $\var_i^p \gets \aggvar_i^p$;
                \label{line:posterior_agg}
                \\ \COMMENT{Use the aggregate posterior}
                \ENDIF
                \label{line:choose-posterior-end}
               
               \STATE $\theta_i^p(t) \sim \Ncal( \hat{\mu}_i^p, \var_i^p )$
               \label{line:draw-sample}
            \ENDFOR
            
            \STATE Player $p$ pulls arm $i_t^p = \argmax_{i \in [K]} \theta_i^p(t)$ and observes reward $r_t^p$.
            \label{line:choose-action}
        \ENDFOR
        
        \FOR{active player $p \in \Pcal_t$}
            \STATE Let $i = i_t^p$. Update $n_i^p \gets n_{i}^p + 1$ and $n_{i} \gets n_{i} + 1$.
        \ENDFOR
        
        \FOR{active player $p \in \Pcal_t$} 
            \STATE Let $i = i_t^p$. 
            \\ \COMMENT{{\bfseries Only} update posteriors  associated with $p$ and $i_t^p$} 
            \label{line:invariant}
            
            \STATE Update 
            \begin{align*}
               & \indmu_i^p \gets  \frac{1}{ n_i^p \vee 1} \sum_{s \le t } \indic \cbr{p \in \Pcal_s, i_s^p = i} r_s^p, 
               \\ 
               & \indvar_i^p \gets \frac{c_2}{ n_i^p \vee 1};
            \end{align*}
            \label{line:compute_ind}
            \vspace{-15pt}
            
            \STATE 
            \begin{align*}
               & \aggmu_i^p \gets \frac{1}{ n_i \vee 1} \sum_{s \le t} \sum_{q \in \Pcal_s} \indic \cbr{i_s^q = i} r_s^q +  \epsilon, \\
               & \aggvar_i^p \gets \frac{c_2}{\rbr{n_i - M} \vee 1}.
            \end{align*}
            \label{line:compute_agg}
            \vspace{-10pt}
        \ENDFOR
    \ENDFOR
\end{algorithmic}
\caption{\robustaggTS$(\epsilon)$}
\label{alg:robustaggTS}
\end{algorithm}

\section{Algorithm}

In this section, we present a TS-type randomized exploration algorithm, $\robustaggTS(\epsilon)$ (Algorithm~\ref{alg:robustaggTS}), which can robustly leverage data collected by all the players.

In each round $t$, for each active player $p \in \Pcal_t$ and arm $i$, $\robustaggTS(\epsilon)$ maintains two Gaussian ``posterior'' distributions.
As a standard single-task TS algorithm with Gaussian priors would normally maintain~\citep[e.g.][]{an17}, $\Ncal \rbr{\indmu_i^p, \indvar_i^p}$, the \emph{individual posterior} is solely based on player $p$'s own interactions with arm $i$, with $\indmu_i^p$ and $\indvar_i^p$ defined in line~\ref{line:compute_ind}.
In contrast, the \emph{aggregate posterior}, $\Ncal \rbr{ \aggmu_i^p, \aggvar_i^p}$, is unique to the multi-task setting---its mean, $\aggmu_i^p$, is the sum of the empirical mean of all players' observed rewards for arm $i$ and a bonus term $\epsilon$, and its variance, $\aggvar_i^p$, is based on the total number of pulls of arm $i$ by all players (line~\ref{line:compute_agg}).

The algorithm chooses one of the posterior distributions (lines~\ref{line:choose-posterior-start} to~\ref{line:choose-posterior-end}), i.e., decides whether to utilize data shared by other players, by balancing a bias-variance trade-off \citep{bbckpv10,soaremulti,wzsrc21}: while an inclusion of $n_i$ reward samples collected by all players leads to a variance, $\aggvar_i^p$, which can be much smaller than $\indvar_i^p$, it may also cause $\aggmu_i^p$ to be biased as the reward distributions for different players may be different. The algorithm then independently draws a sample, $\theta_i^p(t)$, from the chosen posterior distribution (line~\ref{line:draw-sample}) and pulls the arm with the largest $\theta_i^p(t)$ for player $p$ (line~\ref{line:choose-action}).

Specifically, in round $t$, for player $p \in \Pcal_t$ and arm $i \in [K]$, the algorithm chooses a posterior distribution by comparing $n_i^p$, the number of pulls of $i$ by $p$ at the beginning of round $t$, to a threshold in terms of the dissimilarity parameter, i.e., $\frac{c_1 \ln T}{\epsilon^2} + 2 M$ (line~\ref{line:choose-posterior-start}), where $c_1 > 0$ is some numerical constant. Intuitively, when $\epsilon$ is smaller, each player stays longer on using the aggregate posterior to perform randomized exploration, which indicates a higher degree of trust on data from other tasks.

After all players in $\Pcal_t$ obtain rewards for their arm pulls, they compute and {\em update} their posteriors with new data. 
In principle, data from one player can affect the aggregate posteriors of all players. 
We make the design choice that 
this effect gets delayed: 
the algorithm only updates the posteriors for player $p$ and arm $i$ in round $t$, if $p \in \Pcal_t$ and $i = i_t^p$ (line~\ref{line:invariant}). Although our current analysis (see Sections~\ref{sec:main-results} and~\ref{sec:pf-ingr} below) relies on this property to establish sharp regret guarantees, we conjecture that similar regret guarantees can be shown even if the algorithm updates the posteriors of all players and all arms in every round\footnote{In Section~\ref{sec:comp_v} of the appendix, we show that this variation induces little effect on the empirical performance of the algorithm.}.

\section{Main Results}
\label{sec:main-results}
We now present gap-dependent and gap-independent regret upper bounds of $\robustaggTS(\epsilon)$.
Recall that $\Ical_{\alpha} = \{i \in [K]: \exists p,\ \Delta_i^p > \alpha \}$ is the set of $\alpha$-subpar arms.

\begin{theorem}[Gap-dependent bound]
\label{thm:ts_gap_dep_ub}
There exists a setting of $c_1, c_2 > 0$, such that,
the expected collective regret of $\robustaggTS(\epsilon)$ after $T > \max(K,M)$ rounds satisfies: $\Reg(T) \le $
\begin{align*}
    \Ocal \svast( & \frac 1 M \sum_{i \in \Ical_{10\epsilon}} \sum_{ \substack{p \in [M] \\ \Delta_i^p > 0}} \frac{\ln T}{\Delta_i^p}
    +
    \sum_{i \in \Ical_{10\epsilon}^C} \sum_{\substack{p \in [M] \\ \Delta_i^p > 0}} \frac{\ln T}{\Delta_i^p} 
    + M^2K \svast).
\end{align*}
\end{theorem}

\begin{theorem}[Gap-independent bound]
\label{thm:ts_gap_indep_ub}
There exists a setting of $c_1, c_2 > 0$, such that, the expected collective regret of $\robustaggTS(\epsilon)$ after $T > \max(K,M)$ rounds satisfies:
\begin{align*}
    \Reg(T) \le & \tilde{\Ocal} \bigg( \sqrt{ |\Ical_{10\epsilon}| P } + \sqrt{ M \rbr{|\Ical_{10\epsilon}^C| -1 } P } + M^2K \bigg), 
\end{align*}
where $P = \sum_{t=1}^T \abr{ \Pcal_t }$.
\end{theorem}
The proofs of Theorems~\ref{thm:ts_gap_dep_ub} and~\ref{thm:ts_gap_indep_ub} can be found in Appendix~\ref{sec:appendix_main_proofs}; in Section~\ref{sec:pf-ingr}, we also highlight several technical challenges and proof ingredients in our analysis.

\paragraph{Guarantees in the generalized $\epsilon$-MPMAB setting.} Our guarantees for $\robustaggTS(\epsilon)$ hold under the generalized $\epsilon$-MPMAB setting, in that $\Pcal_t$'s at each round can change over time. 
Observe that the regret bound given by Theorem~\ref{thm:ts_gap_dep_ub} does not depend on $\Pcal_t$'s, and the regret bound given by Theorem~\ref{thm:ts_gap_indep_ub} has the highest value when $P = MT$.
In addition, recall that near-matching gap-dependent and gap-independent lower bounds have been shown by~\citet{wzsrc21} in the $\Pcal_t \equiv [M]$ setting (Section~\ref{sec:existing-results}). These lower bounds indicate the near-optimality of $\robustaggTS(\epsilon)$'s guarantees, modulo an additive lower-order term $O(M^2 K)$ which does not depend on $T$.

Furthermore, the gap-independent guarantee in Theorem~\ref{thm:ts_gap_indep_ub} adapts to the value of $\rounds$. This shows the flexibility of $\robustaggTS(\epsilon)$. 
Specifically, if $\abr{\Pcal_t} = 1$ (similar to the settings of~\citealt{cgz13,glz14}), we have $\rounds = T$, and $\Reg(T) \le$
\begin{align*}
     \tilde{\Ocal} \rbr{  \sqrt{ |\Ical_{10\epsilon}| T } + \sqrt{ M \rbr{|\Ical_{10\epsilon}^C| - 1} T } +M^2K}.
\end{align*}
Similarly, if $\Pcal_t = [M]$ for all $t$~\citep{wzsrc21}, then $\rounds = MT$, and $\Reg(T) \le$
\begin{align*}
    \tilde{\Ocal} \rbr{  \sqrt{ M |\Ical_{10\epsilon}| T } + M \sqrt{ \rbr{|\Ical_{10\epsilon}^C| - 1} T }  + M^2K}.
\end{align*}

\paragraph{Comparison with baselines.} In comparison with the guarantees of the UCB-based algorithm $\robustagg(\epsilon)$ in Appendix~\ref{sec:robustagg-ucb}, we see that $\robustaggTS(\epsilon)$ has competitive guarantees, except that the set of arms which benefits from data aggregation changes from $\Ical_{5\epsilon}$ to $\Ical_{10\epsilon}$.

In comparison with the guarantees of \inducb and \indts,
the regret guarantees of $\robustaggTS(\epsilon)$ are never worse (modulo lower-order terms), and save factors of $\frac{1}{M}$ and $\frac{1}{\sqrt{M}}$ in $\Ical_{10\epsilon}$'s contribution in the gap-dependent and gap-independent regret guarantees, respectively.

\section{Proof Ingredients}
\label{sec:pf-ingr}

In this section, we highlight some of the novel proof ingredients used in our analysis of Algorithm~\ref{alg:robustaggTS}, which are unique to the {\em multi-task} setting\footnote{Our analysis involves various proofs by cases. Figure~\ref{fig:case_division} in the appendix provides an overview illustrating the case division rules used in our proofs.}. 

We begin by decomposing the regret in terms of subpar arms and non-subpar arms. 
It follows from Eq.~\eqref{eqn:regret_definition} that
\begin{align}
\Reg(T) 
=\order \svast( \sum_{i \in \Ical_{10\epsilon}} \EE & \sbr{n_i(T)} \Delta_i^{\min} +  \nonumber \\
& \sum_{i \in \Ical^C_{10\epsilon}} \sum_{p \in [M]} \EE \sbr{n_i^p(T)} \Delta_i^p \svast), \nonumber
\end{align}
where we 
let $n_i(T) = \sum_{p=1}^M n_i^p(T)$ be the number of pulls of arm $i$ by all players after $T$ rounds;
we recall that $\Delta_i^{\min} = \min_{p \in [M]} \Delta_i^p$;
and we use the fact that for any subpar arm $i \in \Ical_{10\epsilon}$ and any player $p \in [M]$, $\Delta_i^p \le 2 \Delta_i^{\min}$ (Fact~\ref{fact:basic_facts_subpar_arms}).

In the interest of space, we focus on the analysis for subpar arms and defer the discussion on non-subpar arms to the appendix. The following lemma provides an upper bound on $\EE \sbr{n_i(T)}$ for $i \in \Ical_{10\epsilon}$, which can be subsequently used to derive the upper bounds on the expected collective regret incurred by the $10\epsilon$-subpar arms in Section~\ref{sec:main-results}.

\begin{lemma}
\label{lem:key_subpar_arms}
For any arm $i \in \Ical_{10\epsilon}$,
\[
\EE \sbr{n_i(T)} \le \Ocal \rbr{\frac{\ln T}{(\Delta_i^{\min})^2} + M}.
\]
\end{lemma}

While a similar lemma can be found in \citep[Lemma 20]{wzsrc21} for the UCB-based algorithm, $\robustagg(\epsilon)$, proving Lemma~\ref{lem:key_subpar_arms} requires new ingredients that we present in the rest of this section.

Let us fix an arm $i \in \Ical_{10\epsilon}$. To control $\EE \sbr{n_i(T)} = \EE \sbr{\sum_{t \in [T]} \sum_{p \in \Pcal_t}  \indic \cbr{i_t^p = i}}$, we begin by generalizing a technique introduced by~\citet{an17} for standard TS to the multi-task setting.
In each round $t$ and for each active player $p$, we consider two cases: (1) player $p$ pulls arm $i$ (namely, $i_t^p = i$), and $\theta_i^p(t)$ (line~\ref{line:draw-sample} in Algorithm~\ref{alg:robustaggTS}) is greater than some threshold $y_i^p \in (\mu_i^p, \mu_*^p)$ to be defined shortly, and (2) $i_t^p = i$ and $\theta_i^p(t) \le y_i^p$.
We have
\begin{align*}
\EE \sbr{n_i(T)} = \underbrace{\EE \sbr{\sum_{t \in [T]} \sum_{p \in \Pcal_t} \indic \cbr{i_t^p = i, \theta_i^p(t) > y_i^p, \Ecal_t}}}_{(A)} \\ 
+ \underbrace{\EE \sbr{\sum_{t \in [T]} \sum_{p \in \Pcal_t} \indic \cbr{i_t^p = i, \theta_i^p(t) \le y_i^p, \Ecal_t}}}_{(B)} + \Ocal \rbr{1},
\end{align*}
where $\Ecal_t$, informally, is a high-probability ``clean'' event in which $\hat{\mu}_i^p$'s maintained by Algorithm~\ref{alg:robustaggTS} in round $t$ for each $i$ and $p$ concentrate towards their respective expected values.

Term $(A)$ can be controlled because, as more pulls of arm $i$ are made, $\cbr{\theta_i^p(t) > y_i^p}$ is unlikely to happen, as $\hat{\mu}_i^p$ concentrates towards a value smaller than $y_i^p$, and $\var_i^p$ decreases. See Lemma~\ref{lem:bounding_term_A_appendix}
in the appendix for a detailed proof.

In what follows, we focus on bounding term $(B)$. 
Observe that the event $\cbr{i_t^p = i, \theta_i^p(t) \le y_i^p}$ in $(B)$ happens only if $\forall j \in [K]$, $\theta_j^p(t) \le y_i^p$, including the optimal arm(s) for player $p$. 
Since in an $\epsilon$-MPMAB problem instance, different players may have different optimal arms, we consider a common near-optimal arm $\dagger \in \Ical_{2\epsilon}^C$---see Fact~\ref{fact:basic_facts_subpar_arms} in the appendix for the existence of such an arm. It can be easily verified that, for any arm $i \in \Ical_{10\epsilon}$ and player $p \in [M]$, $\delta_i^p := \mu_{\dagger}^p - \mu_i^p > 0$ (see Fact~\ref{fact:small_delta}). In other words, while $\dagger$ may not necessarily be an optimal arm for every player, it has a larger mean reward than any $i \in \Ical_{10\epsilon}$.
We can now define
$y_i^p := \mu_i^p + \frac12 \delta_i^p \in (\mu_i^p, \mu_\dagger^p)
\subset 
(\mu_i^p, \mu_*^p)$.

Using a technique first introduced in~\citep{an17}, we will show that $\theta_{\dagger}^p(t)$ converges to a value greater than $y_i^p$ {\em fast} enough so that $\cbr{\forall j \in [K], \theta_j^p(t) \le y_i^p}$ will unlikely happen soon enough and thus $(B)$ can be controlled.

\begin{remark}[Comparison with UCB-based analyses]
\label{rem:compare-ucb}
We note that controlling term $(B)$
is often not required in the analyses of UCB-based algorithms. Colloquially, this term concerns the event in which arm $i$ is pulled even when its sample/index value is smaller than $y_i^p$; such an event would unlikely happen for UCB-based algorithms as the {\em optimism in the face of uncertainty} principle ensures that, with high probability, the UCB index of an optimal arm for player $p$ is greater than or equal to $\mu_*^p \ge \mu_{\dagger}^p > y_i^p$. 
\end{remark}

Before we formalize the above-mentioned intuition for bounding term $(B)$ in Lemma~\ref{lem:bounding_term_B}, we first lay out a few helpful definitions. We define $\cbr{\Fcal_t}_{t=0}^{T}$ to be a filtration
such that
$\Fcal_t = \sigma \rbr{ \cbr{i_s^q, r_s^q: s\le t, q \in \Pcal_{s}}}$
is the $\sigma$-algebra generated by interactions of all players up until round $t$. 
Then, let
$
\phi_{i,t}^p = \Pr \rbr{\theta_{\dagger}^p(t) > y_i^p \mid \Fcal_{t-1}}.
$
{
Observe that if $\phi_{i,t}^p$ is large, the event $\cbr{i_t^p = i, \theta_i^p(t) \le y_i^p}$ will unlikely happen.
}

\begin{lemma}
\label{lem:bounding_term_B}
\begin{align*}
(B) \le \underbrace{\sum_{t \in [T]} \sum_{p \in \Pcal_t} \EE \sbr{  {\rbr{\frac{1}{\phi_{i,t}^p} - 1} \indic \cbr{i_t^p = \dagger, \Ecal_t} }}}_{(B*)}.
\end{align*}
\end{lemma}
See Lemma~\ref{lem:B*} and its proof in the appendix for details.
We now consider the following two cases:
in any round $t$ and for any active player $p$ that pulls arm $\dagger$, i.e., $i_t^p = \dagger$, $p$ uses \emph{either} the individual \emph{or} the aggregate posterior distribution associated with arm $\dagger$ (lines~\ref{line:choose-posterior-start} to \ref{line:choose-posterior-end} in Algorithm~\ref{alg:robustaggTS}). Let $H_{\dagger}^p(t)$ be the event that $p$ uses the individual posterior distribution and $\wbar{H_{\dagger}^p(t)}$ be the event that $p$ uses the aggregate posterior (see Definition~\ref{def:alg_criteron} in the appendix for the formal definitions).
We can then decompose $(B*)$ as follows:
\begin{align*}
(B*) =
\underbrace{\sum_{t \in [T]} \sum_{p \in \Pcal_t} \EE \sbr{ {\rbr{\frac{1}{\phi_{i,t}^p}-1} \indic \cbr{i_t^p = \dagger, \Ecal_t, H_{\dagger}^p(t)} }}}_{(b1)} \\
+
\underbrace{\sum_{t \in [T]} \sum_{p \in \Pcal_t} \EE \sbr{ {\rbr{\frac{1}{\phi_{i,t}^p}-1} \indic \cbr{i_t^p = \dagger, \Ecal_t, \wbar{H_{\dagger}^p(t)}} }}}_{(b2)}.
\end{align*}

Let $m_{\dagger}^p(t)$ denote the aggregate number of pulls of arm $\dagger$ maintained by player $p$ after $t$ rounds (see Definition~\ref{def:m_i^p} in the appendix). Note that, by the design choice of Algorithm~\ref{alg:robustaggTS} (line~\ref{line:invariant}), $m_{\dagger}^p(t)$ is not necessarily the same as $n_\dagger(t)$. With foresight, let $L = \Theta \rbr{\frac{\ln T}{(\Delta_i^{\min})^2} + M}$, and let $G_t^p = \cbr{i_t^p = \dagger, \Ecal_t, \wbar{H_{\dagger}^p(t)}}$. We have
\begin{align*}
& (b2) = \\
& \underbrace{\sum_{t \in [T]} \sum_{p \in \Pcal_t} \EE \sbr{ {\rbr{\frac{1 }{\phi_{i,t}^p} - 1} \indic \cbr{G_t^p, m_{\dagger}^p(t-1) < L} }}}_{(b2.1)} 
\\
& + \underbrace{\sum_{t \in [T]} \sum_{p \in \Pcal_t} \EE \sbr{ {\rbr{\frac{1}{\phi_{i,t}^p} - 1} \indic \cbr{G_t^p, m_{\dagger}^p(t-1) \geq L} }}}_{(b2.2)}.
\end{align*}

Both $(b1)$ and $(b2.2)$ can be bounded by $\Ocal \rbr{M}$, because, informally speaking, either player $p$ has pulled arm $\dagger$ many times when the individual posterior is used (term $(b1)$) or the players collectively have pulled $\dagger$ many times when the aggregate posterior is used (term $(b2.2)$), and $\frac{1}{\phi_{i,t}^p} - 1$ can therefore be upper bounded by $\frac{1}{T}$. See Lemma~\ref{lem:b1} and Lemma~\ref{lem:b22} and their proofs for details.

The main challenge in bounding $\EE\sbr{n_i(T)}$ lies in term $(b2.1)$, for which we show the following lemma.
\begin{lemma}[Bounding term $(b2.1)$]
\label{lem:bounding_term_b21}
\begin{align*}
(b2.1) \le \Ocal \rbr{L} \le \Ocal \rbr{\frac{\ln T}{(\Delta_i^{\min})^2} + M}.
\end{align*}
\end{lemma}

Proving Lemma~\ref{lem:bounding_term_b21} is {\em central} to our analysis and as we will see, requires special care. We begin by introducing the following notion.
For any arm $j \in [K]$ and $k \in [TM]$, 
let
\[
\tau_k(j) = \min \cbr {T+1, \min \cbr{t: n_j(t) \ge k}}
\]
be the round in which arm $j$ is pulled the $k$-th time by any player. Furthermore, let $\tau_0(j) = 0$ by convention.
For any $j \in [K]$ and $k \in [TM]$, it is easy to verify that $\tau_k(j)$ is a stopping time with respect to $\cbr{\Fcal_t}_{t=0}^T$.
In what follows, when circumstances permit, we abuse the notation and denote $\tau_k(\dagger)$ by $\tau_k$.

\paragraph{Invariant property.} By the construction of Algorithm~\ref{alg:robustaggTS}, in any round $t$, a player only updates the posteriors associated with an arm if the player pulls the arm in the round $t$ (line~\ref{line:invariant}). This design choice induces an invariant property: for any arm and player, certain random variables associated with them stay invariant between consecutive pulls of the arm by the player (see Definition~\ref{def:invariant} and a few examples in the appendix).

The invariant property allows us to bound $(b2.1)$ as follows in terms of the stopping times $\tau_k$'s (See Lemma~\ref{lem:b21} and Lemma~\ref{lem:auxiliary_sum_f_indic} in the appendix):
\begin{align*}
(b2.1) 
& \le \sum_{p=1}^M \EE \sbr{ \rbr{\frac{1}{\phi_{i,1}^p} - 1} \indic \cbr{\wbar{H_{\dagger}^p(1)}}} + \\
& \sum_{k=1}^{L-1} \EE \sbr{ \rbr{\frac{1}{\phi_{i, \tau_k+1}^{p_k}} - 1} \indic \cbr{\tau_k \le T, \wbar{H_{\dagger}^p(\tau_k+1)}}},
\end{align*}
where $p_k := p_k(\dagger)$ is the player that makes the $k$-th pull of arm $\dagger$ (Definition~\ref{def:p_k}).

Using basic Gaussian tail bounds, we can show that $\EE \sbr{\rbr{\frac{1}{\phi_{i,1}^p} - 1}\indic \cbr{\wbar{H_{\dagger}^p(1)}}} \le \Ocal \rbr{1}$ for any player $p$. Then, the following lemma suffices to prove Lemma~\ref{lem:bounding_term_b21}.

\begin{lemma}
\label{lem:key_lemma_constant}
For any $k \in [TM]$,
\[
\EE \sbr{\rbr{\frac{1}{\phi_{i, \tau_k+1}^{p_k}} - 1} \indic \cbr{\tau_k \le T, \wbar{H_{\dagger}^p(\tau_k+1)}}} \le \Ocal \rbr{1}.
\]
\end{lemma}

\paragraph{Technical highlight.} Lemma~\ref{lem:key_lemma_constant} generalizes \citet[Lemma 2.13]{an17} for standard TS to the {\em multi-task} setting.
A complete proof can be found in the appendix,
which uses anti-concentration bounds of Gaussian random variables \citep{gordon1941millsratio} as well as a {\em novel} concentration inequality for multi-task data aggregation at random stopping times $\tau_k(\dagger)$'s, which we highlight here\footnote{In the single-task case ($M=1$), our proof technique (Lemma~\ref{lem:sg-anticonc}) also simplifies the proof of the first case of~\citet[Lemma 2.13]{an17}.}.
For any arm $j$, let
\[
\aggmu_j(t) = \frac{1}{n_j(t) \vee 1} \sum_{s \le t} \sum_{ q \in \Pcal_s } \indic \cbr{i_s^q = j} r_s^q +  \epsilon
\]
be the aggregate mean reward estimate of $j$ constructed using data by all players after $t$ rounds, offset by $\epsilon$.

\begin{lemma}
\label{lem:novel_concentration_main_paper}
For any arm $j \in [K]$ and $k \in [TM] \cup \cbr{0}$, 
denote by $\tau_k = \tau_k(j)$. Then, 
for any $\delta \in (0,1]$, with probability at least $1 - \delta$, one of the following events happens:
\begin{enumerate}
    \item $\tau_k = T+1$;
    \item 
    $\forall p \in [M],\ \mu_j^{p} - \aggmu_j(\tau_k) \le \sqrt{\frac{2 \ln \rbr{\frac{2}{\delta}}}{ \rbr{n_j(\tau_k) -M} \vee 1}}.$
\end{enumerate}
\end{lemma}

\begin{remark}
\label{rem:novel_concentration}
We note that Lemma~\ref{lem:novel_concentration_main_paper} is critical to the tight performance guarantee in Lemma~\ref{lem:key_lemma_constant} and subsequently the near-optimal regret guarantees. This result is non-trivial, as it is a concentration bound for a sequence of random variables whose length, $n_j(\tau_k(j))$, is also a random variable. Furthermore, since $\tau_k(j)$ is the round in which arm $j$ is pulled the $k$-th time by any player, $n_j(\tau_k(j))$ can potentially take any integer value in $[k, k+M-1]$ because there can be up to $M$ pulls of arm $j$ in round $\tau_k(j)$. We note that using the Azuma-Hoeffding inequality together with a union bound or Freedman's inequality (similar to \citealp[Lemma 17]{wzsrc21}) can lead to extra $\Ocal \rbr{M}$ or $\Ocal \rbr{\ln T}$ terms for Lemma~\ref{lem:key_lemma_constant}, respectively (see Remark~\ref{rem:comparison_freedman_azuma} in the appendix for details).
\end{remark}

To our best knowledge, we are not aware of any similar tight concentration bounds for data aggregation in multi-task bandits, and our technique may be of independent interest for analyzing other multi-task sequential learning problems.

\begin{figure*}[t]
    \centering
    \begin{subfigure}{\textwidth}
        \centering
        \begin{subfigure}{.32\textwidth}
            \centering
            \includegraphics[height=0.7\linewidth]{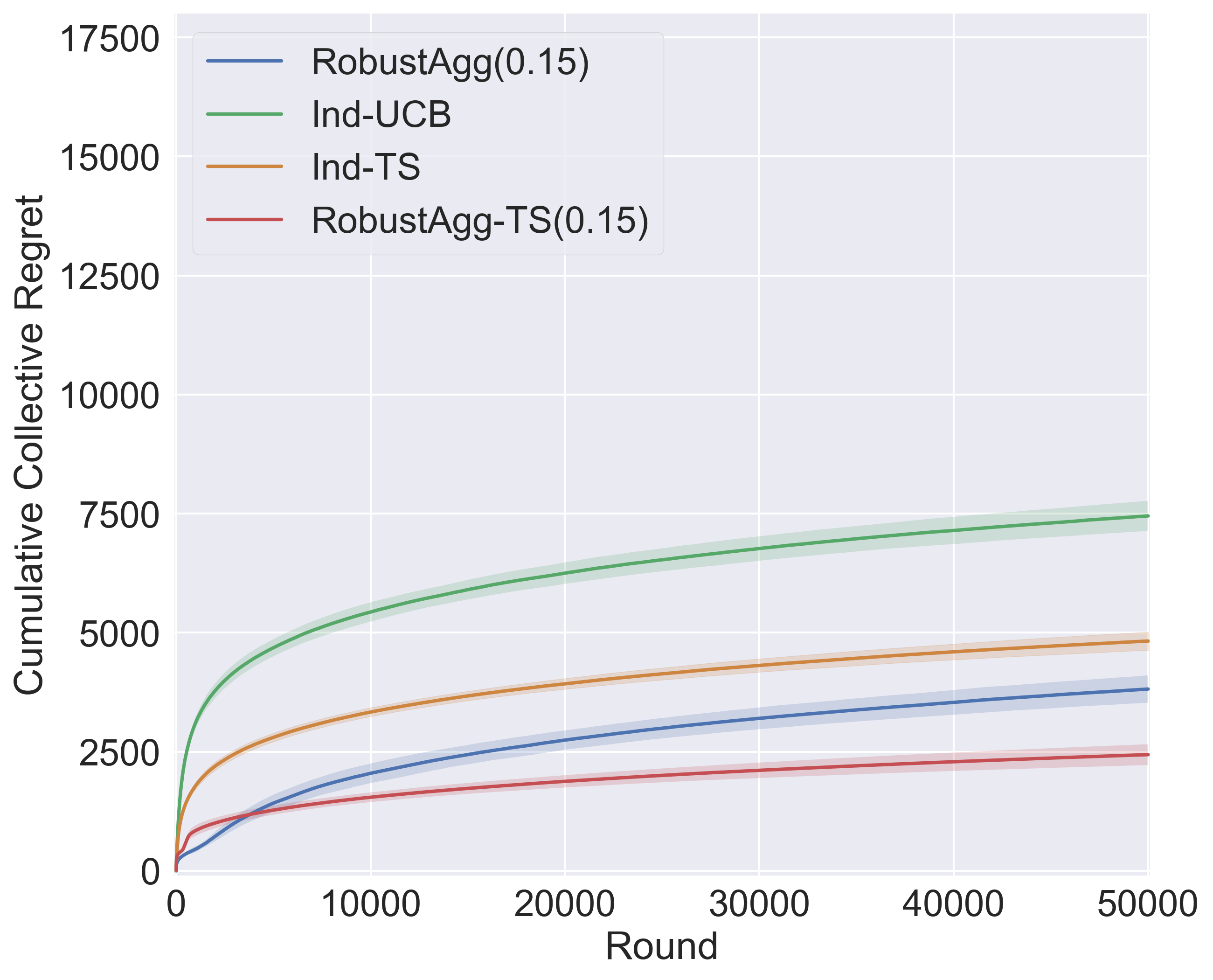}
        \end{subfigure}
        \begin{subfigure}{.32\textwidth}
            \centering
            \includegraphics[height=0.7\linewidth]{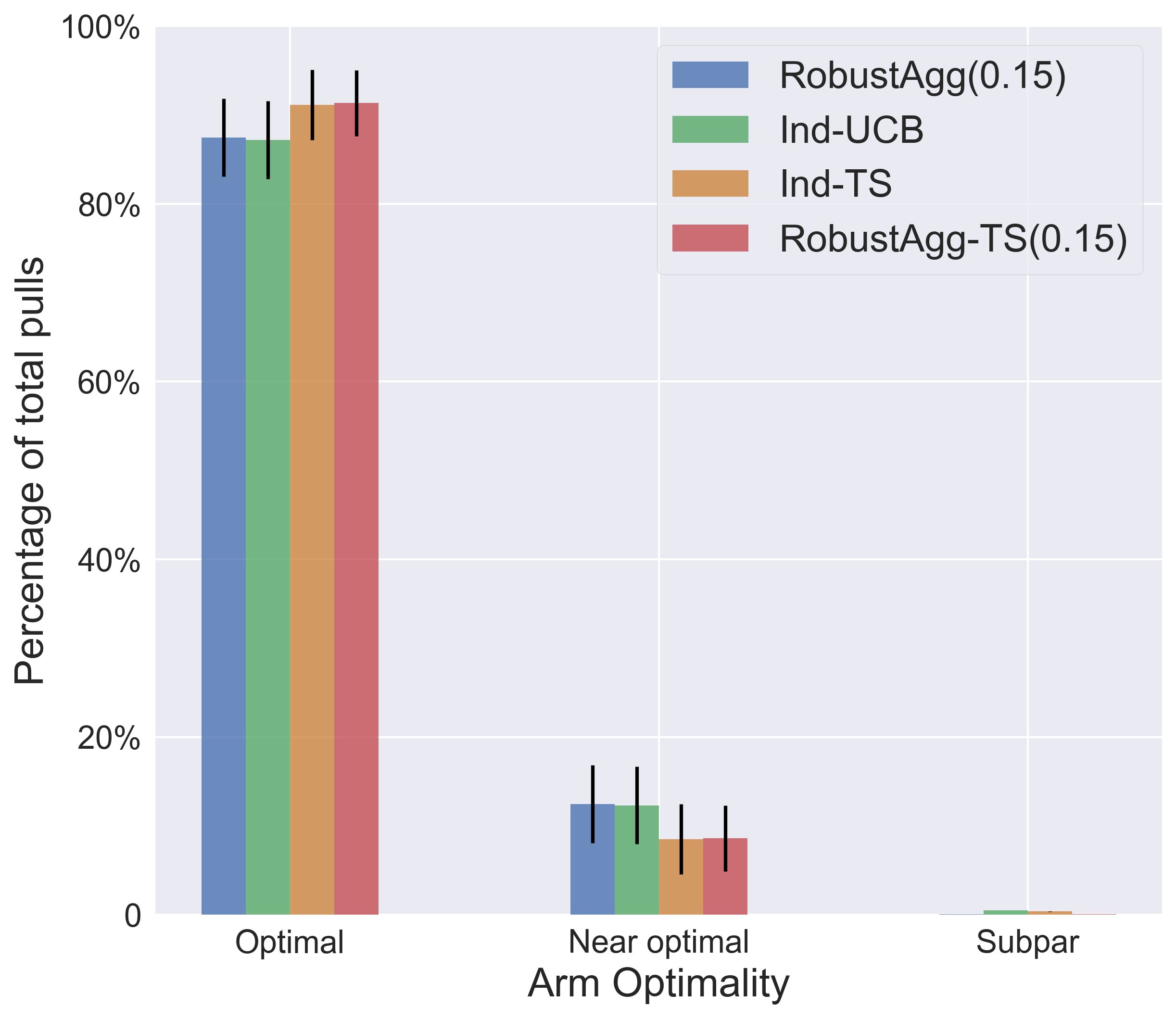}
        \end{subfigure}
        \begin{subfigure}{.32\textwidth}
            \centering
            \includegraphics[height=0.7\linewidth]{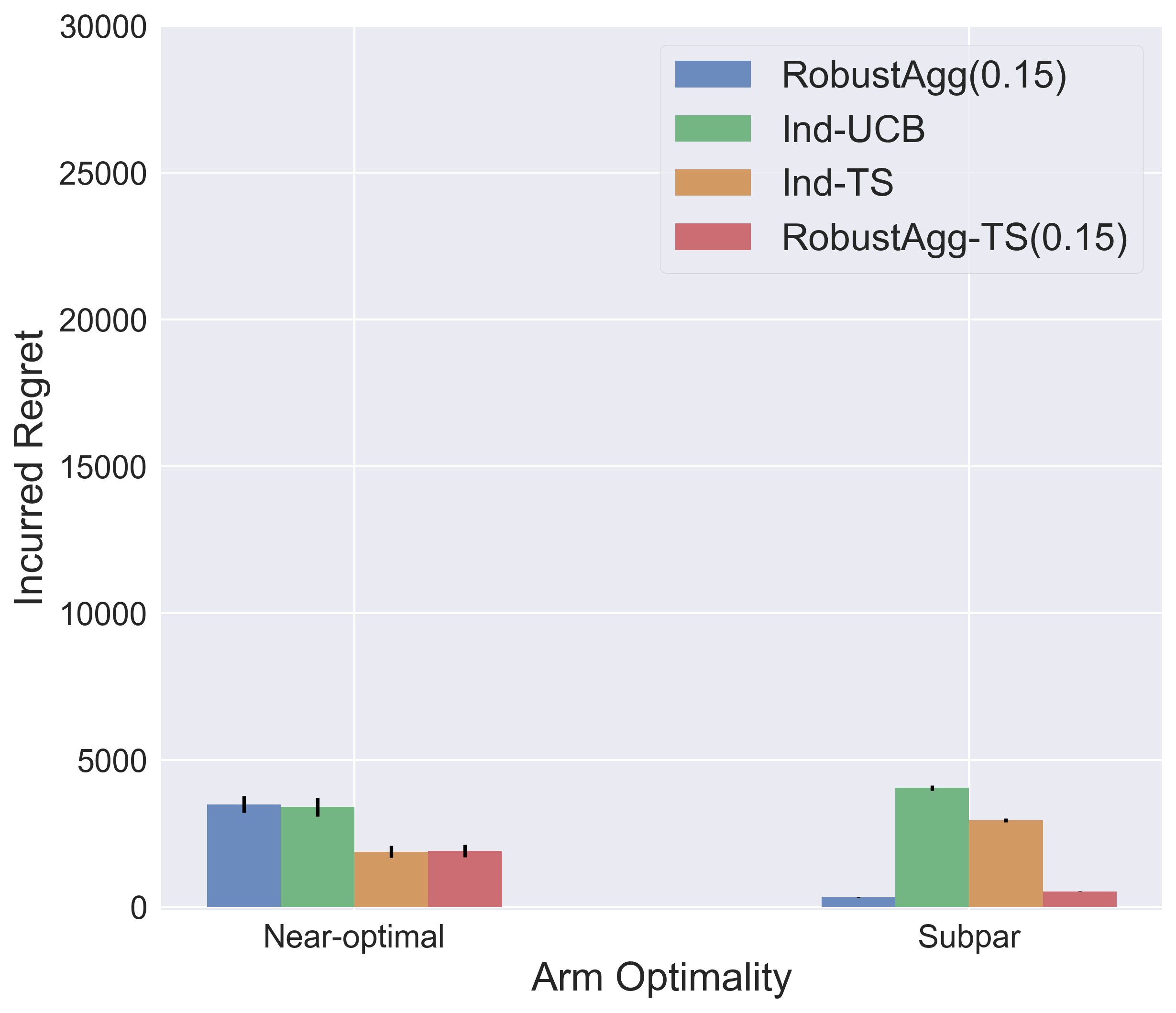}
        \end{subfigure}
        \caption{$\abr{\Ical_{5\epsilon}} = 8$}
        \label{figure:exp_Ical_8_cumulative}
    \end{subfigure}
    \begin{subfigure}{\textwidth}
        \centering
        \begin{subfigure}{.32\textwidth}
            \centering
            \includegraphics[height=0.7\linewidth]{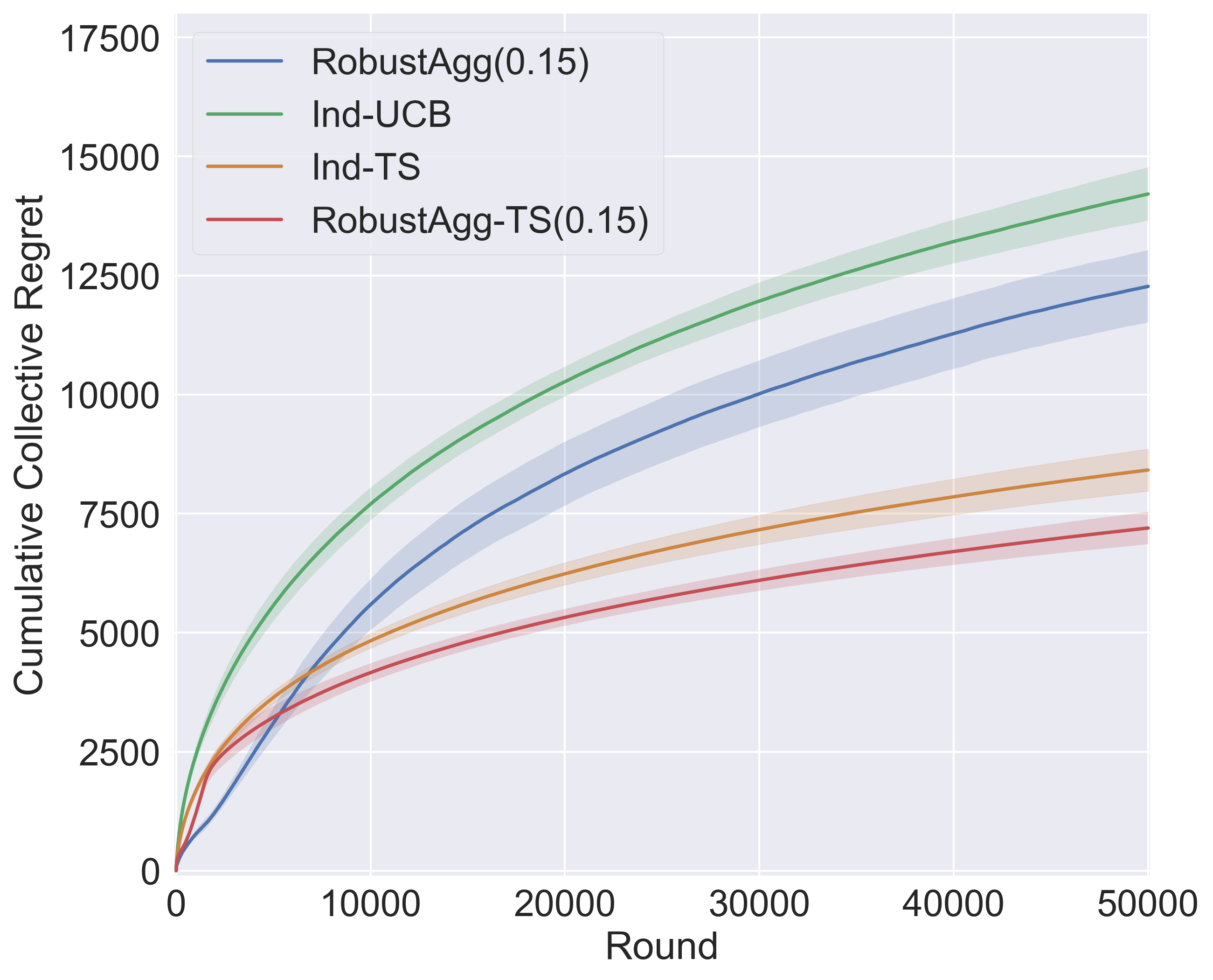}
        \end{subfigure}
        \begin{subfigure}{.32\textwidth}
            \centering
            \includegraphics[height=0.7\linewidth]{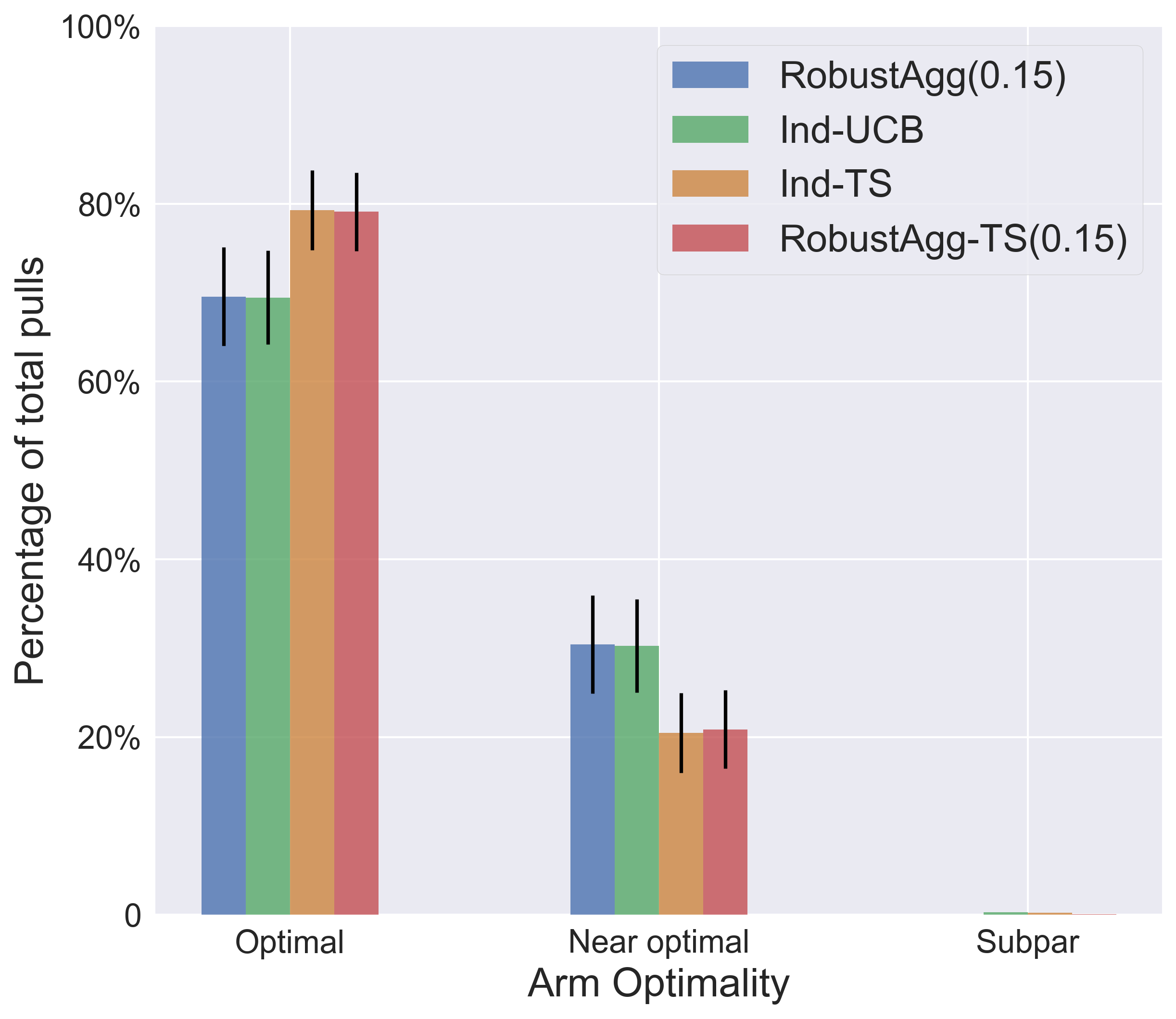}
        \end{subfigure}
        \begin{subfigure}{.32\textwidth}
            \centering
            \includegraphics[height=0.7\linewidth]{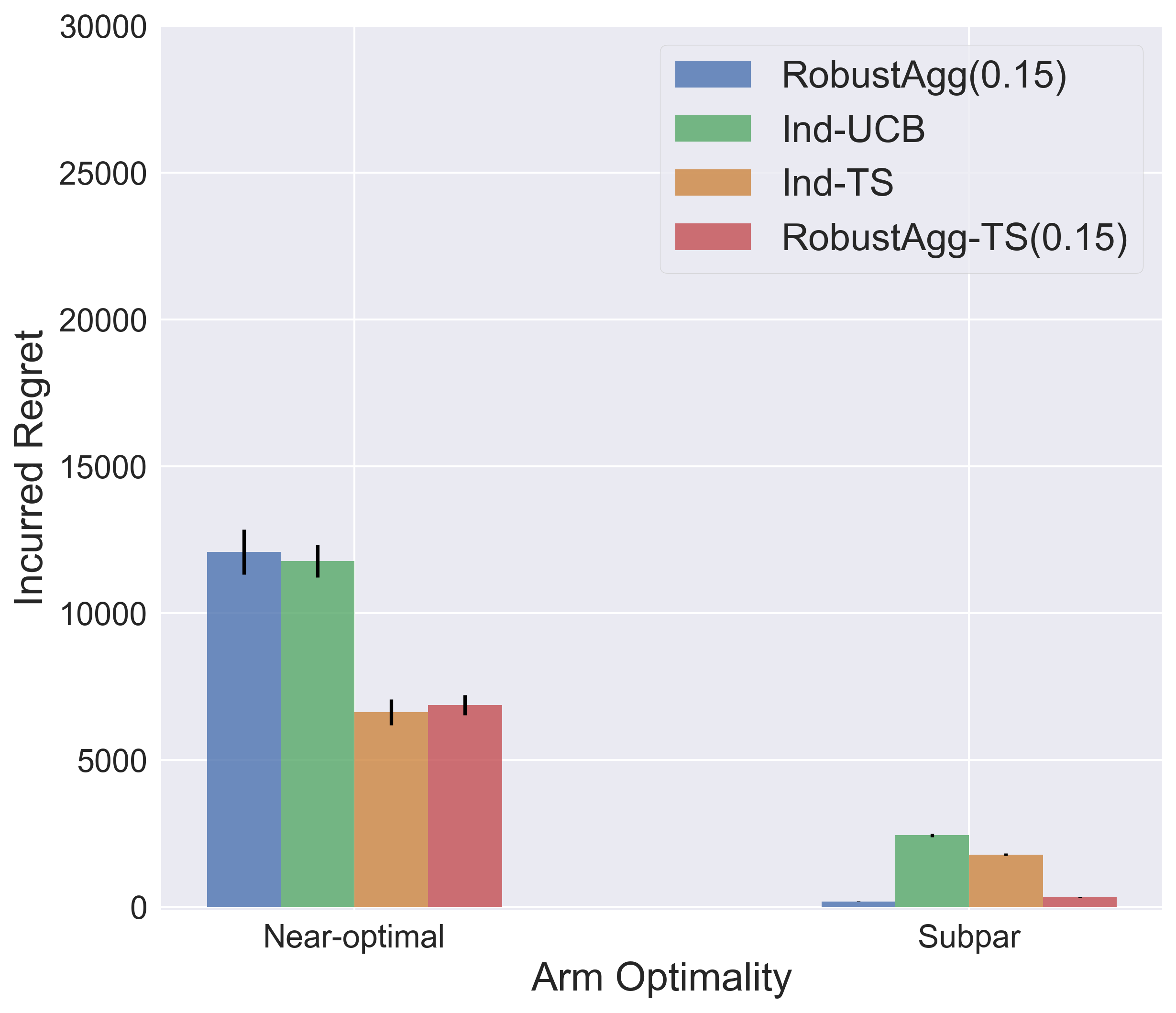}
        \end{subfigure}
        \caption{$\abr{\Ical_{5\epsilon}} = 5$}
        \label{figure:exp_Ical_5_cumulative}
    \end{subfigure}
    \caption{
    Compares the average performance of the algorithms on $30$ randomly generated problem instances with $\abr{\Ical_{5\epsilon}} = 8$ and $\abr{\Ical_{5\epsilon}} = 5$ in a horizon of $T = 50000$ rounds. Figures in the left column plot the cumulative collective regret over time; figures in the middle column demonstrate the percentages of pulls of optimal arms, non-subpar yet non-optimal arms (referred to as near-optimal arms), and subpar arms; figures in the right column then show the incurred cumulative regret by arm optimality.
    }
    \label{figure:experiment1}
\end{figure*}

\section{Related Work}
\label{sec:related_work}

There exist many prior works that study multi-player or multi-task bandits with heterogeneous reward distributions.
For example, \citet{cgz13} use Laplacian-based regularization to learn a network of bandit problem instances such that connected problems have similar parameters;
\citet{glz14}, among others, study clustering of bandit problem instances.
The $\epsilon$-MPMAB problem studied in this paper is introduced by \citet{wzsrc21};
see Appendix A thereof for a detailed comparison with related work.
More recently, \citet{zhang2021provably} generalize the $\epsilon$-MPMAB problem to episodic, tabular Markov decision processes. 
We note that while the methods in the above-mentioned works are UCB-based, we study TS-type algorithms in this work.

TS is initially proposed by \citet{thompson1933likelihood} decades ago, but its frequentist analysis has not emerged until recent years \citep[e.g.,][]{agrawal2012analysis,kaufmann2012thompson}. \citet{jin2021mots} present the first minimax optimal TS-type algorithm.
Our proof techniques in this paper are mostly inspired by the work of \citet{an17}.

TS algorithms have been studied in multi-task Bayesian bandits.
For example, several recent works study the setting of interacting with a sequence of $M$ bandit problem instances (tasks) sampled from a common, unknown prior distribution, with a goal of minimizing the $M$-instance Bayesian regret \citep{bastani2021meta,kveton21a,peleg2021metalearning,basu2021no}.
The recent work of \citet{hong2021hierarchical} proposes a hierarchical Bayesian bandit problem that generalizes many multi-task bandit settings, and analyzes the Bayes regret.
In contrast, we use {frequentist} regret as our performance metric, and we do not assume a shared prior distribution over the players' problem instances/tasks.
\citet{wan2021metadatabased} study multi-task TS in a hierarchical Bayesian model and assume knowledge of metadata of each task; while they provide a frequentist regret bound, we study the $\epsilon$-MPMAB problem which models task relations differently.

Similar models on {\em sequential} transfer between problem instances have also been studied by \citet{alb13} and \citet{soaremulti}. \citet{zhang2017transfer,zailn19,sbss20} investigate warm-starting bandits from misaligned data.
In this work, we focus on a more general interaction protocol, under which the players may interact with the environment concurrently.

\section{Empirical Evaluation}
\label{sec:experiments}

In this section, we present an empirical evaluation of $\robustaggTS(\epsilon)$ on synthetic data\footnote{Our code is available at \url{https://github.com/zhiwang123/eps-MPMAB-TS}.}. We focus on the concurrent setting ($\Pcal_t = [M]$ for all $t$), which is the setting studied in the experiments of \citep{wzsrc21}. Our goal is to address the following two questions:

(1) How does $\robustaggTS(\epsilon)$ perform in comparison with the UCB-based algorithm, $\robustagg(\epsilon)$, and the baseline algorithms without transfer learning? 

(2) Does the notion of subpar arms characterize the performance of the algorithms in practice?

\paragraph{Experimental Setup.} We compared the performance of 4 algorithms: (1) $\robustaggTS(\epsilon)$ with constants $c_1 = \frac12$ and $c_2 = 1$; (2) $\robustagg(\epsilon)$ \citep[Section 6.1]{wzsrc21}; (3) $\indts$, the baseline algorithm that runs TS with Gaussian priors for each player individually; and (4) $\inducb$, the baseline algorithm that runs UCB-1 for each player individually. 

The algorithms were evaluated on randomly generated $0.15$-MPMAB problem instances with different numbers of subpar arms. To stay consistent with the work of \citet{wzsrc21}, we followed the same instance generation procedure and considered $\Ical_{5\epsilon}$ to be the set of subpar arms---we set the number of players $M = 20$ and the number of arms $K = 10$; then, for each integer value $v \in [0, 9]$, we generated $30$ $0.15$-MPMAB problem instances with Bernoulli reward distributions and $\abr{\Ical_{5\epsilon}} = v$. We ran the algorithms on each instance for a horizon of $T = 50,000$ rounds. %

\paragraph{Results and Discussion.} Figure~\ref{figure:experiment1} compares the average performance of the algorithms on instances with $\abr{\Ical_{5\epsilon}} = 8$ and $5$. We defer the rest of the results to Appendix~\ref{sec:appendix_exp}.

From the left column, we first observe that, while the UCB-based algorithm, $\robustagg(\epsilon)$, outperforms its counterpart, $\inducb$, in the cumulative collective regret ($\sum_{t \in [T]} \sum_{p \in \Pcal_t} \mu_*^p - \mu_{i_t^p}^p$), its empirical performance is underwhelming in comparison with TS algorithms. In particular, even on instances with half of the arms {\em subpar} ($\abr{\Ical_{5\epsilon}} = 5$), $\robustagg(\epsilon)$ is outperformed by the $\indts$ baseline without transfer learning. Importantly, we note that $\robustaggTS(\epsilon)$ shows a superior performance than the other algorithms.

The figures in the middle and right columns illustrate the arm selection of each algorithm. We categorize all arms into three groups: optimal arms, subpar arms, and near-optimal arms which are neither subpar nor optimal. Comparing the TS-type algorithms with the UCB-based algorithms, we observe that the former algorithms perform better mainly because they pull near-optimal arms a smaller number of times and incur less regret on these arms. 

Furthermore, we observe that $\robustagg(\epsilon)$ and $\robustaggTS(\epsilon)$, when compared with their counterparts ($\inducb$ and $\indts$, respectively), incur a similar amount of regret from near-optimal arms. Meanwhile, they make fewer pulls on subpar arms. This may be less obvious from the plots on the percentage of total pulls because none of the algorithms pull subpar arms extensively over the horizon. However, since the suboptimality gaps of subpar arms are large, we see from the figures in the right column that $\robustagg(\epsilon)$ and $\robustaggTS(\epsilon)$ incur far less regret on subpar arms. These results thereby demonstrate that the notion of subpar arms can capture the amenability of transfer learning in subpar arms but not near-optimal arms. 

In addition, the results show that, empirically, our proposed algorithm $\robustaggTS(\epsilon)$ can robustly leverage transfer for arms in $\Ical_{5\epsilon} \supseteq \Ical_{10\epsilon}$---this suggests that our upper bounds may be improved; we leave this as future work.

\section{Conclusion}
In this work, we studied transfer learning in multi-task bandits under the framework of a generalized version of the $\epsilon$-MPMAB problem \citep{wzsrc21}. 
We proposed a TS-type algorithm, $\robustaggTS(\epsilon)$, which can robustly leverage auxiliary data collected for other tasks.
We showed that $\robustaggTS(\epsilon)$ is empirically superior when evaluated on synthetic data, and also near-optimal in gap-dependent and gap-independent frequentist guarantees. 
In our analysis, we also proved a novel concentration inequality for multi-task data aggregation, which can be of independent interest in the analysis of other multi-task online learning problems.
For future work, we are interested in improving the lower-order terms in our regret bounds and evaluating our algorithm in real-world applications.

\section{Acknowledgements}
We thank Geelon So for insightful discussions.
ZW and KC thank the National Science Foundation under IIS 1915734 and CCF 1719133 for research support. 
CZ acknowledges startup funding support from the University of Arizona. We also thank the anonymous reviewers for their constructive feedback.

\bibliography{ref}

\begin{thebibliography}{27}
\providecommand{\natexlab}[1]{#1}
\providecommand{\url}[1]{\texttt{#1}}
\expandafter\ifx\csname urlstyle\endcsname\relax
  \providecommand{\doi}[1]{doi: #1}\else
  \providecommand{\doi}{doi: \begingroup \urlstyle{rm}\Url}\fi

\bibitem[Agrawal \& Goyal(2012)Agrawal and Goyal]{agrawal2012analysis}
Agrawal, S. and Goyal, N.
\newblock Analysis of thompson sampling for the multi-armed bandit problem.
\newblock In \emph{Conference on learning theory}, pp.\  39--1. JMLR Workshop
  and Conference Proceedings, 2012.

\bibitem[Agrawal \& Goyal(2017)Agrawal and Goyal]{an17}
Agrawal, S. and Goyal, N.
\newblock Near-optimal regret bounds for thompson sampling.
\newblock \emph{J. ACM}, 64\penalty0 (5), sep 2017.
\newblock ISSN 0004-5411.
\newblock \doi{10.1145/3088510}.
\newblock URL \url{https://doi.org/10.1145/3088510}.

\bibitem[Auer et~al.(2002)Auer, Cesa-Bianchi, and Fischer]{acf02}
Auer, P., Cesa-Bianchi, N., and Fischer, P.
\newblock Finite-time analysis of the multiarmed bandit problem.
\newblock \emph{Machine learning}, 47\penalty0 (2-3):\penalty0 235--256, 2002.

\bibitem[Azar et~al.(2013)Azar, Lazaric, and Brunskill]{alb13}
Azar, M.~G., Lazaric, A., and Brunskill, E.
\newblock Sequential transfer in multi-armed bandit with finite set of models.
\newblock In Burges, C. J.~C., Bottou, L., Welling, M., Ghahramani, Z., and
  Weinberger, K.~Q. (eds.), \emph{Advances in Neural Information Processing
  Systems 26}, pp.\  2220--2228. Curran Associates, Inc., 2013.

\bibitem[Bastani et~al.(2021)Bastani, Simchi-Levi, and Zhu]{bastani2021meta}
Bastani, H., Simchi-Levi, D., and Zhu, R.
\newblock Meta dynamic pricing: Transfer learning across experiments.
\newblock \emph{Management Science}, 2021.

\bibitem[Basu et~al.(2021)Basu, Kveton, Zaheer, and Szepesv{\'a}ri]{basu2021no}
Basu, S., Kveton, B., Zaheer, M., and Szepesv{\'a}ri, C.
\newblock No regrets for learning the prior in bandits.
\newblock \emph{arXiv preprint arXiv:2107.06196}, 2021.

\bibitem[Ben-David et~al.(2010)Ben-David, Blitzer, Crammer, Kulesza, Pereira,
  and Vaughan]{bbckpv10}
Ben-David, S., Blitzer, J., Crammer, K., Kulesza, A., Pereira, F., and Vaughan,
  J.~W.
\newblock A theory of learning from different domains.
\newblock \emph{Machine learning}, 79\penalty0 (1-2):\penalty0 151--175, 2010.

\bibitem[Cesa-Bianchi et~al.(2013)Cesa-Bianchi, Gentile, and Zappella]{cgz13}
Cesa-Bianchi, N., Gentile, C., and Zappella, G.
\newblock A gang of bandits.
\newblock In \emph{Advances in Neural Information Processing Systems}, pp.\
  737--745, 2013.

\bibitem[Chapelle \& Li(2011)Chapelle and Li]{chapelle2011empirical}
Chapelle, O. and Li, L.
\newblock An empirical evaluation of thompson sampling.
\newblock \emph{Advances in neural information processing systems},
  24:\penalty0 2249--2257, 2011.

\bibitem[Gentile et~al.(2014)Gentile, Li, and Zappella]{glz14}
Gentile, C., Li, S., and Zappella, G.
\newblock Online clustering of bandits.
\newblock In \emph{International Conference on Machine Learning}, pp.\
  757--765, 2014.

\bibitem[Gordon(1941)]{gordon1941millsratio}
Gordon, R.~D.
\newblock Values of mills' ratio of area to bounding ordinate and of the normal
  probability integral for large values of the argument.
\newblock \emph{The Annals of Mathematical Statistics}, 12\penalty0
  (3):\penalty0 364--366, 1941.

\bibitem[Hong et~al.(2021)Hong, Kveton, Zaheer, and
  Ghavamzadeh]{hong2021hierarchical}
Hong, J., Kveton, B., Zaheer, M., and Ghavamzadeh, M.
\newblock Hierarchical bayesian bandits.
\newblock \emph{arXiv preprint arXiv:2111.06929}, 2021.

\bibitem[Jin et~al.(2021)Jin, Xu, Shi, Xiao, and Gu]{jin2021mots}
Jin, T., Xu, P., Shi, J., Xiao, X., and Gu, Q.
\newblock Mots: Minimax optimal thompson sampling.
\newblock In \emph{International Conference on Machine Learning}, pp.\
  5074--5083. PMLR, 2021.

\bibitem[Kaufmann et~al.(2012)Kaufmann, Korda, and Munos]{kaufmann2012thompson}
Kaufmann, E., Korda, N., and Munos, R.
\newblock Thompson sampling: An asymptotically optimal finite-time analysis.
\newblock In \emph{International conference on algorithmic learning theory},
  pp.\  199--213. Springer, 2012.

\bibitem[Kubota et~al.(2020)Kubota, Peterson, Rajendren, Kress-Gazit, and
  Riek]{kubota2020jessie}
Kubota, A., Peterson, E.~I., Rajendren, V., Kress-Gazit, H., and Riek, L.~D.
\newblock Jessie: Synthesizing social robot behaviors for personalized
  neurorehabilitation and beyond.
\newblock In \emph{Proceedings of the 2020 ACM/IEEE International Conference on
  Human-Robot Interaction}, pp.\  121--130, 2020.

\bibitem[Kveton et~al.(2021)Kveton, Konobeev, Zaheer, Hsu, Mladenov, Boutilier,
  and Szepesvari]{kveton21a}
Kveton, B., Konobeev, M., Zaheer, M., Hsu, C.-W., Mladenov, M., Boutilier, C.,
  and Szepesvari, C.
\newblock Meta-thompson sampling.
\newblock In \emph{Proceedings of the 38th International Conference on Machine
  Learning}, volume 139 of \emph{Proceedings of Machine Learning Research},
  pp.\  5884--5893. PMLR, 18--24 Jul 2021.

\bibitem[Peleg et~al.(2021)Peleg, Pearl, and Meir]{peleg2021metalearning}
Peleg, A., Pearl, N., and Meir, R.
\newblock Metalearning linear bandits by prior update.
\newblock \emph{arXiv preprint arXiv:2107.05320}, 2021.

\bibitem[Qian et~al.(2013)Qian, Feng, Zhao, and Mei]{qian2013personalized}
Qian, X., Feng, H., Zhao, G., and Mei, T.
\newblock Personalized recommendation combining user interest and social
  circle.
\newblock \emph{IEEE transactions on knowledge and data engineering},
  26\penalty0 (7):\penalty0 1763--1777, 2013.

\bibitem[Rosenstein et~al.(2005)Rosenstein, Marx, Kaelbling, and
  Dietterich]{rosenstein2005transfer}
Rosenstein, M.~T., Marx, Z., Kaelbling, L.~P., and Dietterich, T.~G.
\newblock To transfer or not to transfer.
\newblock In \emph{NIPS 2005 workshop on transfer learning}, 2005.

\bibitem[Sharma et~al.(2020)Sharma, Basu, Shanmugam, and Shakkottai]{sbss20}
Sharma, N., Basu, S., Shanmugam, K., and Shakkottai, S.
\newblock Warm starting bandits with side information from confounded data.
\newblock \emph{arXiv preprint arXiv:2002.08405}, 2020.

\bibitem[Soare et~al.(2014)Soare, Alsharif, Lazaric, and Pineau]{soaremulti}
Soare, M., Alsharif, O., Lazaric, A., and Pineau, J.
\newblock Multi-task linear bandits.
\newblock \emph{NIPS2014 Workshop on Transfer and Multi-task Learning : Theory
  meets Practice}, 2014.

\bibitem[Thompson(1933)]{thompson1933likelihood}
Thompson, W.~R.
\newblock On the likelihood that one unknown probability exceeds another in
  view of the evidence of two samples.
\newblock \emph{Biometrika}, 25\penalty0 (3/4):\penalty0 285--294, 1933.

\bibitem[Wan et~al.(2021)Wan, Ge, and Song]{wan2021metadatabased}
Wan, R., Ge, L., and Song, R.
\newblock Metadata-based multi-task bandits with bayesian hierarchical models.
\newblock In Beygelzimer, A., Dauphin, Y., Liang, P., and Vaughan, J.~W.
  (eds.), \emph{Advances in Neural Information Processing Systems}, 2021.

\bibitem[Wang et~al.(2021)Wang, Zhang, Singh, Riek, and Chaudhuri]{wzsrc21}
Wang, Z., Zhang, C., Singh, M.~K., Riek, L., and Chaudhuri, K.
\newblock Multitask bandit learning through heterogeneous feedback aggregation.
\newblock In \emph{International Conference on Artificial Intelligence and
  Statistics}, pp.\  1531--1539. PMLR, 2021.

\bibitem[Zhang \& Wang(2021)Zhang and Wang]{zhang2021provably}
Zhang, C. and Wang, Z.
\newblock Provably efficient multi-task reinforcement learning with model
  transfer.
\newblock \emph{arXiv preprint arXiv:2107.08622}, 2021.

\bibitem[Zhang et~al.(2019)Zhang, Agarwal, Iii, Langford, and
  Negahban]{zailn19}
Zhang, C., Agarwal, A., Iii, H.~D., Langford, J., and Negahban, S.
\newblock Warm-starting contextual bandits: Robustly combining supervised and
  bandit feedback.
\newblock In \emph{International Conference on Machine Learning}, pp.\
  7335--7344, 2019.

\bibitem[Zhang \& Bareinboim(2017)Zhang and Bareinboim]{zhang2017transfer}
Zhang, J. and Bareinboim, E.
\newblock Transfer learning in multi-armed bandit: a causal approach.
\newblock In \emph{Proceedings of the 16th Conference on Autonomous Agents and
  MultiAgent Systems}, pp.\  1778--1780, 2017.

\end{thebibliography}
\bibliographystyle{icml2022}

\newpage
\appendix
\onecolumn

\paragraph{Outline.}
The structure of this appendix is as follows. 
\begin{itemize}
    \item In Section~\ref{sec:appendix_preliminaries}, we introduce some basic definitions, facts and additional notations that are used in our analysis. 
    
    \item In Section~\ref{sec:appendix_concentration_bounds}, we formally present and prove the concentration bounds used in our proofs, including our novel concentration inequality for multi-task data aggregation at stopping times.
    
    \item In Section~\ref{sec:appendix_main_proofs}, we prove Theorem~\ref{thm:ts_gap_dep_ub} and Theorem~\ref{thm:ts_gap_indep_ub}.
    
    \item In Section~\ref{sec:appendix_baselines}, we discuss the performance guarantees of the baseline algorithms in the $\epsilon$-MPMAB problem, which include $\inducb, \indts$, and $\robustagg(\epsilon)$.
    
    \item Finally, we provide additional experimental results in Section~\ref{sec:appendix_exp}.
\end{itemize}

\section{Basic Definitions and Facts}
\label{sec:appendix_preliminaries}

In this section, we revisit and introduce a few basic definitions, facts and additional notations that are useful in our proofs.

\begin{definition}[Constants used in the analysis]
In the analysis, we set
\[
c_1 = 40, c_2 = 4
\]
to be the constants used in Algorithm~\ref{alg:robustaggTS}.\footnote{If we choose $c_1$ to some other positive number, we can still show guarantees similar to Theorems~\ref{thm:ts_gap_dep_ub} and~\ref{thm:ts_gap_indep_ub}, except that $\Ical_{10\epsilon}$ needs to be changed to $\Ical_{\Ocal \rbr{\sqrt{\frac{1}{c_1}}\epsilon}}$---the analysis of case $(A1)$ needs to be changed accordingly. On the other hand, it is also possible to change $c_2$ to any constant $> 1$ and establish similar regret guarantees, by tightening the exponents of the concentration inequalities (Corollaries~\ref{col:agg_concentration_alternative} and~\ref{col:ind_concentration_alternative}) and Lemma~\ref{lem:sg-anticonc}. We leave the details to interested readers.
}
\end{definition}

\begin{definition}[Number of pulls]
\label{def:ind_num_pulls}
Recall that 
\[
n_i^p(t) = \sum_{s \le t} \indic \cbr{p \in \Pcal_s, i_s^p = i}
\]
is the number of pulls of arm $i$ by player $p$ after $t$ rounds.
We define
\[
n_i(t) = \sum_{p \in [M]} n_i^p(t)
\]
to be the total of number of pulls of arm $i$ by all the players after $t$ rounds.
\end{definition}

\begin{definition}[Individual mean estimate]
For any $i \in [K]$, $p \in [M]$, and $t \in [T] \cup \cbr{0}$, let
    \[
    \indmu_i^p(t) = \frac{1}{n_i^p(t) \vee 1} \sum_{s \le t} \indic \cbr{p \in \Pcal_s, i_s^p = i} r_s^p
    \]
    be the empirical mean computed for arm $i$ using player $p$'s own data from the first $t$ rounds.
\end{definition}

\begin{definition}
Define
\[
\indvar_i^p(t) = \frac{4}{n_i^p(t) \vee 1}.
\]
\end{definition}

\begin{remark}[mean and variance of the individual posteriors]
\label{pro:ind_posterior}
By the construction of Algorithm~\ref{alg:robustaggTS}, we have that, in any round $t \in [T]$, for any active player $p \in \Pcal_t$ and arm $i$,
$\indmu_i^p(t-1)$ and $\indvar_i^p(t-1)$
are the mean and variance of the individual posterior associated with arm $i$ and player $p$ in round $t$, respectively.
\end{remark}

\begin{definition}[Aggregate mean estimate]
\label{def:agg_mean_estimate}
For any $i \in [K]$ and $t \in [T] \cup \cbr{0}$, let
\[
\aggmu_i(t) = \frac{1}{n_i(t) \vee 1} \sum_{s \le t} \sum_{q: q \in \Pcal_s } \indic \cbr{i_s^q = i} r_s^q +  \epsilon
\]
be the empirical mean computed for arm $i$ using all players' data from the first $t$ rounds, offset by the dissimilarity parameter $\epsilon$. {Note that the definition of $\aggmu_i(t)$ does not depend on the identity of a specific player $p$.}
\end{definition}

\begin{definition}[Most recent pull]
\label{def:last_pulled_round}
In any round $t \in [T] \cup \cbr{0}$, for any player $p \in [M]$ and arm $i \in [K]$, we define
\[
u_i^p(t) = 
\begin{cases}
\max \cbr{s \le t: p \in \Pcal_s, i_s^p = i},  & n_i^p(t) > 0 \\
0, & n_i^p(t) = 0
\end{cases}
\]
to be the round in which player $p$ most recently pulled arm $i$ (including round $t$); we let $u_i^p(t) = 0$ by convention if player $p$ has not yet pulled arm $i$.
\end{definition}

\begin{definition}[Aggregate mean estimate maintained by player $p$]
For any $t \in [T] \cup \cbr{0}$, $p \in [M]$, and $i \in [K]$, define
\[
\aggmu_i^p(t) = \aggmu_i(u_i^p(t)).
\]
Note that the superscript $p$ differentiates this player-dependent aggregate mean estimate from $\aggmu_i(t)$ in Definition~\ref{def:agg_mean_estimate}, which does not depend on any individual player.
\end{definition}

\begin{definition}[Aggregate number of pulls maintained by player $p$]
\label{def:m_i^p}
For any $t \in [T] \cup \cbr{0}$, $p \in [M]$, and $i \in [K]$, define
\[
m_i^p(t) = n_i(u_i^p(t))
\]
to be the total number of pulls of arm $i$ by all the players until the round in which player $p$ last pulled arm $i$.
\end{definition}

\begin{definition}
Define
\[
\aggvar_i^p(t) = \frac{4}{ (m_i^p(t) - M) \vee 1}.
\]
\end{definition}

\begin{remark}[mean and variance of the aggregate posteriors]
\label{pro:agg_posterior}
By the construction of Algorithm~\ref{alg:robustaggTS},
in any round $t \in [T]$, for any active player $p \in \Pcal_t$ and arm $i$, we have that $\aggmu_i^p(t-1)$ and $\aggvar_i^p(t-1)$
are the mean and variance of the aggregate posterior associated with arm $i$ and player $p$ in round $t$, respectively.
\end{remark}

\begin{definition}[Filtration]
Let $\cbr{\Fcal_t}_{t=0}^{T}$ be a filtration
such that
\[
\Fcal_t = \sigma \rbr{ \cbr{i_s^q, r_s^q: s\le t, q \in \Pcal_s}}
\]
is the $\sigma$-algebra generated by interactions of all players up until and including round $t$.
\end{definition}

\begin{definition}
\label{def:alg_criteron}
Let
\[
H_i^p(t) = \cbr{ n_i^p(t-1) \ge \frac{40 \ln T}{\epsilon^2} + 2M}
\]
be the event that in round $t$, for arm $i$, player $p$ uses the {\em individual} posterior distribution;
correspondingly,
let
\[
\wbar{H_i^p(t)} = \cbr{ n_i^p(t-1) < \frac{40 \ln T}{\epsilon^2} + 2M }
\]
be the event that in round $t$, for arm $i$, player $p$ uses the {\em aggregate} posterior distribution. See lines~\ref{line:choose-posterior-start} to \ref{line:choose-posterior-end} in Algorithm~\ref{alg:robustaggTS}.
\end{definition}

\begin{remark}
\label{rem:mu-i-p-decomp}
With the above notations, 
\[
\hat{\mu}_i^p(t-1) = \aggmu_i^p(t-1) \cdot \indic (\overline{H_i^p(t)}) + \indmu_i^p(t-1) \cdot \indic (H_i^p(t)),
\]
and
\[
\var_i^p(t-1) = \aggvar_i^p(t-1) \cdot \indic (\overline{H_i^p(t)}) + \indvar_i^p(t-1) \cdot \indic (H_i^p(t)).
\]
\end{remark}

\paragraph{Stopping times.} In our analysis, we will frequently use the following notions of stopping times:

\begin{definition}
\label{def:tau_k}
For any arm $i \in [K]$ and $k \in [TM]$, 
let
\[
\tau_k(i) = \min \cbr {T+1, \min \cbr{t: n_i(t) \ge k}}
\]
be the round in which arm $i$ is pulled the $k$-th time by any player. Furthermore, as a convention, let $\tau_0(i) = 0$.%
\end{definition}

\begin{remark}
\label{rem:tau_k_stopping}
For any $i \in [K]$ and $k \in [TM]$, $\tau_k(i)$ is a stopping time with respect to $\cbr{\Fcal_t}_{t=0}^T$.
Indeed, for any $t \le T$,
\[
\cbr{\tau_k(i) \le t}
=
\cbr{
\sum_{s \in [t]} \sum_{p: p \in \Pcal_s} \indic\cbr{i_s^p = i} \ge k
}
\in \Fcal_t.
\]
\end{remark}

\begin{definition}
\label{def:p_k}
For any arm $i \in [K]$ and $k \in [TM]$, such that $\tau_k(i) \leq T$, let 
$p_k(i)$ be the unique $p \in [M]$ such that $i_{\tau_k(i)}^p = i$ and 
\[
\sum_{s=1}^{\tau_k(i)-1}
\sum_{q \in \Pcal_s}  \indic\cbr{i_s^q = i}
+ 
\sum_{q \in \Pcal_{\tau_k(i)}: q \leq p} \indic\cbr{i_s^q = i}
= k.
\]
In words, $p_k(i)$ is the player that makes the $k$-th pull of arm $i$, where arm pulls within a round are ordered by the indices of active players in that round.
\end{definition}

\begin{definition}
\label{def:pi_k}
For any arm $i \in [K]$, player $p \in [M]$, and $k \in [T]$, 
let
\[
\pi_k(i, p) = \min \cbr {T+1, \min \cbr{t: n_i^p(t) \ge k}}
\]
be the round in which arm $i$ is pulled the $k$-th time by player $p$. In addition, let $\pi_0(i,p) = 0$ by convention.
\end{definition}

\begin{remark}
\label{rem:pi_k_stopping}
For any $i \in [K]$ and $k \in [T]$, $\pi_k(i,p)$ is a stopping time with respect to $\cbr{\Fcal_t}_{t=0}^T$.
Indeed, for any $t \le T$,
\[
\cbr{\pi_k(i, p) \le t}
=
\cbr{
\sum_{s \in [t]: p \in \Pcal_s}
\indic\cbr{i_s^p = i} \ge k
}
\in \Fcal_t.
\]
\end{remark}

The following property, namely, the invariant property, will also be useful for our analysis.

\begin{definition}[Invariant property]
\label{def:invariant}
We say that:
\begin{enumerate}
\item a set of random variables $\cbr{g_t: t \in [T]}$ satisfies the \emph{invariant property with respect to arm $i \in [K]$ and player $p \in [M]$}, if 
$g_t$ stays constant/invariant between two consecutive pulls of arm $i$ by player $p$, i.e., for any $s \in [T]$ such that $\pi_s(i,p) \le T$, $g_t$ is constant for all $t \in [\pi_{s-1}(i,p)+1, \pi_{s}(i,p)]$.
In other words, for any $s \in [T]$ such that $\pi_s(i,p) \le T$,
\[
g_{\pi_{s-1}(i,p)+1}
=
g_{\pi_{s-1}(i,p)+2}
= \ldots
=
g_{\pi_{s}(i,p)}.
\]
\item a set of random variables $\cbr{f_t^p: t \in [T], p \in [M]}$ satisfies the \emph{invariant property with respect to arm $i \in [K]$}, if for every player $p \in [M]$, $\cbr{f_t^p: t \in [T]}$ satisfy the invariant property with respect to $(i,p)$.
\end{enumerate}
\end{definition}

\begin{example}
\label{ex:h-invariant}
By the construction of Algorithm~\ref{alg:robustaggTS}, in any round $t$, a player only updates the posteriors associated with an arm if the player pulls the arm in round $t$ (line~\ref{line:invariant}). It is easy to verify that for any arm $i \in [K]$ and $p \in [M]$, $\cbr{H_i^p(t): t \in [T]}$ satisfies the invariant property with respect to $(i,p)$. Specifically, for any $s \in [T]$ such that $\pi_s(i,p) \le T$,
\[
H_i^p(\pi_{s-1}(i,p) + 1)
=
H_i^p(\pi_{s-1}(i,p) + 2)
=
\ldots
= 
H_i^p(\pi_{s}(i,p)).
\]
Consequently, $\cbr{H_i^p(t): t \in [T], p \in [M]}$ satisfies the invariant property with respect to $i$.
\end{example}

\begin{example}
\label{ex:m-n-invariant}
For any arm $i \in [K]$ and any player $p \in [M]$, 
$\cbr{n_i^p(t-1): t \in [T]}$ and $\cbr{m_i^p(t-1): t \in [T]}$ both satisfy the invariant property with respect to $(i,p)$ (see Definition~\ref{def:ind_num_pulls} and Definition~\ref{def:m_i^p}, respectively). Specifically, for any player $p$ and any $s \in [T]$ such that $\pi_s(i,p) \le T$,
\[
n_i^p(\pi_{s-1}(i,p))
= 
n_i^p(\pi_{s-1}(i,p)+1)
=
\ldots
=
n_i^p(\pi_s(i,p) - 1)
=s-1,
\]
\[
m_i^p(\pi_{s-1}(i,p))
= 
m_i^p(\pi_{s-1}(i,p)+1)
=
\ldots
=
m_i^p(\pi_s(i,p) - 1)
=
n_i(\pi_{s-1}(i,p))
\]

However, $\cbr{n_i^p(t): t \in [T]}$ and $\cbr{m_i^p(t): t \in [T]}$ do \emph{not} necessarily satisfy the invariant property with respect to $i$. %
Similarly, 
$\cbr{ \indmu_i^p(t-1): t \in [T] }$, $\cbr{ \indvar_i^p(t-1): t \in [T] }$, $\cbr{\aggmu_i^p(t-1): t \in [T] }$, $\cbr{\aggvar_i^p(t-1): t \in [T] }$ all satisfy the invariant property with respect to $(i,p)$.
\end{example}

\begin{example}
\label{ex:mu-var-invariant}
For any arm $i \in [K]$ and any player $p \in [M]$, 
$\cbr{ \hat{\mu}_i^p(t-1): t \in [T] }$ satisfy the invariant property with respect to $(i,p)$.   
This follows from Eq.~\eqref{rem:mu-i-p-decomp}, 
and the above two examples that $\cbr{ \indmu_i^p(t-1): t \in [T] }$, $\cbr{ \aggmu_i^p(t-1): t \in [T] }$, $\cbr{ H_i^p(t): t \in [T] }$ all satisfy the invariant property with respect to $(i,p)$. 

Following a similar reasoning, $\cbr{ \var_i^p(t-1): t \in [T] }$ satisfy the invariant property with respect to $(i,p)$.
\end{example}

\paragraph{Facts about Subpar Arms.} We now present some facts about subpar arms.

\begin{fact}[Properties of subpar arms, see also \citealt{wzsrc21}, Fact 15]
\label{fact:basic_facts_subpar_arms}
The following are true:
\begin{enumerate}
    \item for any $i \in [K]$ and $p,q \in [M]$, $\abr{\Delta_i^p - \Delta_i^q} \le 2\epsilon$ \citep[Fact 14]{wzsrc21}; \label{item:delta_difference}
    
    \item For any $i \in \Ical_{10\epsilon}$ and $p \in [M]$, $\Delta_i^p > 8\epsilon$, which means that $\Delta_i^{\min} > 8\epsilon$. \label{item:Delta_i^p_greater_than_8}
    
    \item $\abr{\Ical_{2\epsilon}^C} \ge 1$; \label{item:I_2eps_C_non_empty}
    
    \item Let $\Delta_i^{\max} = \max_{p \in [M]} \Delta_i^p$. For any $i \in \Ical_{10\epsilon} \subseteq \Ical_{5\epsilon}$, $\Delta_{i}^{\max} \le 2\Delta_i^{\min}$; furthermore, $\frac{1}{\Delta_i^{\min}} \le \frac{2}{M} \sum_{p \in [M]} \frac{1}{\Delta_i^p}$ \citep[Fact 15]{wzsrc21}.
\end{enumerate}

\begin{proof}

For item~\ref{item:Delta_i^p_greater_than_8}, by the definition of $\Ical_{10\epsilon}$, there exists $p$ such that $\Delta_i^p > 10\epsilon$. Then, for all $q \in [M]$, we have $\Delta_i^q > 8\epsilon$ by item~\ref {item:delta_difference}. 

For item~\ref{item:I_2eps_C_non_empty}, using a similar argument, we have, for each $i \in \Ical_{2\epsilon}$ and $p \in [M]$, $\Delta_i^p > 0$. Let $j$ be an optimal for player $1$ such that $\Delta_j^p = 0$. Then $j \notin \Ical_{2\epsilon}$. 
\end{proof}

\end{fact}

\paragraph{Additional notations.}
\begin{itemize}
    \item Denote by $\Phi(x) = \int_{-\infty}^x \frac{1}{\sqrt{2\pi}} e^{-\frac{z^2}{2}} dz$ the cumulative distribution function (CDF) of the standard Gaussian distribution. 
    
    \item Let $\wbar{\Phi}(x) = 1 - \Phi(x) = \int_{x}^\infty \frac{1}{\sqrt{2\pi}} e^{-\frac{z^2}{2}} dz$ denote the complementary CDF of the standard Gaussian distribution.

    \item Denote by $\rbr{z}_+ = z \vee 0$.
    
    \item For any arm $i \in [K]$, player $p \in [M]$ and $t \in [T] \cup \cbr{0}$, let $$\wbar{n_i^p}(t) := n_i^p(t) \vee 1,$$ and $$\wbar{m_i^p}(t) := (m_i^p(t)-M) \vee 1.$$
    
\end{itemize}

\section{Concentration Bounds}
\label{sec:appendix_concentration_bounds}

\subsection{Novel concentration inequality for multi-task data aggregation at random stopping time $\tau_k$'s}

We begin by introducing the following definition.

\begin{definition}[Mixture expected reward at $t$]
For any arm $i \in [K]$ and $t \in [T]$, define
\[
\tilde{\mu}_i(t) = \frac{1}{n_i(t) \vee 1} \sum_{s \le t} \sum_{q \in \Pcal_s} \indic \cbr{ i_s^q = i} \mu_i^q + \epsilon
\]
to be the $\epsilon$-offset mixture expected reward of arm $i$ up to round $t$.
\end{definition}

In what follows, we will consider $\tilde{\mu}_i(\tau_k(i))$ for any $i \in [K]$ and $k \in [TM]$, where the definition of $\tau_k(i)$ can be found in Definition~\ref{def:tau_k}.

\begin{lemma}
\label{lem:agg_concentration_mu_tilde}
For any arm $i \in [K]$ and $k \in [TM]$, denote by $\tau_k = \tau_k(i)$. If $\tau_k \leq T$, then for every player $p \in [M]$, we have
\begin{align*}
\aggmu_i(\tau_k)  - \mu_i^p & \le \aggmu_i(\tau_k) - \tilde{\mu}_i(\tau_k) + 2\epsilon; \text{ and}\\
\mu_i^p - \aggmu_i(\tau_k) & \le \tilde{\mu}_i(\tau_k) - \aggmu_i(\tau_k).
\end{align*}
\end{lemma}

\begin{proof}
For every $t \in [T]$, observe that
$$\tilde{\mu}_i(t) = \frac{1}{n_i(t) \vee 1}\sum_{s \le t}  \sum_{\substack{q \in \Pcal_s: \\ i_s^q = i}} \mu_i^q + \epsilon= \sum_{q \in [M]} \frac{n_i^q(t) \cdot \mu_i^q }{n_i(t) \vee 1} + \epsilon.$$
It can be easily verified that, if $n_i(t) > 0$, for every player $p \in [M]$,
\[
\tilde{\mu}_i(t) - \mu_i^p \le 2\epsilon \ \text{ and } \ \mu_i^p - \tilde{\mu}_i(t) \le 0,
\]
where we note that the asymmetry comes from the additive term $\epsilon$ in $\tilde{\mu}_i(t)$. Therefore, for $k \in [TM]$, if $\tau_k \leq T$, then $n_i(\tau_k) \geq k > 0$ and 
we have
\[
\tilde{\mu}_i(\tau_k) - \mu_i^p \le 2\epsilon \ \text{ and } \ \mu_i^p - \tilde{\mu}_i(\tau_k) \le 0.
\]

It then follows that, for every player $p \in [M]$,
\begin{align*}
\aggmu_i(\tau_k)  - \mu_i^p & \le \aggmu_i(\tau_k) - \tilde{\mu}_i(\tau_k) + 2\epsilon, \text{ and}\\
\mu_i^p - \aggmu_i(\tau_k) & \le \tilde{\mu}_i(\tau_k) - \aggmu_i(\tau_k).
\qedhere
\end{align*}
\end{proof}

We are now ready to present  Lemma~\ref{lem:agg_concentration_tau_k}, our novel concentration bound (see also Lemma~\ref{lem:novel_concentration_main_paper}). 

\begin{lemma}
\label{lem:agg_concentration_tau_k}
For any arm $i \in [K]$ and $k \in [TM] \cup \cbr{0}$, 
denote by $\tau_k = \tau_k(i)$;
for $\delta \in (0,1]$, we have
\begin{align}
\Pr \rbr{ \cbr{\tau_k = T+1} \cup \cbr{ \cbr{\tau_k \le T} \cap \cbr{\forall p \in [M],\ \aggmu_i(\tau_k) - \mu_i^{p} \le \sqrt{\frac{2 \ln \rbr{\frac{2}{\delta}}}{ \rbr{n_i(\tau_k) -M} \vee 1}} + 2 \epsilon}}} & > 1 - \delta; %
\label{eqn:agg_concentration_1} \\
\Pr \rbr{ \cbr{\tau_k = T+1} \cup \cbr{ \cbr{\tau_k \le T} \cap \cbr{\forall p \in [M],\ \mu_i^{p} - \aggmu_i(\tau_k) \le \sqrt{\frac{2 \ln \rbr{\frac{2}{\delta}}}{\rbr{n_i(\tau_k) -M} \vee 1}}}}} & > 1 - \delta. \label{eqn:agg_concentration_2}
\end{align}
\end{lemma}

The following corollary is an equivalent form of Equation~\eqref{eqn:agg_concentration_2}:
\begin{corollary}
\label{col:agg_concentration_alternative}
For any arm $i \in [K]$ and $k \in [TM] \cup \cbr{0}$, 
denote by $\tau_k = \tau_k(i)$.
Equivalently, for any $z \ge 0$, we have
\begin{align}
\Pr \rbr{ \rbr{\tau_k \le T} \wedge \rbr{ \exists p \in [M],\ \mu_i^{p} - \aggmu_i(\tau_k) \geq z \sqrt{\frac{4}{(n_i(\tau_k) - M) \vee 1}}}} \leq 2 e^{-2z^2}.\label{eqn:agg_concentration_alternative}
\end{align}
\end{corollary}
\begin{proof}[Proof of Corollary~\ref{col:agg_concentration_alternative}]
If $z \leq \sqrt{\frac12 \ln 2}$, Equation~\eqref{eqn:agg_concentration_alternative} holds trivially as $2 e^{-2z^2} \geq 1$. 
Otherwise $z > \sqrt{\frac12 \ln 2}$. In this case, let $\delta = 2 e^{-2z^2} \in (0,1]$ in Equation~\eqref{eqn:agg_concentration_2}, and using De Morgan's law, we also obtain Equation~\eqref{eqn:agg_concentration_alternative}.
\end{proof}

\begin{proof}[Proof of Lemma~\ref{lem:agg_concentration_tau_k}]
Fix any arm $i \in [K]$. For $k = 0$, we have $\tau_0 = 0$; both Eq.~\eqref{eqn:agg_concentration_1} and Eq.~\eqref{eqn:agg_concentration_2} hold trivially because for all $p \in [M]$ and $\delta \in (0,1]$, $\abr{\aggmu_i(\tau_0) - \mu_i^p} \le 1 \leq \sqrt{2 \ln 2} \leq \sqrt{2\ln(\frac{2}{\delta})}$.

We now focus on $k \in [TM]$.
By Lemma~\ref{lem:agg_concentration_mu_tilde}, it suffices to show that
\begin{align}
\Pr \rbr{ \cbr{\tau_k = T+1} \cup \cbr{  \cbr{\tau_k \le T} \cap \cbr{ \aggmu_i(\tau_k) - \tilde{\mu}_i(\tau_k) \le \sqrt{\frac{2 \ln \rbr{\frac{2}{\delta}}}{ \rbr{n_i(\tau_k) -M} \vee 1}}}}} & > 1 - \delta; \text { and,} \label{eqn:agg_concentration_mu_tilde_1}\\
\Pr \rbr{ \cbr{\tau_k = T+1} \cup \cbr{ \cbr{\tau_k \le T} \cap  \cbr{ \tilde{\mu}_i(\tau_k) - \aggmu_i(\tau_k) \le \sqrt{\frac{2 \ln \rbr{\frac{2}{\delta}}}{\rbr{n_i(\tau_k) -M} \vee 1}}}}} & > 1 - \delta. \nonumber
\end{align}

To avoid redundancy, we only prove Eq.~\eqref{eqn:agg_concentration_mu_tilde_1}; the other inequality follows by symmetry.

Now, for $t \in [T] \cup \cbr{0}$, consider $Z_t = \sum_{s=1}^t \sum_{p \in \Pcal_s} \indic \cbr{i_s^p = i} (r_s^p - \mu_i^p)$. Furthermore, for $t \in [T] \cup \cbr{0}$ and $\lambda > 0$, let 
\[
w_t(\lambda) = \exp \rbr{\lambda Z_t - n_i(t) \frac{\lambda^2}{8}}.
\] 
We now show that $\cbr{w_t(\lambda)}_{t=0}^T$ is a nonnegative supermartingale with respect to $\cbr{\Fcal_t}_{t=0}^T$ for all $\lambda > 0$. Since $\EE \sbr{ \abr{w_t(\lambda)}} < \infty$ and $w_t(\lambda) \ge 0$ for all $t \in [T] \cup \cbr{0}$, it suffices to show that, for all $t \in [T]$,
\begin{align*}
    & \EE \sbr{w_t(\lambda) \mid \Fcal_{t-1}}  \\
    = & \EE \sbr{ \exp \rbr{  \sum_{s \in [t]} \sum_{p \in \Pcal_s} \indic \cbr{i_s^p = i} \rbr{ \lambda (r_s^p - \mu_i^p) -  \frac{\lambda^2}{8}}} \mid \Fcal_{t-1}} \\
    = & \EE \sbr{ \exp \rbr{ \sum_{s \in [t-1]} \sum_{p \in \Pcal_s} \indic \cbr{i_s^p = i} \rbr{ \lambda (r_s^p - \mu_i^p) -  \frac{\lambda^2}{8}}} \exp \rbr{ \sum_{p \in \Pcal_t} \indic \cbr{i_t^p = i} \rbr{ \lambda (r_t^p - \mu_i^p) -  \frac{\lambda^2}{8}}} \mid \Fcal_{t-1}} \\
    = & \exp \rbr{ \sum_{s \in [t-1]} \sum_{p \in \Pcal_s} \indic \cbr{i_s^p = i} \rbr{ \lambda(r_s^p - \mu_i^p) -  \frac{\lambda^2}{8}}} \EE \sbr{\exp \rbr{ \sum_{p \in \Pcal_t} \indic \cbr{i_t^p = i} \rbr{ \lambda(r_t^p - \mu_i^p) -  \frac{\lambda^2}{8}}} \mid \Fcal_{t-1}} \\
    = & w_{t-1}(\lambda) \cdot \EE \sbr{\exp \rbr{\lambda \sum_{p \in \Pcal_t} \indic \cbr{i_t^p = i} (r_t^p - \mu_i^p)} \exp \rbr{ - \sum_{p \in \Pcal_t} \indic \cbr{i_t^p = i} \frac{\lambda^2}{8}}\mid \Fcal_{t-1}} \\
    \le & w_{t-1}(\lambda),
\end{align*}
where the last inequality uses the law of iterated expectation along with Hoeffding's lemma, i.e.,
\begin{align*}
& \EE\sbr{ 
\exp \rbr{\lambda \sum_{p \in \Pcal_t} \indic \cbr{i_t^p = i} (r_t^p - \mu_i^p)} \cdot 
\exp \rbr{ - \sum_{p \in \Pcal_t} \indic \cbr{i_t^p = i} \frac{\lambda^2}{8}}\mid \Fcal_{t-1} } 
\\
\leq &
\EE\sbr{ \EE\sbr{ 
\exp \rbr{\lambda \sum_{p \in \Pcal_t} \indic \cbr{i_t^p = i} (r_t^p - \mu_i^p)} \mid \Fcal_{t-1}, (i_t^p)_{p \in \Pcal_t}
} \cdot \exp \rbr{ - \sum_{p \in \Pcal_t} \indic \cbr{i_t^p = i} \frac{\lambda^2}{8}} \mid \Fcal_{t-1} } \\
\le &
\EE\sbr{
\prod_{p \in \Pcal_t} \exp \rbr{\frac{\lambda^2 \cdot  \rbr{ \indic \cbr{i_t^p = i}}^2}{8}} \cdot \exp \rbr{ - \sum_{p \in \Pcal_t} \indic \cbr{i_t^p = i} \frac{\lambda^2}{8}} \mid \Fcal_{t-1} }
\leq 1
\end{align*}

Recall from Remark~\ref{rem:tau_k_stopping} that $\tau_k$ is a stopping time with respect to $\cbr{\Fcal_t}_{t=0}^{T}$ and $\tau_k \le T+1 < \infty$ almost surely, it follows that, by the optional sampling theorem, for all $\lambda > 0$,
\begin{align}
    \EE \sbr{ \indic \cbr{\tau_k \le T} \cdot w_{\tau_k}(\lambda) } \le \EE \sbr{w_0(\lambda)} = 1. \label{eqn:optional_sampling_tau_k}
\end{align}

Rewriting Eq.~\eqref{eqn:optional_sampling_tau_k}, we have
\begin{align*}
    \EE \sbr{ \indic \cbr{\tau_k \le T} \cdot \exp \rbr{\lambda Z_{\tau_k} - n_i(\tau_k) \frac{\lambda^2}{8}} } \le 1.
\end{align*}

It then follows that, by Markov's inequality, for any $\delta > 0$,
\begin{align*}
    \Pr \rbr{\indic \cbr{\tau_k \le T} \cdot \exp \rbr{\lambda Z_{\tau_k} - n_i(\tau_k) \frac{\lambda^2}{8}} \ge \frac{1}{\delta}} & \le \frac{ \EE \sbr{ \indic \cbr{\tau_k \le T} \cdot \exp \rbr{\lambda Z_{\tau_k} - n_i(\tau_k) \frac{\lambda^2}{8}} }}{\frac{1}{\delta}} 
     \le \delta;
\end{align*}
therefore,
\[
\Pr \rbr{ \cbr{\tau_k \le T} \cap \cbr{\exp \rbr{\lambda Z_{\tau_k} - n_i(\tau_k) \frac{\lambda^2}{8}} \ge \frac{1}{\delta}}} \le \delta.
\]

Rearranging the terms in the above inequality, we have, for any $\lambda > 0$,
\begin{align*}
    \Pr \rbr{ \cbr{\tau_k = T + 1} \cup \cbr{ \cbr{\tau_k \le T} \cap \cbr{ \frac{1}{n_i(\tau_k)} Z_{\tau_k} - \frac{\lambda}{8} < \frac{\ln \rbr{\frac{1}{\delta}}}{n_i(\tau_k) \cdot \lambda}}}} > 1 - \delta,
\end{align*}
where we use the elementary fact that for sets $A$ and $B$, $\neg (A \cap B) = \neg A \cup \rbr{A \cap \neg B}$.

Choosing $\lambda = \sqrt{\frac{\ln (\frac{1}{\delta})}{k}}$ and using the fact that $n_i(\tau_k) \ge k$, we have
\begin{align*}
    \Pr \rbr{ \cbr{\tau_k = T + 1} \cup \cbr{\cbr{\tau_k \le T} \cap \cbr{ \frac{1}{n_i(\tau_k)} Z_{\tau_k} < \sqrt{\frac{2\ln (\frac{1}{\delta})}{k}} }}} > 1 - \delta;
\end{align*}
it then follows that
\begin{align}
    \Pr \rbr{ \cbr{\tau_k = T + 1} \cup \cbr{ \cbr{\tau_k \le T} \cap \cbr{ \frac{1}{n_i(\tau_k)} Z_{\tau_k} < \sqrt{\frac{2\ln (\frac{2}{\delta})}{k}} }}} > 1 - \delta. \label{eqn:agg_concentration_k}
\end{align}

We now consider two cases:
\begin{enumerate}
    \item $n_i(\tau_k) \le M$. We have $\frac{1}{n_i(\tau_k)}Z_{\tau_k} \le 1 < \sqrt{\frac{2\ln(\frac{2}{\delta})}{(n_i(\tau_k) - M) \vee 1}} = \sqrt{2\ln(\frac{2}{\delta})}$ trivially for $\delta \in (0,1]$.
    
    \item $n_i(\tau_k) \ge M + 1$. Since $k \ge n_i(\tau_k) - M$, we have $\sqrt{\frac{2\ln(\frac{2}{\delta})}{k}} \le \sqrt{\frac{2\ln(\frac{2}{\delta})}{n_i(\tau_k) - M}} = \sqrt{\frac{2\ln(\frac{2}{\delta})}{(n_i(\tau_k) - M) \vee 1}}$.
\end{enumerate}

Eq.~\eqref{eqn:agg_concentration_mu_tilde_1} then follows from Eq.~\eqref{eqn:agg_concentration_k} and the elementary fact that $A \subseteq B$ if $(A \cap C) \subseteq B$ and $(A \cap \neg C) \subseteq B$.
This completes the proof.
\end{proof}

\subsection{Other concentration bounds}

Recall the definition of stopping times $\pi_k(i,p)$ for any arm $i$ and player $p$ (see Definition~\ref{def:pi_k}). 
\begin{lemma}
\label{lem:ind_concentration_pi_k}
For any $i \in [K]$, $p \in [M]$, $k \in [T] \cup \cbr{0}$, and $\delta \in (0,1]$, we have
\begin{align}
\Pr \rbr{ \cbr{\pi_k(i,p) = T+1} \cup \cbr{ \cbr{\pi_k(i,p) \le T} \cap \cbr{ \bigg\lvert \indmu_i^p(\pi_k(i,p)) - \mu_i^{p} \bigg\rvert \le \sqrt{\frac{2 \ln \rbr{\frac{4}{\delta}}}{ n_i^p(\pi_k(i,p)) \vee 1}} }}} & > 1 - \delta.
\label{eqn:ind_concentration_1}
\end{align}
\end{lemma}

\begin{corollary}
\label{col:ind_concentration_alternative}
For any $i \in [K]$, $p \in [M]$, $k \in [T] \cup \cbr{0}$, and $z \ge 0$,
we have
\begin{equation}
\Pr \rbr{ \rbr{\pi_k(i,p) \le T} \wedge \rbr{ \bigg\lvert  \mu_i^{p} - \indmu_i^p(\pi_k(i,p)) \bigg\rvert \geq z \sqrt{\frac{4}{n_i^p(\pi_k(i,p)) \vee 1}}}} \leq 4 e^{-2z^2}.
\label{eqn:ind_concentration_alternative}
\end{equation}
\end{corollary}
\begin{proof}[Proof of Corollary~\ref{col:ind_concentration_alternative}]
If $z \leq \sqrt{\frac12 \ln 4}$, Equation~\eqref{eqn:ind_concentration_alternative} holds trivially as $4 e^{-2z^2} \geq 1$. 
Otherwise $z > \sqrt{\frac12 \ln 4}$. In this case, let $\delta = 4 e^{-2z^2} \in (0,1]$ in Equation~\eqref{eqn:ind_concentration_1}, and using De Morgan's law, we also obtain Equation~\eqref{eqn:ind_concentration_alternative}.
\end{proof}

\begin{proof}[Proof of Lemma~\ref{lem:ind_concentration_pi_k}]
The proof of Lemma~\ref{lem:ind_concentration_pi_k} is largely similar to the one for Lemma~\ref{lem:agg_concentration_tau_k}. Therefore, we omit some details here to avoid redundancy. See the proof of Lemma~\ref{lem:agg_concentration_tau_k} for full details.

Let us fix any arm $i \in [K]$ and player $p \in [M]$.
Throughout this proof, 
to ease the exposition, we use $\pi_k$ to denote $\pi_k(i,p)$.

We first observe that when $k = 0$, we have $\pi_k = 0$, $\indmu_i^p(0) = 0$, and $n_i^p(0) = 0$. It follows that $\bigg\lvert \indmu_i^p(\pi_k) - \mu_i^{p} \bigg\rvert \le 1 \le \sqrt{2\ln \rbr{\frac 4 \delta}}$ trivially.

It then suffices to only consider the case when $k \in [T]$. Note that $n_i^p(\pi_k) = k \ge 1$. We will show that
\begin{align}
\Pr \rbr{ \cbr{\pi_k = T+1} \cup \cbr{ \cbr{\pi_k \le T} \cap \cbr{ \indmu_i^p(\pi_k) - \mu_i^{p} \le \sqrt{\frac{2 \ln \rbr{\frac{2}{\delta}}}{ n_i^p(\pi_k) }} }}} & > 1 - \delta.
\label{eqn:ind_concentration_pi_direction1}
\end{align}

For $t \in [T]\cup \cbr{0}$, let $X_t =  \sum_{s \in [t]} \indic \cbr{p \in \Pcal_s, i_s^p = i} (r_s^p - \mu_i^p)$; and for $\lambda > 0$, further define $\xi_t(\lambda) = \exp \rbr{\lambda X_t - n_i^p(t) \frac{\lambda^2}{8}}$.
It can be verified that $\cbr{\xi_t(\lambda)}_{t=0}^T$ is a nonnegative supermartingale with respect to $\cbr{\Fcal_t}_{t=0}^T$ for all $\lambda > 0$:
\begin{enumerate}
    \item $\EE \sbr{ \abr{\xi_t(\lambda)}} < \infty$ for all $t \in [T] \cup \cbr{0}$;
    
    \item $\xi_t(\lambda) \ge 0$ for all $t \in [T] \cup \cbr{0}$;
    
    \item $\EE \sbr{\xi_t(\lambda) \mid \Fcal_{t-1}} \le \xi_{t-1}(\lambda)$ for all $t \in [T]$.
\end{enumerate}
Item 3 is true because
\begin{align*}
    &\ \EE \sbr{\xi_t(\lambda) \mid \Fcal_{t-1}} \\
    = &\ \exp \rbr{ \sum_{s=1}^{t-1} \indic \cbr{p \in \Pcal_s, i_s^p = i} \rbr{ \lambda(r_s^p - \mu_i^p) -  \frac{\lambda^2}{8}}} \EE \sbr{\exp \rbr{ \indic \cbr{p \in \Pcal_t, i_t^p = i} \rbr{ \lambda(r_t^p - \mu_i^p) -  \frac{\lambda^2}{8}}} \mid \Fcal_{t-1}} \\
    = &\ \xi_{t-1}(\lambda) \cdot \EE \sbr{\exp \rbr{\lambda \cdot \indic \cbr{p \in \Pcal_t, i_t^p = i} (r_t^p - \mu_i^p)} \exp \rbr{ - \indic \cbr{p \in \Pcal_t, i_t^p = i} \frac{\lambda^2}{8}}\mid \Fcal_{t-1}} \\
    = & \xi_{t-1}(\lambda) \cdot \EE\sbr{ \EE \sbr{\exp \rbr{\lambda \cdot \indic \cbr{p \in \Pcal_t, i_t^p = i} (r_t^p - \mu_i^p)} \mid \Fcal_{t-1}, i_t^p  }  \cdot \exp \rbr{ - \indic \cbr{p \in \Pcal_t, i_t^p = i} \frac{\lambda^2}{8}}   \mid \Fcal_{t-1}  } \\
    \le &\ \xi_{t-1}(\lambda),
\end{align*}
where we use the law of total expectation, the observation that $\xi_{t-1}(\lambda)$ is $\Fcal_{t-1}$-measurable, and Hoeffding's Lemma. %

Recall from Remark~\ref{rem:pi_k_stopping} that $\pi_k$ is a stopping time with respect to $\cbr{\Fcal_t}_{t=0}^{T}$ and $\pi_k \le T+1 < \infty$ almost surely. Then, by the optional sampling theorem, for all $\lambda > 0$,
\begin{align}
    \EE \sbr{ \indic \cbr{\pi_k \le T} \cdot \xi_{\pi_k}(\lambda) } \le \EE \sbr{\xi_0(\lambda)} = 1. \label{eqn:optional_sampling_pi_k}
\end{align}
In other words,
\begin{align*}
    \EE \sbr{ \indic \cbr{\pi_k \le T} \cdot \exp \rbr{\lambda X_{\pi_k} - n_i^p(\pi_k) \frac{\lambda^2}{8}} } \le 1.
\end{align*}

By Markov's inequality, we have
\begin{align*}
    \Pr \rbr{\indic \cbr{\pi_k \le T} \cdot \exp \rbr{\lambda X_{\pi_k} - n_i^p(\pi_k) \frac{\lambda^2}{8}} \ge \frac{1}{\delta}} \le \delta;
\end{align*}
and thus,
\[
\Pr \rbr{ \cbr{\pi_k \le T} \cap \cbr{\exp \rbr{\lambda X_{\pi_k} - n_i^p(\pi_k) \frac{\lambda^2}{8}} \ge \frac{1}{\delta}}} \le \delta.
\]

Using the elementary fact that for sets $A$ and $B$, $\neg (A \cap B) = \neg A \cup \rbr{A \cap \neg B}$, we have, for any $\lambda > 0$,
\begin{align*}
    \Pr \rbr{ \cbr{\pi_k = T + 1} \cup \cbr{ \cbr{\pi_k \le T} \cap \cbr{ \frac{1}{n_i^p(\pi_k)} X_{\pi_k} - \frac{\lambda}{8} < \frac{\ln \rbr{\frac{1}{\delta}}}{n_i^p(\pi_k) \cdot \lambda}}}} > 1 - \delta,
\end{align*}
where we slightly rearrange the terms.

Choose $\lambda = \sqrt{\frac{\ln (\frac{1}{\delta})}{k}}$ and observe that $n_i^p(\pi_k) = k$. It follows that
\begin{align*}
    \Pr \rbr{ \cbr{\pi_k = T + 1} \cup \cbr{\cbr{\pi_k \le T} \cap \cbr{ \frac{1}{n_i^p(\pi_k)} X_{\pi_k} < \sqrt{\frac{2\ln (\frac{1}{\delta})}{n_i^p(\pi_k)}} }}} > 1 - \delta.
\end{align*}

Eq.~\eqref{eqn:ind_concentration_pi_direction1} follows trivially by the observation that $\ln (\frac{2}{\delta}) > \ln (\frac{1}{\delta})$. By symmetry, it can also be shown that the following inequality is true:
\begin{align*}
\Pr \rbr{ \cbr{\pi_k = T+1} \cup \cbr{ \cbr{\pi_k \le T} \cap \cbr{ \mu_i^{p} - \indmu_i^p(\pi_k) \le \sqrt{\frac{2 \ln \rbr{\frac{2}{\delta}}}{ n_i^p(\pi_k)}} }}} & > 1 - \delta.
\end{align*}
The proof is then completed by applying the union bound.
\end{proof}

\begin{definition}
\label{def:event_e_tau_pi}
For any $\delta \in (0,1]$, let
\begin{align*}
& E_{\text{agg}}(\delta) = \Vast\{ 
\forall i \in [K], \forall k \in [TM] \cup \cbr{0}, \rbr{\tau_k(i) = T + 1} \vee \vast( \rbr{\tau_k(i) \le T} \wedge \\         
& \rbr{\forall p \in [M],\ \aggmu_i(\tau_k(i)) - \mu_i^{p} \le \sqrt{\frac{2 \ln \rbr{\frac{2}{\delta}}}{ \rbr{n_i(\tau_k(i)) -M} \vee 1}} + 2 \epsilon, \mu_i^{p} - \aggmu_i(\tau_k(i)) \le \sqrt{\frac{2 \ln \rbr{\frac{2}{\delta}}}{\rbr{n_i(\tau_k(i)) -M} \vee 1}}} \vast)
\Vast\},
\end{align*}
and
\begin{align*}
E_{\text{ind}}(\delta) = \vast\{
\forall i \in [K], \forall p \in [M], & \forall k \in [T] \cup \cbr{0}, \rbr{\pi_k(i,p) = T + 1} \vee \\
& \rbr{ \rbr{\pi_k(i,p) \le T} \wedge \rbr{ \bigg\lvert\indmu_i^p(\pi_k(i,p)) - \mu_i^p \bigg\rvert \le \sqrt{\frac{2 \ln (\frac{4}{\delta})}{n_i^p(\pi_k(i,p)) \vee 1}}}} \vast\}.
\end{align*}
Furthermore, let
\[
E(\delta) = E_{\text{agg}}(\delta) \cap E_{\text{ind}}(\delta).
\]
\end{definition}

\begin{corollary}
\label{cor:E}
For $\delta \in (0,1]$,
\[
\Pr (E(\delta)) \ge 1 - 6T^3\delta.
\]
\end{corollary}

\begin{proof}
By the union bound, Lemma~\ref{lem:agg_concentration_tau_k}, Lemma~\ref{lem:ind_concentration_pi_k}, and the assumption that $T \ge \max(K, M)$, we have
\begin{align*}
    \Pr (E_{\text{agg}}(\delta)) & \ge 1 - K(TM+1) (2\delta) \ge 1 - 4T^3\delta. \\
    \Pr (E_{\text{ind}}(\delta)) & \ge 1 - KM(T+1) \delta \ge 1 - 2T^3\delta.
\end{align*}
The corollary then follows by the union bound.
\end{proof}

\subsection{Clean Event}

We now define our notion of ``clean'' event for each $t$.
\begin{definition}
\label{def:event_e}
For any $t \in [T+1]$, let
\begin{align*}
\Ecal_t = \vast\{
\forall p \in [M],
\forall i \in [K],
\ \Big\lvert{\indmu_i^p(t-1) - \mu_i^p} \Big\rvert & \le \sqrt{\frac{ 10 \ln T}{ \wbar{n_i^p}(t-1)}}, \\
\aggmu_i^p(t-1) - \mu_i^p & \le \sqrt{\frac{ 10 \ln T}{ \wbar{m_i^p}(t-1) } } + 2\epsilon, \\
\mu_i^p - \aggmu_i^p(t-1) & \le \sqrt{\frac{ 10 \ln T}{  \wbar{m_i^p}(t-1) }}
\qquad \quad \vast\},
\end{align*}
where we recall that $\wbar{n_i^p}(t-1) = n_i^p(t-1) \vee 1$, $\wbar{m_i^p}(t-1) = (m_i^p(t-1) - M) \vee 1$. 
Furthermore, let $\wbar{\Ecal_t}$ denote the complement of $\Ecal_t$.
\end{definition}

The following lemma shows that the clean event happens with high probability.
\begin{lemma}
\label{lem:clean_event}
\[
\Pr (\Ecal_t) > 1 - \frac{24}{T^2}.
\]
\end{lemma}

\begin{proof}[Proof]

The proof of Lemma~\ref{lem:clean_event} follows from Corollary~\ref{cor:E}.
It suffices to show that, for any $t$, $E(\frac{4}{T^5}) \subseteq \Ecal_t$. To this end, we will show that if $E(\frac{4}{T^5})$ happens, then $\Ecal_t$ must happen. 

For any $t \in [T+1]$, $i \in [K]$, $p \in [M]$, let $u = u_i^p(t-1)$ be the round in which player $p$ last pulls arm $i$ (see Definition~\ref{def:last_pulled_round}). In addition, let $s = n_i^p(u) \in ([T]\cup \cbr{0})$ and $k = n_i(u) \in  ([TM]\cup \cbr{0})$.
Note that $\pi_s(i,p) = u \le T$ and $\tau_k(i) = u \le T$.

It then follows by definition that,
\begin{align*}
    \indmu_i^p(t-1) = \indmu_i^p(\pi_s(i,p)), & \quad n_i^p(t-1) = n_i^p(\pi_s(i,p)); \\
    \aggmu_i^p(t-1) = \aggmu_i(\tau_k(i)), & \quad m_i^p(t-1) = n_i(\tau_k(p)).
\end{align*}
The proof is then completed straightforwardly by the definition of $E(\frac{4}{T^5})$, which indicates that for all $s \in [T]\cup \cbr{0}$ and $k \in  [TM]\cup \cbr{0}$,
\begin{align*}
    \Big\lvert{ \indmu_i^p(\pi_s(i,p)) - \mu_i^p} \Big\rvert & \le \sqrt{\frac{ 10 \ln T}{ n_i^p(\pi_s(i,p)) \vee 1}}, \\
    \aggmu_i(\tau_k(i)) - \mu_i^p & \le \sqrt{\frac{ 10 \ln T}{ \rbr{n_i(\tau_k(p)) -M} \vee 1}} + 2\epsilon, \text{ and}\\
    \mu_i^p - \aggmu_i(\tau_k(i)) & \le \sqrt{\frac{ 10 \ln T}{ \rbr{n_i(\tau_k(p)) -M} \vee 1}}.
    \qedhere
\end{align*}
\end{proof}

\newpage
\section{Proofs of Theorem~\ref{thm:ts_gap_dep_ub} and Theorem~\ref{thm:ts_gap_indep_ub}}
\label{sec:appendix_main_proofs}

The following lemmas are central to our proofs of Theorem~\ref{thm:ts_gap_dep_ub} and Theorem~\ref{thm:ts_gap_indep_ub}. In Section~\ref{sec:appendix_subpar_arms_proof}, we prove Lemma~\ref{lem:i-eps-robustagg-ts}. In Section~\ref{sec:appendix_nonsubpar_arms}, we prove Lemma~\ref{lem:i-eps-c-robustagg-ts}. We then conclude our proofs in Section~\ref{sec:appendix_concluding_proofs}.

\begin{lemma}[Subpar arms]
For any arm $i \in \Ical_{10\epsilon}$,
\[
\EE \sbr{n_i(T)} \le \Ocal \rbr{\frac{\ln T}{(\Delta_i^{\min})^2} + M},
\]
where we recall that $\Delta_i^{\min} = \min_{p \in [M]} \Delta_i^p$.
\label{lem:i-eps-robustagg-ts}
\end{lemma}

\begin{lemma}[Non-subpar arms]
For any arm $i \in \Ical_{10\epsilon}^C$ and player $p \in [M]$,
\[
\EE \sbr{n_i^p(T)} \le \Ocal \rbr{\frac{\ln T}{(\Delta_i^p)^2} + M}.
\]
\label{lem:i-eps-c-robustagg-ts}
\end{lemma}

Our analysis in the following Section~\ref{sec:appendix_subpar_arms_proof} and Section~\ref{sec:appendix_nonsubpar_arms} involve various proofs by cases. Figure~\ref{fig:case_division} provides an overview of the case division rules used in our analysis.

\begin{figure}
    \begin{subfigure}{.49\textwidth}
        \centering
        \includegraphics[height=0.72\linewidth]{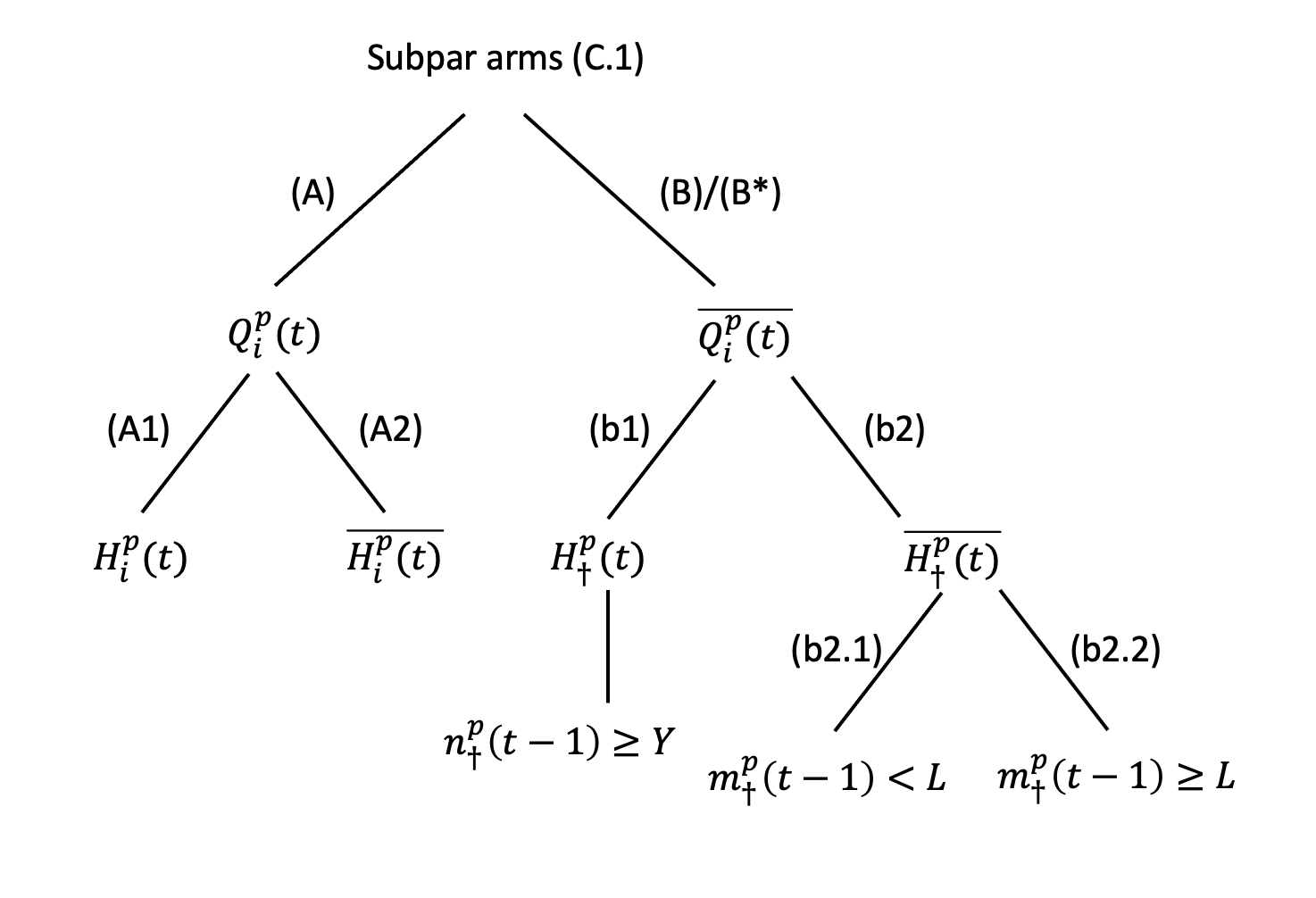}
        \caption{Subpar arms (Section~\ref{sec:appendix_subpar_arms_proof})}
    \end{subfigure}
    \begin{subfigure}{.49\textwidth}
        \centering
        \includegraphics[height=0.7\linewidth]{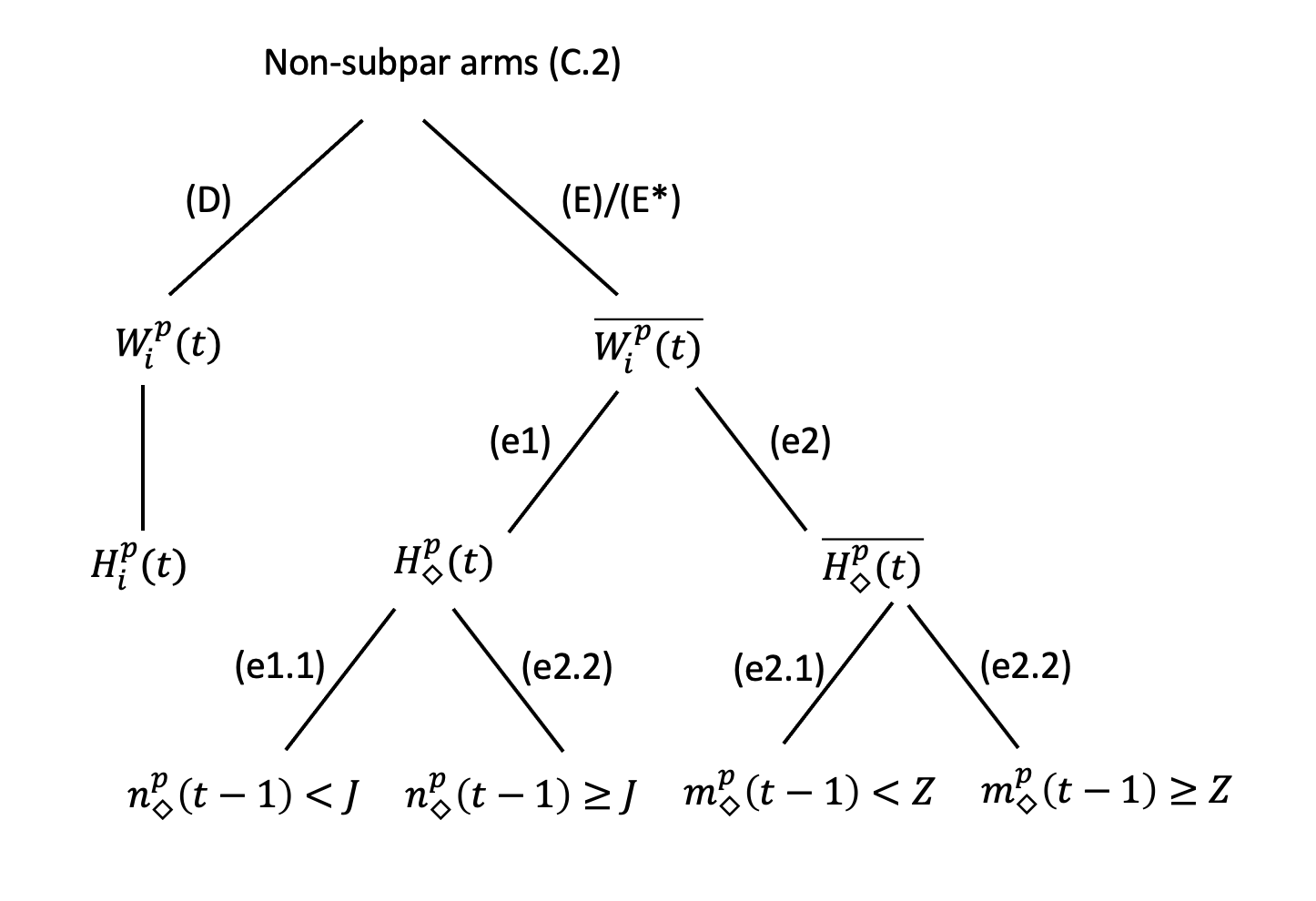}
        \caption{Non-subpar arms (Section~\ref{sec:appendix_nonsubpar_arms})}
    \end{subfigure}
    \caption{Illustrations of the case division rules used in the proofs of Theorem~\ref{thm:ts_gap_dep_ub} and Theorem~\ref{thm:ts_gap_indep_ub}, respectively. Formal definitions of the notions used in the figure can be found in Section~\ref{sec:appendix_preliminaries}, Section~\ref{sec:appendix_subpar_arms_proof} and Section~\ref{sec:appendix_nonsubpar_arms}.}
    \label{fig:case_division}
\end{figure}

\subsection{Subpar Arms}
\label{sec:appendix_subpar_arms_proof}

In this section, we prove Lemma~\ref{lem:i-eps-robustagg-ts}.

Fix any subpar arm $i \in \Ical_{10\epsilon}$ and an arm $\dagger \in \Ical_{2\epsilon}^C$. See Fact~\ref{fact:basic_facts_subpar_arms} for the existence of such an arm. We first consider the following definitions.

\begin{definition}
For any arm $i \in \Ical_{10\epsilon}$ and any player $p$, let 
\[
\delta_i^p = \mu_{\dagger}^p - \mu_i^p > 0.
\]
\end{definition}

\begin{fact}
\label{fact:small_delta}
For any $i \in \Ical_{10\epsilon}$ and player $p \in [M]$,
\[
\frac34 \Delta_i^p < \delta_i^p \le \Delta_i^p.
\]
\end{fact}

\begin{proof}
For any player $p \in [M]$, since $\dagger \in \Ical_{2\epsilon}^C$, we have $\Delta_{\dagger}^p = \mu_*^p - \mu_{\dagger}^p \le 2\epsilon$ by the definition of $\Ical_{2\epsilon}^C$. Furthermore, for any $i \in \Ical_{10\epsilon}$, $\Delta_i^p = \mu_*^p - \mu_i^p > 8\epsilon$. Therefore, we have
\begin{enumerate}
    \item $\delta_i^p = \mu_{\dagger}^p - \mu_i^p \le \mu_*^p - \mu_i^p = \Delta_i^p$;
    
    \item Note that $\frac{\mu_*^p - \mu_{\dagger}^p}{\mu_*^p - \mu_i^p} \leq \frac{2\epsilon}{8\epsilon} \leq \frac 1 4$. 
    This implies that 
    $\frac{\delta_i^p}{\Delta_i^p} = 1 - \frac{\mu_*^p - \mu_{\dagger}^p}{\mu_*^p - \mu_i^p} \geq \frac 3 4$.
    \qedhere
\end{enumerate}
\end{proof}

\begin{definition}
For any player $p$, let $y_i^p = \mu_i^p + \frac12 \delta_i^p$ be a threshold;
in any round $t$, further define
\[
Q_i^p(t) = \cbr{\theta_i^p(t) > y_i^p}
\]
to be the event that the sample $\theta_i^p(t)$ from the posterior distribution associated with arm $i$ and player $p$ in round $t$ is greater than the threshold $y_i^p$. 
In addition, let $\wbar{Q_i^p(t)} = \cbr{\theta_i^p(t) \le y_i^p}$.
\end{definition}

\subsubsection{Subpar Arms---Decomposition}
We can then decompose $\EE \sbr{n_i(T)}$ as follows.

\begin{align}
\EE \sbr{n_i(T)} 
= &
\EE \sbr{\sum_{t \in [T]} \sum_{p \in \Pcal_t} \indic \cbr{i_t^p = i}}
\nonumber
\\
\leq &  \EE \sbr{\sum_{t \in [T]} \sum_{p \in \Pcal_t} \indic \cbr{i_t^p = i, Q_i^p(t), \Ecal_t}} +
\EE \sbr{\sum_{t \in [T]} \sum_{p \in \Pcal_t} \indic \cbr{i_t^p = i, \wbar{Q_i^p(t)}, \Ecal_t}}
+
\EE \sbr{\sum_{t \in [T]} \sum_{p \in \Pcal_t} \indic \cbr{ \wbar{\Ecal_t} }}
\nonumber
\\
\leq & 
\underbrace{\EE \sbr{\sum_{t \in [T]} \sum_{p \in \Pcal_t} \indic \cbr{i_t^p = i, Q_i^p(t), \Ecal_t}}}_{(A)} +
\underbrace{\EE \sbr{\sum_{t \in [T]} \sum_{p \in \Pcal_t} \indic \cbr{i_t^p = i, \wbar{Q_i^p(t)}, \Ecal_t}}}_{(B)}
+ \order\rbr{1},
\label{eqn: subpar_first_decomposition}
\end{align}
where the second inequality follows from Lemma~\ref{lem:clean_event}.
In the following two subsections, we bound term $(A)$ and $(B)$, respectively.

\subsubsection{Bounding Term $(A)$}
The following lemma provides an upper bound on term $(A)$.
\begin{lemma}
\label{lem:bounding_term_A_appendix}
\begin{align}
(A) \le \Ocal \rbr{\frac{\ln T}{(\Delta_i^{\min})^2} + M},
\end{align}
where we recall that $\Delta_i^{\min} = \min_{p \in [M]} \Delta_i^p$.
\end{lemma}

\begin{proof}[Proof of Lemma~\ref{lem:bounding_term_A_appendix}]

Recall the definition of $\Ecal_t$ in Definition~\ref{def:event_e} and the definition of $H_i^p(t)$ in Definition~\ref{def:alg_criteron}, we have
\begin{align*}
(A) =
\underbrace{ \sum_{t \in [T]} \sum_{p \in \Pcal_t} \EE \sbr{ \indic \cbr{i_t^p = i, Q_i^p(t), \Ecal_t, H_i^p(t)}}}_{(A1)} 
+
\underbrace{ \sum_{t \in [T]} \sum_{p \in \Pcal_t} \EE \sbr{ \indic \cbr{i_t^p = i, Q_i^p(t), \Ecal_t, \wbar{H_i^p(t)}}}}_{(A2)}.
\end{align*}

We first consider term $(A1)$. 
Recall that, for simplicity, we let $\wbar{n_i^p}(t-1)$ denote $n_i^p(t-1) \vee 1$; also recall that $\wbar{\Phi}(\cdot)$ is the complementary CDF of the standard Gaussian distribution, and $\rbr{z}_+ = z \vee 0$.
We have
\begin{align*}
(A1) 
& \le \sum_{t \in [T]} \sum_{p \in \Pcal_t} \EE \sbr{ \indic \cbr{Q_i^p(t), \Ecal_t, H_i^p(t)}} \\
& = \sum_{t \in [T]} \sum_{p \in \Pcal_t} \EE \sbr{ \EE \sbr{\indic \cbr{Q_i^p(t), \Ecal_t, H_i^p(t)} \mid \Fcal_{t-1}}} \\
& = \sum_{t \in [T]} \sum_{p \in \Pcal_t} \EE \sbr{ \indic \cbr{\Ecal_t, H_i^p(t)} \cdot \EE \sbr{\indic \cbr{\theta_i^p(t) > y_i^p} \mid \Fcal_{t-1}}} \\
& = 
\sum_{t \in [T]} \sum_{p \in \Pcal_t} \EE \sbr{ \indic \cbr{\Ecal_t, H_i^p(t)} \cdot \wbar{\Phi}\rbr{  \sqrt{ \wbar{n_i^p}(t-1) / 4 } \rbr{y_i^p - \indmu_i^p(t-1) } } } \\
& \leq \sum_{t \in [T]} \sum_{p \in \Pcal_t} \EE \sbr{ \indic \cbr{\Ecal_t, H_i^p(t)} \cdot \exp \rbr{ - \frac{  n_i^p(t-1) (y_i^p - \indmu_i^p(t-1) )_+^2 }{ 8 } } }
\\
& \leq \sum_{t \in [T]} \sum_{p \in \Pcal_t} \EE \sbr{ \indic \cbr{\Ecal_t, H_i^p(t)} \cdot \exp \rbr{ - \frac{  n_i^p(t-1) (\mu_i^p + \frac38 \Delta_i^p - \mu_i^p - \frac1{16}\Delta_i^p )_+^2 }{ 8 } } } \\
&  \leq \sum_{t \in [T]} \sum_{p \in \Pcal_t} \EE \sbr{ \indic \cbr{\Ecal_t, H_i^p(t)} \cdot \exp \rbr{ - \frac{ n_i^p(t-1) (\Delta_i^p)^2 }{ 8(16) } } } \\
& \leq \sum_{t \in [T]} \sum_{p \in \Pcal_t} \frac{1}{T^2} = \order\rbr{ 1 }.
\end{align*}
where
the first inequality drops the indicator $\indic\cbr{i_t^p = i}$; 
the first equality uses the law of total expectation;
the second equality follows from the observation that $\Ecal_t$ and $H_i^p(t)$ are $\Fcal_{t-1}$-measurable;
the third equality follows from the observation that when $H_i^p(t)$ happens,
$ \EE \sbr{\indic \cbr{\theta_i^p(t) > y_i^p} \mid \Fcal_{t-1}} = \PP \rbr{ \theta_i^p(t) > y_i^p \mid \Fcal_{t-1}} = \wbar{\Phi}\rbr{ \frac{ y_i^p - \indmu_i^p(t-1) }{ \sqrt{ 4 / \wbar{n_i^p}(t-1) } } }$;
the second inequality is from Lemma~\ref{lem:phi-bounds} and that $\wbar{n_i^p}(t-1) \ge n_i^p(t-1)$;
the third inequality follows from the facts that when $\Ecal_t$ and $H_i^p(t)$ happen, 
\begin{enumerate}
    \item $\wbar{n_i^p}(t-1) \ge n_i^p(t-1) \ge \frac{40\ln T}{\epsilon^2} \ge \frac{2560\ln T}{(\Delta_i^p)^2}$ (see Fact~\ref{fact:basic_facts_subpar_arms}), 
    
    \item $\indmu_i^p(t-1) \leq \mu_i^p + \sqrt{\frac{10 \ln T}{\wbar{n_i^p}(t-1)}} \leq \mu_i^p + \frac{1}{16} \Delta_i^p$ (see Definition~\ref{def:event_e}), and
    
    \item $y_i^p = \mu_i^p + \frac12 \delta_i^p > \mu_i^p + \frac 3 8 \Delta_i^p$ (see Fact~\ref{fact:small_delta});
\end{enumerate}
the fourth inequality is by algebra; 
and the fifth inequality again uses the observation that when $H_i^p(t)$ happens, $n_i^p(t-1) \ge \frac{2560\ln T}{(\Delta_i^p)^2}$.

We now turn our attention to term $(A2)$. With foresight, let $l = \frac{ 10240 \ln T}{\rbr{\Delta_i^{\min}}^2} + M$. We have
\begin{align}
(A2)
& = \sum_{t \in [T]} \sum_{p \in \Pcal_t} \EE \sbr{ \indic \cbr{i_t^p = i, Q_i^p(t), \Ecal_t, \wbar{H_i^p(t)}}} \nonumber \\
& \le \sum_{t \in [T]} \sum_{p \in \Pcal_t} \EE \sbr{ \indic \cbr{i_t^p = i, Q_i^p(t), \Ecal_t, \wbar{H_i^p(t)}, m_i^p(t-1) < l}} \nonumber \\
& \hspace{180pt} + \sum_{t \in [T]} \sum_{p \in \Pcal_t} \EE \sbr{ \indic \cbr{i_t^p = i, Q_i^p(t), \Ecal_t, \wbar{H_i^p(t)}, m_i^p(t-1) \ge l}} \nonumber \\
& \le (l + M) + \sum_{t \in [T]} \sum_{p \in \Pcal_t} \EE \sbr{ \indic \cbr{i_t^p = i, Q_i^p(t), \Ecal_t, \wbar{H_i^p(t)}, m_i^p(t-1) \ge l}}. \label{eqn:A2_m_decomposition}
\end{align}

To see why Eq.~\eqref{eqn:A2_m_decomposition} is true, it suffices to show that, with probability $1$,
\[
\sum_{t \in [T]} \sum_{p \in \Pcal_t} \indic \cbr{i_t^p = i, m_i^{p}(t-1) < l }
\leq l + M.
\]
Indeed, let us define $\iota = \min \cbr{ t: n_i(t) = \sum_{s \in [t]} \sum_{p \in \Pcal_s} \indic \cbr{i_s^p = i} \geq l }$. The above summation can be simplified as
\begin{align*}
& \sum_{t = 1}^T \sum_{p \in \Pcal_t} \indic \cbr{i_t^p = i, m_i^{p}(t-1) < l } \\
= & \sum_{t=1}^{\iota-1} \sum_{p \in \Pcal_t} \indic \cbr{i_t^p = i, m_i^{p}(t-1) < l } 
+ 
\sum_{t=\iota}^{T} \sum_{p \in \Pcal_t} \indic \cbr{i_t^p = i, m_i^{p}(t-1) < l }  \\
\leq & 
\sum_{t=1}^{\iota-1} \sum_{p \in \Pcal_t} \indic \cbr{i_t^p = i }  
+ 
\sum_{p \in [M]} \sum_{t \geq \iota: p \in \Pcal_t} \indic \cbr{i_t^p = i, m_i^{p}(t-1) < l } \\
\leq & 
(l - 1) + M,
\end{align*}

where the $\sum_{p \in [M]} \sum_{t \geq \iota: p \in \Pcal_t} \indic \cbr{i_t^p = i, m_i^{p}(t-1) < l } \leq M$ follows from the observation that, 
once the total number of pulls of arm $i$ by all players has reached $l$, 
any player $p$ cannot pull arm $i$ more than once before the aggregate number of pulls of $i$ maintained by $p$ is updated to a value $\ge l$ (see Definition~\ref{def:m_i^p}).

\begin{remark}
Eq.~\eqref{eqn:A2_m_decomposition} can also be deducted from the more general  Lemma~\ref{lem:auxiliary_sum_f_indic} in Section~\ref{sec:appendix_auxiliary}, by taking $f_t^p = 1$ for all $t,p$.
\end{remark}

Now, recall that we denote $\rbr{ m_i^p(t-1) -M } \vee 1$ by  $\wbar{m_i^p}(t-1)$.
And again, recall that $\wbar{\Phi}(\cdot)$ is the complementary CDF of the standard Gaussian distribution, and $\rbr{z}_+ = z \vee 0$.
It follows from Eq.~\eqref{eqn:A2_m_decomposition} that
\begin{align*}
(A2) 
& \le \rbr{l + M} + \sum_{t \in [T]} \sum_{p \in \Pcal_t} \EE \sbr{ \EE \sbr{ \indic \cbr{Q_i^p(t), \Ecal_t, \wbar{H_i^p(t)}, m_i^p(t-1) \ge l} \mid \Fcal_{t-1}}  }\\
& = \rbr{l + M} + \sum_{t \in [T]} \sum_{p \in \Pcal_t} \EE \sbr{ \indic \cbr{\Ecal_t, \wbar{H_i^p(t)}, m_i^p(t-1) \ge l} \EE \sbr{ \indic \cbr{ \theta_i^p(t) > y_i^p } \mid \Fcal_{t-1}}  }\\
& = \rbr{l + M} + \sum_{t \in [T]} \sum_{p \in \Pcal_t} \EE \sbr{ \indic \cbr{\Ecal_t, \wbar{H_i^p(t)},  m_i^p(t-1) \ge l} \cdot \wbar{\Phi}\rbr{  \sqrt{ \wbar{m_i^p}(t-1) / 4 } \rbr{y_i^p - \aggmu_i^p(t-1) } } } \\
& \le \rbr{l + M} + \sum_{t \in [T]} \sum_{p \in \Pcal_t} \EE \sbr{ \indic \cbr{\Ecal_t, \wbar{H_i^p(t)},  m_i^p(t-1) \ge l} \cdot \exp \rbr{ - \frac{ \wbar{m_i^p}(t-1) \rbr{y_i^p - \aggmu_i^p(t-1) }_+^2  }{8}   }} \\
& \le \rbr{l + M} + \sum_{t \in [T]} \sum_{p \in \Pcal_t} \EE \sbr{ \indic \cbr{\Ecal_t, \wbar{H_i^p(t)},  m_i^p(t-1) \ge l} \cdot \exp \rbr{ - \frac{ \wbar{m_i^p}(t-1) \rbr{\mu_i^p + \frac38 \Delta_i^p - \mu_i^p - \frac{9}{32} \Delta_i^p }_+^2  }{8}   }} \\
& \le \rbr{l + M} + \sum_{t \in [T]} \sum_{p \in \Pcal_t} \EE \sbr{ \indic \cbr{ \Ecal_t, \wbar{H_i^p(t)},  m_i^p(t-1) \ge l} \exp \rbr{ - \frac{\wbar{m_i^p}(t-1) \rbr{ \Delta_i^{\min}}^2}{(8)(256)}}  } \\
& \le (l+M) + \sum_{t \in [T]} \sum_{p \in \Pcal_t} \frac{1}{T^2} \\
& = \Ocal \rbr{ \frac{\ln T}{\rbr{\Delta_i^{\min}}^2} + M},
\end{align*}
where 
the first inequality is from Eq.~\eqref{eqn:A2_m_decomposition}, dropping the indicator $\indic\cbr{i_t^p = i}$ and 
using the law of total expectation; 
the first equality follows from the observation that $\Ecal_t$, $\wbar{H_i^p(t)}$, and $\cbr{m_i^p(t-1) \ge l}$ are $\Fcal_{t-1}$-measurable; 
the second equality follows from the observation that when $\wbar{H_i^p(t)}$ happens,
$ \EE \sbr{\indic \cbr{\theta_i^p(t) > y_i^p} \mid \Fcal_{t-1}} = \PP \rbr{ \theta_i^p(t) > y_i^p \mid \Fcal_{t-1}} = \wbar{\Phi}\rbr{ \frac{ y_i^p - \aggmu_i^p(t-1) }{ \sqrt{ 4 / \wbar{m_i^p}(t-1) } } }$;
the second inequality follows from Lemma~\ref{lem:phi-bounds};
the third inequality uses the facts that
\begin{enumerate}
    \item when $\cbr{m_i^p(t-1) \ge l}$ happens, $\wbar{m_i^p}(t-1) \ge m_i^p(t-1) - M \ge l-M = \frac{10240 \ln T}{ \rbr{\Delta_i^{\min}}^2}$,
    
    \item $y_i^p = \mu_i^p + \frac12 \delta_i^p > \mu_i^p + \frac38 \Delta_i^{\min}$ (see Fact~\ref{fact:small_delta}), and
    
    \item when $\Ecal_t$ happens, $\aggmu_i^p(t-1) \le \mu_i^p + \sqrt{\frac{10 \ln T}{\wbar{m_i^p}(t-1) }} + 2\epsilon < \mu_i^p + \frac{1}{32} \Delta_i^{\min} + \frac{1}{4}\Delta_i^{\min} = \mu_i^p + \frac{9}{32} \Delta_i^{\min}$ (see Definition~\ref{def:event_e} and Fact~\ref{fact:basic_facts_subpar_arms});
\end{enumerate}
the fourth inequality is by algebra;
and the fifth inequality again uses the fact that when $\cbr{m_i^p(t-1) \ge l}$ happens, $\wbar{m_i^p}(t-1) \ge m_i^p(t-1) - M \ge \frac{10240 \ln T}{ \rbr{\Delta_i^{\min}}^2}$.

In summary, we have
\[
    (A) \le (A1) + (A2) + \Ocal(1) \le \Ocal \rbr{ \frac{\ln T}{\rbr{\Delta_i^{\min}}^2} + M}.
    \qedhere
\]
\end{proof}

\subsubsection{Bounding Term $(B)$}
We now bound term $(B)$ in Eq.~\eqref{eqn: subpar_first_decomposition}. 

\begin{lemma}
\label{lem:B}
\[
(B) \le \Ocal \rbr{ \frac{\ln T}{(\Delta_i^{\min})^2} + M}.
\]
\end{lemma}

\begin{proof}
Lemma~\ref{lem:B} follows from Lemmas~\ref{lem:B*} and~\ref{lem:bounding_B*}, which we present shortly.
\end{proof}

Consider the following definition.
\begin{definition}
In any round $t \in [T]$, for any active player $p \in \Pcal_t$, define
\[
\phi_{i,t}^p = \Pr \rbr{\theta_{\dagger}^p(t) > y_i^p \mid \Fcal_{t-1}}.
\]
\end{definition}
\begin{remark}

Recall that $\wbar{\Phi}(\cdot)$ denotes the complementary CDF of the standard Gaussian distribution;
and recall $\wbar{n_i^p}(t-1) = n_i^p(t-1) \vee 1$, and $\wbar{m_i^p}(t-1) = (m_i^p(t-1)-M) \vee 1$.
$\phi_{i,t}^p$ can be explicitly written as: 
\begin{align}
\phi_{i,t}^p 
= & \bPhi \rbr{\frac{y_i^p - \hat{\mu}_{\dagger}^p(t-1)}{\sqrt{ \var_{\dagger}^{p}(t-1) }}} 
\label{eqn:phi_pre_decomposition}
\\ 
= & \bPhi \rbr{( y_i^p - \indmu_{\dagger}^p(t-1) ) \sqrt{\wbar{n_\dagger^p}(t-1)/4}} \cdot \indic \cbr{H_{\dagger}^p(t)} 
+ \bPhi \rbr{(y_i^p - \aggmu_{\dagger}^p(t-1))\sqrt{ \wbar{m_{\dagger}^p}(t-1)/4}} \cdot \indic \cbr{\wbar{H_{\dagger}^p(t)}}. \label{eqn:phi_decomposition_new}
\end{align}
\label{rem:phi_explicit}
\end{remark}

\begin{proof}[Proof of Remark~\ref{rem:phi_explicit}]
We have
\begin{align*}
\phi_{i,t}^p
= & \Pr \rbr{\theta_\dagger^p(t) > y_i^p \mid \hat{\mu}_{\dagger}^{p}(t-1), \var_{\dagger}^{p}(t-1)} \\
= & 1 - \Pr \rbr{\theta_\dagger^p(t) \le y_i^p \mid \hat{\mu}_{\dagger}^{p}(t-1), \var_{\dagger}^{p}(t-1)} \\
= & 1 - \Phi \rbr{\frac{y_i^p - \hat{\mu}_{\dagger}^p(t-1)}{\sqrt{ \var_{\dagger}^{p}(t-1) }}} 
= \wbar{\Phi} \rbr{\frac{y_i^p - \hat{\mu}_{\dagger}^p(t-1)}{\sqrt{ \var_{\dagger}^{p}(t-1) }}}.
\end{align*}

Eq.~\eqref{eqn:phi_decomposition_new} now follows by observing that:
\begin{enumerate}
    \item if $H_{\dagger}^p(t)$ happens, then $\hat{\mu}_{\dagger}^p(t-1) = \indmu_{\dagger}^p(t-1)$ and 
    $\var_{\dagger}^{p}(t-1) = \frac{4}{n_\dagger^p(t-1) \vee 1}$;
    \item if $\overline{H_{\dagger}^p(t)}$ happens, then $\hat{\mu}_{\dagger}^p(t-1) = \aggmu_{\dagger}^p(t-1)$ and 
    $\var_{\dagger}^{p}(t-1) = \frac{4}{(m_\dagger^p(t-1)-M) \vee 1}$.
    \qedhere
\end{enumerate}
\end{proof}

We now present the following lemma, which is inspired by a technique introduced in the work of \citep{an17}.

\begin{lemma}
\label{lem:B*}
\begin{align*}
(B) \le \underbrace{\sum_{t \in [T]} \sum_{p \in \Pcal_t} \EE \sbr{  {\rbr{\frac{1 - \phi_{i,t}^p}{\phi_{i,t}^p}} \indic \cbr{i_t^p = \dagger, \Ecal_t} }}}_{(B*)}.
\end{align*}
\end{lemma}

\begin{proof}
In any round $t$ and for any active player $p \in \Pcal_t$, consider
\begin{align}
    & \Pr \rbr{i_t^p = i, \wbar{Q_i^p(t)}, \Ecal_t \mid \Fcal_{t-1}} \nonumber \\
    = & \Pr \rbr{i_t^p = i, \theta_i^p(t) \leq y_i^p \mid \Fcal_{t-1}} \cdot \indic \cbr{\Ecal_t} \nonumber \\
    \leq & \Pr \rbr{i_t^p = \dagger \mid \Fcal_{t-1}} \cdot 
    \frac{\Pr \rbr{\theta_\dagger^p(t) \leq y_i^p \mid \Fcal_{t-1}}}{\Pr \rbr{\theta_\dagger^p(t) > y_i^p \mid \Fcal_{t-1}}}
    \cdot \indic \cbr{\Ecal_t} \nonumber \\
    = & 
    \rbr{ \frac{1-\phi_{i,t}^p}{\phi_{i,t}^p} } \cdot \Pr \rbr{i_t^p = \dagger \mid \Fcal_{t-1}}
    \cdot \indic \cbr{\Ecal_t} \nonumber \\
    = & 
    \rbr{\frac{1 - \phi_{i,t}^p}{\phi_{i,t}^p}} \Pr \rbr{i_t^p = \dagger, \Ecal_t \mid \Fcal_{t-1}}, \label{eqn:phi_intermediate}
\end{align}

where the first equality follows from the definition of $Q_i^p(t)$ and that $\Ecal_t$ is $\Fcal_{t-1}$-measurable; the first inequality uses Lemma~\ref{lem:change-arm} with $l = \dagger$ and $z = y_i^p$; the second equality inequality is from the definition of $\phi_{i,t}^p$; 
and the last equality is again because $\Ecal_t$ is $\Fcal_{t-1}$-measurable.

Finally, we have
\begin{align*}
\EE \sbr{ \indic \cbr{i_t^p = i, \wbar{Q_i^p(t)}, \Ecal_t} } 
& = \EE \sbr{ \Pr \rbr{i_t^p = i, \wbar{Q_i^p(t)}, \Ecal_t \mid \Fcal_{t-1}} } \\ 
& \le \EE \sbr{ \rbr{\frac{1 - \phi_{i,t}^p}{\phi_{i,t}^p}} \Pr \rbr{i_t^p = \dagger, \Ecal_t \mid \Fcal_{t-1}}} \\
& = \EE \sbr{ \EE \sbr{\rbr{\frac{1 - \phi_{i,t}^p}{\phi_{i,t}^p}} \indic \cbr{i_t^p = \dagger, \Ecal_t} \mid \Fcal_{t-1}}} \\
& = \EE \sbr{ {\rbr{\frac{1 - \phi_{i,t}^p}{\phi_{i,t}^p}} \indic \cbr{i_t^p = \dagger, \Ecal_t} }},
\end{align*}
where we use the law of total expectation and Eq.~\eqref{eqn:phi_intermediate}. The lemma follows by summing over all $t, p$'s.
\end{proof}

With foresight, let {$L = \frac{2560 \ln T}{\rbr{\Delta_i^{\min}}^2} + M$}. We further decompose term $(B*)$ as follows.

\begin{align}
(B*) & = \sum_{t \in [T]} \sum_{p \in \Pcal_t} \EE \sbr{  {\rbr{\frac{1 - \phi_{i,t}^p}{\phi_{i,t}^p}} \indic \cbr{i_t^p = \dagger, \Ecal_t} }} \nonumber \\
& =
\underbrace{\sum_{t \in [T]} \sum_{p \in \Pcal_t} \EE \sbr{ {\rbr{\frac{1 - \phi_{i,t}^p}{\phi_{i,t}^p}} \indic \cbr{i_t^p = \dagger, \Ecal_t, H_{\dagger}^p(t)} }}}_{(b1)} 
+
\underbrace{\sum_{t \in [T]} \sum_{p \in \Pcal_t} \EE \sbr{ {\rbr{\frac{1 - \phi_{i,t}^p}{\phi_{i,t}^p}} \indic \cbr{i_t^p = \dagger, \Ecal_t, \wbar{H_{\dagger}^p(t)}} }}}_{(b2)}, \nonumber \\
& = 
(b1)
+
\underbrace{\sum_{t \in [T]} \sum_{p \in \Pcal_t} \EE \sbr{ {\rbr{\frac{1 - \phi_{i,t}^p}{\phi_{i,t}^p}} \indic \cbr{i_t^p = \dagger, \Ecal_t, \wbar{H_{\dagger}^p(t)}, m_{\dagger}^p(t-1) < L} }}}_{(b2.1)} \nonumber \\
& \hspace{150pt} +
\underbrace{\sum_{t \in [T]} \sum_{p \in \Pcal_t} \EE \sbr{ {\rbr{\frac{1 - \phi_{i,t}^p}{\phi_{i,t}^p}} \indic \cbr{i_t^p = \dagger, \Ecal_t, \wbar{H_{\dagger}^p(t)}, m_{\dagger}^p(t-1) \ge L} }}}_{(b2.2)}.  \label{eqn:B*_decomposition_new}
\end{align}
where the inequality uses Lemma~\ref{lem:B*}.

\begin{lemma}
\label{lem:bounding_B*}
\[
(B*) \le \Ocal \rbr{\frac{\ln T}{\rbr{\Delta_i^{\min}}^2} + M}.
\]
\end{lemma}
\begin{proof}
Lemma~\ref{lem:bounding_B*} follows directly from Eq.~\eqref{eqn:B*_decomposition_new} and the following Lemma~\ref{lem:b1}, Lemma~\ref{lem:b21} and Lemma~\ref{lem:b22}, which provide  upper bounds on terms $(b1)$, $(b2.1)$ and $(b2.2)$, respectively.
\end{proof}

\begin{lemma}[Bounding term $(b1)$]
\label{lem:b1}
\begin{align*}
(b1) \le 
\Ocal \rbr{M}.
\end{align*}
\end{lemma}

\begin{proof}[Proof of Lemma~\ref{lem:b1}]
For any player $p \in [M]$ and $t \in [T]$, recall that $\wbar{n_{\dagger}^p}(t-1) = n_{\dagger}^p(t-1) \vee 1$ and $\rbr{z}_+ = z \vee 0$.
When $\Ecal_t$ and $H_{\dagger}^p(t)$ happen, 
$n_\dagger^p(t-1) \ge \frac{40\ln T}{\epsilon^2} =: Y$; we have:
\begin{align*}
    & 1 - \phi_{i,t}^p \\
    = & \Pr \rbr{ \theta_{\dagger}^p(t) \le y_i^p \mid \Fcal_{t-1}} \\
    = & \Phi\rbr{ (y_i^p - \indmu_\dagger^p(t-1)) \sqrt{ \wbar{n_\dagger^p}(t-1) / 4 } } \\
    \leq & \exp \del{ - 
    \frac{ \wbar{n_\dagger^p}(t-1)(\indmu_\dagger^p(t-1) - y_i^p)_+^2  }{ 8 }  } \\
     \leq & \exp \del{ - 
    \frac{ n_\dagger^p(t-1)(\mu_\dagger^p - \frac14 \Delta_i^p - \mu_\dagger^p + \frac38 \Delta_i^p)_+^2  }{ 8 }  } \\
     \leq & \exp \del{ - 
    \frac{ n_\dagger^p(t-1) (\Delta_i^p)^2  }{ 8 (64) }  } \\
    \leq & \frac{1}{T+1},
\end{align*}
where 
the second equality uses Remark~\ref{rem:phi_explicit};
the first inequality uses Lemma~\ref{lem:phi-bounds};
the second inequality follows from the observations that, when $\Ecal_t$ and $H_{\dagger}^p(t)$ happen:
\begin{enumerate}
    \item $\wbar{n_\dagger^p}(t-1) \ge n_\dagger^p(t-1) \ge Y = \frac{40\ln T}{\epsilon^2} \ge \frac{2560\ln T}{(\Delta_i^p)^2}$ (see Fact~\ref{fact:small_delta}),
    
    \item $\indmu_{\dagger}^p(t-1) \ge \mu_{\dagger}^p - \sqrt{ \frac{10 \ln T}{\wbar{n_\dagger^p}(t-1)}} \ge \mu_{\dagger}^p - \frac{1}{4}\Delta_i^p$ (see Definition~\ref{def:event_e}), and
    
    \item $y_i^p = \mu_{\dagger}^p - \frac12 \delta_i^p < \mu_{\dagger}^p - \frac38 \Delta_i^p$;
\end{enumerate}
the third inequality is by algebra;
and the last inequality follows because, again, when $H_{\dagger}^p(t)$ happens, $n_\dagger^p(t-1) \ge Y =  \frac{40\ln T}{\epsilon^2} \ge \frac{2560\ln T}{(\Delta_i^p)^2} \ge \frac{1280 \ln (T+1)}{(\Delta_i^p)^2}$ for $T > 1$.

It follows that, when $\Ecal_t$ and $H_{\dagger}^p(t)$ happen, $\phi_{i,t}^p \ge \frac{T}{T+1}$ and $\frac{1 - \phi_{i,t}^p}{\phi_{i,t}^p} \le \frac{1}{T}$. Hence, 
\[
(b1) \le \sum_{p \in [M]} \sum_{t: p \in \Pcal_t}  \EE \sbr{ {\rbr{\frac{1 - \phi_{i,t}^p}{\phi_{i,t}^p}} \indic \cbr{\Ecal_t, H_{\dagger}^p(t)} }} \le M.
\]

\end{proof}

\begin{lemma}[Bounding term $(b2.1)$]
\label{lem:b21}
\begin{align*}
(b2.1) 
& \le \Ocal \rbr{\frac{\ln T}{\rbr{\Delta_i^{\min}}^2} + M}.
\end{align*}

\end{lemma}

The remark below is useful for proving Lemma~\ref{lem:b21}.
\begin{remark}[Invariant property]
\label{rem:invariant}
Recall from Example~\ref{ex:h-invariant} that 
$\cbr{H_\dagger^p(t): t \in [T], p \in [M]}$ satisfies the invariant property with respect to $\dagger$. 

Moreover, the construction of Algorithm~\ref{alg:robustaggTS} enforces that 
$\cbr{\phi_{i,t}^p: t \in [T], p \in [M]}$ satisfies the invariant property with respect to $\dagger$ (note that it does not necessarily satisfy the invariant property with respect to $i$). Indeed, this follows from Eq.~\eqref{eqn:phi_pre_decomposition}, along with  Example~\ref{ex:mu-var-invariant} that shows that the posterior parameters, $\cbr{ (\hat{\mu}_\dagger^p(t-1), \var_\dagger^p(t-1)): t \in [T], p \in [M]}$, satisfy the invariant property with respect to $\dagger$. 

Combining the two observations above, $\cbr{\rbr{ \frac{1}{\phi_{i,t}^p}-1} \indic \cbr{ \wbar{H_{\dagger}^p(t)}}: t \in [T], p \in [M]}$ also satisfies the invariant property with respect to arm $\dagger$. 
\end{remark}

\begin{proof}[Proof of Lemma~\ref{lem:b21}]
Proving Lemma~\ref{lem:b21} requires more special care.
Recall that
\begin{align*}
(b2.1) = & \sum_{t \in [T]} \sum_{p \in \Pcal_t} \EE \sbr{ {\rbr{\frac{1 - \phi_{i,t}^p}{\phi_{i,t}^p}} \indic \cbr{i_t^p = \dagger, \Ecal_t, \wbar{H_{\dagger}^p(t)}, m_{\dagger}^p(t-1) < L} }} \\
\le & \sum_{t \in [T]} \sum_{p \in \Pcal_t} \EE \sbr{ {\rbr{\frac{1 - \phi_{i,t}^p}{\phi_{i,t}^p}} \indic \cbr{i_t^p = \dagger, \wbar{H_{\dagger}^p(t)}, m_{\dagger}^p(t-1) < L} }}.
\end{align*}
Also recall the definition of stopping time $\tau_k(\dagger)$ (Definition~\ref{def:tau_k}), the round in which $\dagger$ is pulled the $k$-th time by any player. To ease exposition, we abuse the notation and denote $\tau_k(\dagger)$ by $\tau_k$. Similarly, let $p_k := p_k(\dagger)$ denote the player that issues the $k$-th pull of arm $\dagger$ (recall Definition~\ref{def:p_k}). %

Since $\cbr{\rbr{ \frac{1}{\phi_{i,t}^p}-1} \indic \cbr{ \wbar{H_{\dagger}^p(t)}}: t \in [T], p \in [M]}$ satisfies the invariant property with respect to arm $\dagger$, by Lemma~\ref{lem:auxiliary_sum_f_indic}, we have
\begin{align}
(b2.1) \le \sum_{p=1}^M \EE \sbr{ \rbr{\frac{1}{\phi_{i,1}^p} - 1} \indic \cbr{ \wbar{H_{\dagger}^p(1)}}} + \sum_{k=1}^{L-1} \EE \sbr{ \rbr{\frac{1}{\phi_{i, \tau_k+1}^{p_k}} - 1} \indic \cbr{\tau_k \le T, \wbar{H_{\dagger}^p(\tau_k + 1)}}}, \label{eqn:b21_phi}
\end{align}
where we also use the linearity of expectations.

Since the variance of the aggregate posteriors are initialized as the constant $c_2 = 4$ in $\robustaggTS(\epsilon)$, we have that $\rbr{\frac{1}{\phi_{i,1}^p} - 1} \indic \cbr{ \wbar{H_{\dagger}^p(1)}}  \le \Ocal \rbr{1}$ with probability 1.
Therefore, 
\begin{align}
    \sum_{p=1}^M \EE \sbr{ \rbr{\frac{1}{\phi_{i,1}^p} - 1} \indic \cbr{ \wbar{H_{\dagger}^p(1)}}} \le \Ocal \rbr{M}. \label{eqn:constant_variance}
\end{align}
It then suffices to bound the second term in Eq.~\eqref{eqn:b21_phi}---it follows straightforwardly from Lemma~\ref{lem:key_lemma_constant_appendix}, which we present shortly, that the second term is bounded by $\Ocal \rbr{L}$.
It then follows from Eq.~\eqref{eqn:b21_phi}, Eq.~\eqref{eqn:constant_variance}, and Lemma~\ref{lem:key_lemma_constant_appendix} that $(b2.1) \le \Ocal \rbr{\frac{\ln T}{\rbr{\Delta_i^{\min}}^2} + M}$.
\end{proof}

\begin{lemma}
\label{lem:key_lemma_constant_appendix}
For any $k \in [TM]$,
\[
\EE \sbr{\rbr{\frac{1}{\phi_{i, \tau_k+1}^{p_k}} - 1} \indic \cbr{\tau_k \le T, \wbar{H_{\dagger}^p(\tau_k+1)}}} \le \Ocal \rbr{1},
\]
where we recall that $\tau_k = \tau_k(\dagger)$ and $p_k = p_k(\dagger)$ is the player that issues the $k$-th pull of arm $\dagger$.
\end{lemma}

\begin{proof}
Using Remark~\ref{rem:phi_explicit}, we observe that
\begin{align}
    \phi_{i,\tau_k+1}^{p_k} = & \sbr{\wbar{ \Phi} \rbr{\frac{y_i^p - \indmu_{\dagger}^p(\tau_k)}{2}\sqrt{\rbr{n_\dagger^p(\tau_k)} \vee 1}}} \cdot \indic \cbr{H_{\dagger}^p(\tau_k+1)} \nonumber \\
    & + \sbr{\wbar { \Phi} \rbr{\frac{y_i^p - \aggmu_{\dagger}^p(\tau_k)}{2}\sqrt{ \rbr{m_{\dagger}^p(\tau_k)-M} \vee 1}}} \cdot \indic \cbr{\wbar{H_{\dagger}^p(\tau_k+1)}}. \label{eqn:phi_decomposition_tau}
\end{align}

We have
\begin{align}
& \EE \sbr{ \rbr{\frac{1}{\phi_{i,\tau_k+1}^{p_k}} - 1} \indic \cbr{\tau_k \le T, \wbar{H_{\dagger}^{p_k}(\tau_k+1)}} } \nonumber \\
= & \EE \sbr{ \rbr{\frac{1}{\bPhi \rbr{\rbr{y_i^{p_k} - \aggmu_{\dagger}^{p_k}(\tau_k)}\sqrt{ \rbr{ \rbr{m_{\dagger}^{p_k}(\tau_k) - M} \vee 1 }/4}}} - 1} \indic \cbr{\tau_k \le T, \wbar{H_{\dagger}^{p_k}(\tau_k+1)}} } \nonumber \\
\leq & 
\EE \sbr{ \frac{1}{\bPhi \rbr{\rbr{\mu_\dagger^{p_k} - \aggmu_{\dagger}(\tau_k)}\sqrt{ \rbr{\rbr{n_{\dagger}(\tau_k) - M} \vee 1} / 4}}} \indic \cbr{\tau_k \le T} },
\label{eqn:inv_comp_cdf_tau_k}
\end{align}
where the last inequality uses the observations that $y_i^{p_k} \leq \mu_\dagger^{p_k}$, $\aggmu_{\dagger}^{p_k}(\tau_k) = \aggmu_{\dagger}(\tau_k)$ and $m_{\dagger}^{p_k}(\tau_k) = n_{\dagger}(\tau_k)$, as well as the monotonic increasing property of $z \mapsto \frac{1}{\bPhi(z)}$.

Observe that, from Corollary~\ref{col:agg_concentration_alternative}, for any $z \ge 1$,
\begin{align*}
& \Pr \rbr{ \rbr{\tau_k \le T} \wedge \rbr{ \mu_{\dagger}^{p_k} - \aggmu_\dagger(\tau_k) \geq z \sqrt{\frac{4}{(n_{\dagger}(\tau_k) - M) \vee 1}}}} \\
\le & 
\Pr \rbr{ \rbr{\tau_k \le T} \wedge \rbr{ \exists p \in [M],\ \mu_{\dagger}^{p} - \aggmu_{\dagger}(\tau_k) \geq z \sqrt{\frac{4}{(n_{\dagger}(\tau_k) - M) \vee 1}}}} \\
\le & 2 e^{-2z^2},
\end{align*}

Applying Lemma~\ref{lem:sg-anticonc} with $X = 
\rbr{\aggmu_{\dagger}(\tau_k) - \mu_{\dagger}^{p_k}}\sqrt{ \rbr{\rbr{n_{\dagger}(\tau_k) - M} \vee 1} / 4}$ and $E = \cbr{ \tau_k \leq T }$, we conclude the proof. 
\end{proof}

\begin{remark}
\label{rem:comparison_freedman_azuma}
Note that it follows from our novel concentration inequality (Corollary~\ref{col:agg_concentration_alternative}) that 
\[
\Pr \rbr{ \tau_k \le T, \mu_{\dagger}^{p} - \aggmu_{\dagger}(\tau_k) > { \sqrt{\frac{2\ln \rbr{\frac{2}{\delta}}}{ \rbr{n_{\dagger}(\tau_k) -M} \vee 1}}}} < \delta;
\]
this tight bound enables us to bound Eq.~\eqref{eqn:inv_comp_cdf_tau_k} by $\order\rbr{1}$, which is essential to our proof of Lemma~\ref{lem:key_lemma_constant_appendix}.

Since $n_i(\tau_k) \le [k, k+M-1]$, using the Azuma-Hoeffding inequality and the union bound, one can obtain
\[
\Pr \rbr{ \tau_k \le T, \mu_{\dagger}^{p} - \aggmu_{\dagger}(\tau_k) > \Ocal \rbr{\sqrt{ \frac{\ln \rbr{\frac{M}{\delta}}}{ \rbr{n_{\dagger}(\tau_k) -M} \vee 1}}}} < \delta;
\]
and using Freedman's inequality \citep[e.g.,][Lemma 17]{wzsrc21}, one can obtain
\[
\Pr \rbr{ \tau_k \le T, \mu_{\dagger}^{p} - \aggmu_{\dagger}(\tau_k) > \Ocal \rbr{\sqrt{ \frac{\ln \rbr{\frac{\ln T}{\delta}}}{ \rbr{n_{\dagger}(\tau_k) -M} \vee 1}}}} < \delta.
\]
However, naively combining the above bounds with Lemma~\ref{lem:sg-anticonc}, one needs to set $C_1$ in Lemma~\ref{lem:sg-anticonc} to be $\Ocal \rbr{ M}$ or $\Ocal \rbr{ \ln T}$, which
incurs extra (undesirable) $\Ocal\rbr{ M}$ or $\Ocal\rbr{ \ln T}$ factors for bounding Eq.~\eqref{eqn:inv_comp_cdf_tau_k}.
\end{remark}

\begin{lemma}[Bounding term $(b2.2)$]
\label{lem:b22}
\begin{align*}
(b2.2) 
& \le \Ocal \rbr{ M }.
\end{align*}

\end{lemma}

\begin{proof}[Proof of Lemma~\ref{lem:b22}]

For any player $p \in [M]$ and $t \in [T]$, recall that $\wbar{m_{\dagger}^p}(t-1) = (m_{\dagger}^p(t-1) -M) \vee 1$ and $\rbr{z}_+ = z \vee 0$.
When $\Ecal_t, \cbr{m_{\dagger}^p(t-1) \ge L}$ and $\wbar{H_{\dagger}^p(t)}$ happen, 
\begin{align*}
    & 1 - \phi_{i,t}^p \\
    = & \Pr \rbr{ \theta_{\dagger}^p(t) \le y_i^p \mid \Fcal_{t-1}} \\
    = & \Phi\rbr{ (y_i^p - \aggmu_\dagger^p(t-1)) \sqrt{ \wbar{m_\dagger^p}(t-1) / 4 }} \\
    \le & \exp \rbr{ - \frac{\wbar{m_\dagger^p}(t-1) (\aggmu_\dagger^p(t-1) - y_i^p)_+^2}{8} } \\
    \le & \exp \rbr{ - \frac{\wbar{m_\dagger^p}(t-1) (\mu_\dagger^p - \frac14 \Delta_i^p - \mu_{\dagger}^p + \frac38 \Delta_i^p)_+^2}{8} } \\
    \le & \exp \rbr{ - \frac{\wbar{m_\dagger^p}(t-1) (\Delta_i^p)^2}{8(64)}} \\
    \le & \frac{1}{T+1},
\end{align*}
where 
the second equality uses Remark~\ref{rem:phi_explicit};
the first inequality uses Lemma~\ref{lem:phi-bounds};
the second inequality follows from the observations that, when $\Ecal_t, \cbr{m_{\dagger}^p(t-1) \ge L}$ and $\wbar{H_{\dagger}^p(t)}$ happen:
\begin{enumerate}

    \item $\wbar{m_\dagger^p}(t-1) \ge m_\dagger^p(t-1) - M \ge L - M \ge \frac{2560\ln T}{(\Delta_i^{p})^2}$,
    
    \item $\aggmu_{\dagger}^p(t-1) \ge \mu_{\dagger}^p - \sqrt{ \frac{10 \ln T}{\wbar{m_\dagger^p}(t-1)}} \ge \mu_{\dagger}^p - \frac{1}{4}\Delta_i^p$ (see Definition~\ref{def:event_e}), and
    
    \item $y_i^p = \mu_{\dagger}^p - \frac12 \delta_i^p < \mu_{\dagger}^p - \frac38 \Delta_i^p$;
\end{enumerate}
the third inequality is by algebra;
and the last inequality follows from the observation that $\wbar{m_\dagger^p}(t-1) \ge m_\dagger^p(t-1) - M \ge L - M \ge \frac{2560\ln T}{(\Delta_i^{p})^2} \ge \frac{1280 \ln (T+1)}{(\Delta_i^p)^2}$ for $T > 1$.

It follows that, when $\Ecal_t, \cbr{m_{\dagger}^p(t-1) \ge L}$ and $\wbar{H_{\dagger}^p(t)}$ happen, $\phi_{i,t}^p \ge \frac{T}{T+1}$ and $\frac{1 - \phi_{i,t}^p}{\phi_{i,t}^p} \le \frac{1}{T}$. Hence, 
\[
(b2.2) \le \sum_{p \in [M]} \sum_{t: p \in \Pcal_t}  \EE \sbr{ {\rbr{\frac{1 - \phi_{i,t}^p}{\phi_{i,t}^p}} \indic \cbr{\Ecal_t, \wbar{H_{\dagger}^p(t)}, m_{\dagger}^p(t-1) \ge L} }} \le M.
\qedhere
\]

\end{proof}

\subsection{Non-subpar Arms}
\label{sec:appendix_nonsubpar_arms}

In this section, we provide a proof for Lemma~\ref{lem:i-eps-c-robustagg-ts}.

Let us fix any player $p \in [M]$ and any suboptimal arm $i \in \Ical_{10\epsilon}^C$ for player $p$ such that $\Delta_i^p > 0$. In the rest of this section, let us also fix an optimal arm for player $p$, $\optarm_p$, and we abbreviate it by $\optarm$. We have $\mu_\optarm^p = \mu_*^p = \max_{j \in [K]} \mu_j^p$.

\begin{definition}
\label{def:z_and_W}
Let $z_i^p = \mu_i^p + \frac12 \Delta_i^p$ be a threshold.
In any round $t$, define
\[
W_i^p(t) = \cbr{\theta_i^p(t) > z_i^p}
\]
to be the event that the sample $\theta_i^p(t)$ from the posterior distribution associated with arm $i$ and player $p$ in round $t$ is greater than the threshold $z_i^p$.
Therefore, 
$\wbar{W_i^p(t)} = \cbr{\theta_i^p(t) \le z_i^p}$.
\end{definition}

\subsubsection{Non-subpar Arms---Decomposition}
We consider the following decomposition.
\begin{align}
& \EE \sbr{n_i^p(T)} \nonumber \\
= & \EE \sbr{\sum_{t: p \in \Pcal_t} \indic \cbr{i_t^p = i}} \nonumber \\
= & \EE \sbr{\sum_{t: p \in \Pcal_t} \indic \cbr{i_t^p = i, W_i^p(t), \Ecal_t}} +
\EE \sbr{\sum_{t: p \in \Pcal_t} \indic \cbr{i_t^p = i, \wbar{W_i^p(t)}, \Ecal_t}} + \sum_{t: p \in \Pcal_t}  \EE \sbr{ \indic \cbr{i_t^p = i, \wbar{\Ecal_t}}} \nonumber\\
\le & \underbrace{\EE \sbr{\sum_{t: p \in \Pcal_t} \indic \cbr{i_t^p = i, W_i^p(t), \Ecal_t}}}_{(D)} +
\underbrace{\EE \sbr{\sum_{t: p \in \Pcal_t} \indic \cbr{i_t^p = i, \wbar{W_i^p(t)}, \Ecal_t}}}_{(E)} + \Ocal(1),\label{eqn:nonsubpar_first_decomposition}
\end{align}
where the last inequality follows from the observation that $\EE \sbr{ \indic \cbr{i_t^p = i, \wbar{\Ecal_t}}} \le \EE \sbr{ \indic \cbr{\wbar{\Ecal_t}}}$ and Lemma~\ref{lem:clean_event}.

Following this decomposition, Lemma~\ref{lem:i-eps-c-robustagg-ts} is proved straightforwardly given Lemma~\ref{lem:D} and Lemma~\ref{lem:e} which we present in what follows.

\subsubsection{Bounding Term $(D)$}
We first bound term $(D)$ in Eq.~\eqref{eqn:nonsubpar_first_decomposition}.

\begin{lemma}
\label{lem:D}
\[
(D) \le \Ocal \rbr{ \frac{\ln T}{(\Delta_i^p)^2} + M}.
\]
\end{lemma}

\begin{proof}[Proof of Lemma~\ref{lem:D}]
With foresight, let $h = \frac{4000 \ln T}{(\Delta_i^p)^2} + 2M$. 
Recall that $H_i^p(t)$ is the event that the individual posterior is used in round $t$ by active player $p$ for arm $i$ (see Definition~\ref{def:alg_criteron}). 
We have
\begin{align*}
(D) & = \EE \sbr{ \sum_{t: p \in \Pcal_t}  \indic \cbr{i_t^p = i, W_i^p(t), \Ecal_t}} \\
& \le h + \sum_{t: p \in \Pcal_t} \EE \sbr{ \indic \cbr{i_t^p = i, W_i^p(t), \Ecal_t, n_i^p(t-1) \ge h}} \\
& = h + \underbrace{\sum_{t: p \in \Pcal_t} \EE \sbr{ \indic \cbr{i_t^p = i, W_i^p(t), \Ecal_t, H_i^p(t), n_i^p(t-1) \ge h}}}_{(d)} ,
\end{align*}
where 
the last equality follows from the observation that
$\cbr{n_i^p(t-1) \ge h}$ implies that
$H_i^p(t)$ happening. 
To see why this is true, 
recall that $H_i^p(t) = \cbr{ n_i^p(t-1) \geq \frac{40 \ln T}{\epsilon^2} + 2M }$;
and
observe that for non-subpar arm $i \in \Ical_{10\epsilon}^C$ and player $p$,
$\cbr{n_i^p(t-1) \ge h = \frac{4000\ln T}{(\Delta_i^p)^2} + 2M}$ %
implies $\cbr{n_i^p(t-1) \ge \frac{40\ln T}{\epsilon^2} + 2M}$ because $\Delta_i^p \le 10\epsilon$.

It therefore suffices to bound term $(d)$. 
We have
\begin{align*}
(d) 
\le & \sum_{t: p \in \Pcal_t} \EE \sbr{ \indic \cbr{ W_i^p(t), \Ecal_t, H_i^p(t), n_i^p(t-1) \ge h}} \\
= & \sum_{t: p \in \Pcal_t} \EE \sbr{ \indic \cbr{ \Ecal_t, H_i^p(t), n_i^p(t-1) \ge h} \EE\sbr{ \indic\cbr{ W_i^p(t) } \mid \Fcal_{t-1} } } \\
= & 
\sum_{t: p \in \Pcal_t} \EE \sbr{ \indic \cbr{ \Ecal_t, H_i^p(t), n_i^p(t-1) \ge h} \bPhi\rbr{  (z_i^p - \indmu_i^p(t-1)) \sqrt{ \wbar{n_i^p}(t-1) / 4 } } } \\
\leq & 
\sum_{t: p \in \Pcal_t} \EE \sbr{ \indic \cbr{ \Ecal_t, H_i^p(t), n_i^p(t-1) \ge h} \exp \rbr{ - \frac{  \wbar{n_i^p}(t-1) (z_i^p - \indmu_i^p(t-1) )_+^2 }{ 8} } } \\
\leq & 
\sum_{t: p \in \Pcal_t} \EE \sbr{ \indic \cbr{ \Ecal_t, H_i^p(t), n_i^p(t-1) \ge h} \exp \rbr{ - \frac{  n_i^p(t-1) (\mu_i^p + \frac12 \Delta_i^p - \mu_i^p - \frac1{16} \Delta_i^p )_+^2 }{ 8 } } }
\\
\leq & 
\sum_{t: p \in \Pcal_t} \EE \sbr{ \indic \cbr{ \Ecal_t, H_i^p(t), n_i^p(t-1) \ge h} \exp \rbr{ - \frac{  n_i^p(t-1) ( \Delta_i^p )^2 }{ 8 (16) } } }
\\
\leq & \Ocal \rbr{1}.
\end{align*}
where the first inequality drops the indicator $\indic\cbr{i_t^p = i}$; 
the first equality uses the law of total expectation and the observation that $\Ecal_t$, $H_i^p(t)$ and $\cbr{n_i^p(t-1) \ge h}$ are $\Fcal_{t-1}$-measurable;
the second inequality follows from Lemma~\ref{lem:phi-bounds};
the third inequality is from the observations that when $\Ecal_t$ and $H_i^p(t)$ happen: 
\begin{enumerate}
    \item $\wbar{n_i^p}(t-1) \ge n_i^p(t-1) \ge h = \frac{4000 \ln T}{(\Delta_i^p)^2} + 2M$,
    \item $\indmu_i^p(t-1) \leq \mu_i^p + \sqrt{ \frac{10 \ln T}{ \wbar{n_i^p}(t-1) } } \leq \mu_i^p + \frac{1}{16}\Delta_i^p$ (see Definition~\ref{def:event_e}), and
    \item $z_i^p = \mu_i^p + \frac12 \Delta_i^p$;
\end{enumerate}
the fourth inequality is by algebra; 
and the last inequality is from the observation that when $n_i^p(t-1) \geq h$, $\exp \rbr{ - \frac{  n_i^p(t-1) ( \Delta_i^p )^2 }{ 8 (16) } } \le \frac 1 T$.

In summary, $(D) \le h + (d) \le \Ocal \rbr{ \frac{\ln T}{(\Delta_i^p)^2} + M}$.
\end{proof}

\subsubsection{Bounding Term $(E)$}
We now bound $(E)$ in Eq.~\eqref{eqn:nonsubpar_first_decomposition}:

\begin{lemma}
\label{lem:e}
\[
(E) \le \Ocal \rbr{ \frac{\ln T}{(\Delta_i^p)^2} + M}.
\]
\end{lemma}
\begin{proof}
Lemma~\ref{lem:e} follows from Lemma~\ref{lem:psi_decomposition_of_e}, Eq.~\eqref{eqn:e*_decomposition}, Lemma~\ref{lem:e1}, and Lemma~\ref{lem:e2} which we present shortly.
\end{proof}

We begin with the following definition, similar to the notion of $\phi_{i,t}^p$ used for subpar arms.
\begin{definition}
\label{def:psi}
Recall that $p$ is a fixed player, $i$ is a fixed suboptimal arm for $p$, and $\optarm$ is a fixed optimal arm for $p$.  In any round $t$, define
\[
\psi_{i,t}^p = \Pr \rbr{\theta_{\optarm}^p(t) > z_i^p \mid \Fcal_{t-1}}.
\]

\end{definition}
\begin{remark}
Recall that $\wbar{n_\optarm^p}(t-1) = n_\optarm^p(t-1) \vee 1$ and $\wbar{m_\optarm^p}(t-1) = \rbr{m_\optarm^p(t-1) - M} \vee 1$.
$\psi_{i,t}^p$ can be explicitly written as: 
\begin{align}
\psi_{i,t}^p 
= & \bPhi \rbr{\frac{z_i^p - \hat{\mu}_{\optarm}^p(t-1)}{\sqrt{ \var_{\optarm}^{p}(t-1) }}}  
\label{eqn:psi_pre_decomposition}
\\ 
= & \bPhi \rbr{( z_i^p - \indmu_{\optarm}^p(t-1) ) \sqrt{\wbar{n_{\optarm}^p}(t-1)/4}} \cdot \indic \cbr{H_{\optarm}^p(t)}
+ \bPhi \rbr{(z_i^p - \aggmu_{\optarm}^p(t-1))\sqrt{ \wbar{m_{\optarm}^p}(t-1)/4}} \cdot \indic \cbr{\wbar{H_{\optarm}^p}(t)}.
\end{align}
\label{rem:psi_explicit}
\end{remark}
The proof for the above remark is omitted, as it is very similar to that of Remark~\ref{rem:phi_explicit}.

We now present the following lemma.
\begin{lemma}
\label{lem:psi_decomposition_of_e}
\begin{align*}
(E) = \EE \sbr{\sum_{t: p \in \Pcal_t} \indic \cbr{i_t^p = i, \wbar{W_i^p(t)}, \Ecal_t}} 
\le 
\underbrace{\sum_{t: p \in \Pcal_t} \EE \sbr{  {\rbr{\frac{1 - \psi_{i,t}^p}{\psi_{i,t}^p}} \indic \cbr{i_t^p = \optarm, \Ecal_t} }}}_{(E*)}
\end{align*}
\end{lemma}
\begin{proof}
The proof largely follows the same outline as that of Lemma~\ref{lem:B*}.

In any round $t$ and such that $p \in \Pcal_t$, consider
\begin{align}
    & \Pr \rbr{i_t^p = i, \wbar{Q_i^p(t)}, \Ecal_t \mid \Fcal_{t-1}} \nonumber \\
    = & \Pr \rbr{i_t^p = i, \theta_i^p(t) \leq z_i^p \mid \Fcal_{t-1}} \cdot \indic \cbr{\Ecal_t} \nonumber \\
    \leq & \Pr \rbr{i_t^p = \optarm \mid \Fcal_{t-1}} \cdot 
    \frac{\Pr \rbr{\theta_\optarm^p(t) \leq z_i^p \mid \Fcal_{t-1}}}{\Pr \rbr{\theta_\optarm^p(t) > z_i^p \mid \Fcal_{t-1}}}
    \cdot \indic \cbr{\Ecal_t} \nonumber \\
    = & 
    \rbr{ \frac{1-\psi_{i,t}^p}{\psi_{i,t}^p} } \cdot \Pr \rbr{i_t^p = \optarm \mid \Fcal_{t-1}}
    \cdot \indic \cbr{\Ecal_t} \nonumber \\
    = & 
    \rbr{\frac{1 - \psi_{i,t}^p}{\psi_{i,t}^p}} \Pr \rbr{i_t^p = \optarm, \Ecal_t \mid \Fcal_{t-1}}, \label{eqn:psi_intermediate}
\end{align}
where the first equality follows from the definition of $Q_i^p(t)$ and that $\Ecal_t$ is $\Fcal_{t-1}$-measurable; the first inequality uses Lemma~\ref{lem:change-arm} with $l = \optarm$ and $z = z_i^p$; and the second equality inequality is from the definition of $\psi_{i,t}^p$; 
the last equality is again because $\Ecal_t$ is $\Fcal_{t-1}$-measurable.

Finally, we have
\begin{align*}
\EE \sbr{ \indic \cbr{i_t^p = i, \wbar{Q_i^p(t)}, \Ecal_t} } 
& = \EE \sbr{ \Pr \rbr{i_t^p = i, \wbar{Q_i^p(t)}, \Ecal_t \mid \Fcal_{t-1}} } \\ 
& \le \EE \sbr{ \rbr{\frac{1 - \psi_{i,t}^p}{\psi_{i,t}^p}} \Pr \rbr{i_t^p = \optarm, \Ecal_t \mid \Fcal_{t-1}}} \\
& = \EE \sbr{ \EE \sbr{\rbr{\frac{1 - \psi_{i,t}^p}{\psi_{i,t}^p}} \indic \cbr{i_t^p = \optarm, \Ecal_t} \mid \Fcal_{t-1}}} \\
& = \EE \sbr{ {\rbr{\frac{1 - \psi_{i,t}^p}{\psi_{i,t}^p}} \indic \cbr{i_t^p = \optarm, \Ecal_t} }},
\end{align*}
where we use the law of total expectation and Eq.~\eqref{eqn:psi_intermediate}. The lemma follows by summing over all $t$'s.
\end{proof}

Let us further decompose $(E*)$ as follows.
\begin{align}
(E*)
= \underbrace{\sum_{t: p \in \Pcal_t} \EE \sbr{  {\rbr{\frac{1 - \psi_{i,t}^p}{\psi_{i,t}^p}} \indic \cbr{i_t^p = \optarm, \Ecal_t, H_\optarm^p(t)} }}}_{(e1)}
+
\underbrace{\sum_{t: p \in \Pcal_t} \EE \sbr{  {\rbr{\frac{1 - \psi_{i,t}^p}{\psi_{i,t}^p}} \indic \cbr{i_t^p = \optarm, \Ecal_t, \wbar{H_\optarm^p(t)}} }}}_{(e2)}. \label{eqn:e*_decomposition}
\end{align}

We first consider term $(e1)$. 
\begin{lemma}
\label{lem:e1}
\[
(e1) \le \Ocal \rbr{ \frac{\ln T}{(\Delta_i^p)^2}}.
\]
\end{lemma}

\begin{proof}[Proof of Lemma~\ref{lem:e1}]
With foresight, let $J = \frac{640 \ln T}{(\Delta_i^p)^2}$. We have
\begin{align*}
(e1) & = \underbrace{\sum_{t: p \in \Pcal_t} \EE \sbr{  {\rbr{\frac{1 - \psi_{i,t}^p}{\psi_{i,t}^p}} \indic \cbr{i_t^p = \optarm, \Ecal_t, H_\optarm^p(t), n_\optarm^p(t-1) < J} }}}_{(e1.1)}
+ \\
& \hspace{100pt} \underbrace{\sum_{t: p \in \Pcal_t} \EE \sbr{  {\rbr{\frac{1 - \psi_{i,t}^p}{\psi_{i,t}^p}} \indic \cbr{i_t^p = \optarm, \Ecal_t, H_\optarm^p(t), n_\optarm^p(t-1) \ge J} }}}_{(e1.2)}.
\end{align*}
Lemma~\ref{lem:e1} follows straightforwardly from Lemma~\ref{lem:e1_1} and Lemma~\ref{lem:e1_2}, which bound $(e1.1)$ and $(e1.2)$, respectively.
\end{proof}

\begin{lemma}
\label{lem:e1_1}
\[
(e1.1) \le \Ocal \rbr{ \frac{\ln T}{(\Delta_i^p)^2}}.
\]
\end{lemma}

To prove Lemma~\ref{lem:e1_1}, we first present the following Remark~\ref{rem:invariant}.
\begin{remark}[Invariant Property]
\label{rem:invariant-more}
Similar to Remark~\ref{rem:invariant}, by the construction of Algorithm~\ref{alg:robustaggTS}, we have that for any arm $i \in [K]$, and player $p \in [M]$, $\cbr{ \psi_{i,t}^p: t \in [T] }$ 
and 
$\cbr{ H_{\optarm}^p(t): t \in [T] }$
satisfy the invariant property with respect to $(\optarm,p)$ (Definition~\ref{def:invariant}). 
Indeed, the former follows from Eq.~\eqref{eqn:psi_pre_decomposition}, along with  Example~\ref{ex:mu-var-invariant} that shows that the posterior parameters, $\cbr{ (\hat{\mu}_\optarm^p(t-1), \var_\optarm^p(t-1)): t \in [T]}$, satisfy the invariant property with respect to $(\optarm, p)$; and the latter is from Example~\ref{ex:h-invariant}. 
\end{remark}

\begin{proof}[Proof of Lemma~\ref{lem:e1_1}]

We start by rewriting $(e1.1)$ as follows, where we drop $\Ecal_t$.

\begin{align*}
(e1.1) \leq & \EE \sbr{ \sum_{t: p \in \Pcal_t}  {\rbr{\frac{1 - \psi_{i,t}^p}{\psi_{i,t}^p}} \indic \cbr{i_t^p = \optarm, H_\optarm^p(t), n_\optarm^p(t-1) < J} }} \\
= &  \EE \sbr{ \sum_{t: p \in \Pcal_t} { g_t \indic \cbr{i_t^p = \optarm, n_\optarm^p(t-1) < J} }},
\end{align*}
where in the second line, we introduce the notation $g_t := \rbr{ \frac{1 - \psi_{i,t}^p}{\psi_{i,t}^p} } \indic\cbr{ H_\optarm^p(t) }$;

We now focus on the sum inside the expectation. %
Recall that $\pi_s(\optarm,p)$ is the round in which player $p$ pulls arm $\optarm$ the $s$-th time. Here, we abuse the notation and denote $\pi_s(\optarm,p)$ by $\pi_s$.
By Remark~\ref{rem:invariant-more}, $\cbr{g_t: t \in [T]}$ satisfies the invariant property with respect to $(\optarm, p)$. Applying Lemma~\ref{lem:auxiliary_sum_g_indic} on $\cbr{g_t: t \in [T]}$'s, we have that the term inside the above expectation is at most:
\[
 \sum_{s=1}^{J-1} \rbr{\frac{1}{\psi_{i, \pi_s+1}^{p}} - 1} \indic \cbr{\pi_s \le T, H_\optarm^p(\pi_s+1)},
\]
where we also use the observation that $\rbr{\frac{1}{\psi_{i,1}^p} - 1} \indic\cbr{ H_\optarm^p(1) } = 0$.

Therefore, by the linearity of expectation, we have
\begin{align}
(e1.1) 
\le \sum_{s=1}^{J-1} \EE \sbr{ \rbr{\frac{1}{\psi_{i, \pi_s+1}^{p}} - 1} \indic \cbr{\pi_s \le T, H_\optarm^p(\pi_s+1)}}. \nonumber \label{eqn:e1_psi}
\end{align}
Therefore, the following Lemma~\ref{lem:key_lemma_constant_subpar_H} suffices to prove Lemma~\ref{lem:e1_1}, which we prove next.
\end{proof}

\begin{lemma}
\label{lem:key_lemma_constant_subpar_H}
For any $s \in [T]$,
\[
\EE \sbr{\rbr{\frac{1}{\psi_{i, \pi_s+1}^{p}} - 1} \indic \cbr{\pi_s \le T, H_\optarm^p(\pi_s+1)}} \le \Ocal(1),
\]
where we recall that $\pi_s = \pi_s(\optarm,p)$ is the round in which player $p$ pulls arm $\optarm$ the $s$-th time.
\end{lemma}

\begin{proof}[Proof of Lemma~\ref{lem:key_lemma_constant_subpar_H}]
We note that this proof is similar to that of Lemma~\ref{lem:key_lemma_constant_appendix}.
We have
\begin{align*}
& \EE \sbr{\rbr{\frac{1}{\psi_{i, \pi_s+1}^{p}} - 1} \indic \cbr{\pi_s \le T, H_\optarm^p(\pi_s+1)}} \\
= & \EE \sbr{ \rbr{\frac{1}{\bPhi \rbr{\rbr{z_i^{p} - \indmu_{\optarm}^{p}(\pi_s)}\sqrt{ \wbar{n_{\optarm}^{p}}(\pi_s)/4}}} - 1} \indic \cbr{\pi_s \le T, H_{\optarm}^{p}(\pi_s+1)} } \\
\leq & 
\EE \sbr{ \frac{1}{\bPhi \rbr{\rbr{\mu_\optarm^{p} - \indmu_{\optarm}^{p}(\pi_s)}\sqrt{ \wbar{n_{\optarm}^{p}}(\pi_s)/4}}}  \indic \cbr{\pi_s \le T} },
\end{align*}
where the inequality drops $H_\optarm^p(\tau_s+1)$ and uses the observation that $z_i^{p} \leq \mu_\optarm^{p}$, and the monotonic increasing property of $z \mapsto \frac{1}{\bPhi(z)}$.
Now, using Lemma~\ref{lem:sg-anticonc} and Corollary~\ref{col:ind_concentration_alternative}, we conclude that this is at most $\Ocal (1)$.
\end{proof}

\begin{lemma}
\label{lem:e1_2}
\[
(e1.2) \le \Ocal \rbr{ 1}.
\]
\end{lemma}

\begin{proof}

Recall that
\[
(e1.2) = \sum_{t: p \in \Pcal_t} \EE \sbr{  {\rbr{\frac{1 - \psi_{i,t}^p}{\psi_{i,t}^p}} \indic \cbr{i_t^p = \optarm, \Ecal_t, H_\optarm^p(t), n_\optarm^p(t-1) \ge J} }}.
\]
Dropping $\indic\cbr{i_t^p = i_\optarm^p}$, we have
\[
(e1.2) \leq \sum_{t: p \in \Pcal_t} \EE \sbr{  {\rbr{\frac{1 - \psi_{i,t}^p}{\psi_{i,t}^p}} \indic \cbr{\Ecal_t, H_\optarm^p(t), n_\optarm^p(t-1) \geq J} }},
\]

When $\Ecal_t$, $H_\optarm^p(t)$, and $\cbr{n_\optarm^p(t-1) \geq J}$ happen, we have 
\begin{align*}
    & 1 - \psi_{i,t}^p \\
    = & \Pr \rbr{ \theta_{\optarm}^p(t) \le z_i^p \mid \Fcal_{t-1}} \\
    = & \Phi\rbr{ (z_i^p - \indmu_\optarm^p(t-1)) \sqrt{ \wbar{n_\optarm^p}(t-1) / 4 } } \\
    \leq & \exp \del{ - 
    \frac{ \wbar{n_\optarm^p}(t-1) \rbr{\indmu_\optarm^p(t-1) - z_i^p}_+^2   }{ 8 }  } \\
     \leq & \exp \del{ - 
    \frac{ n_\optarm^p(t-1) \rbr{\mu_\optarm^p - \frac14 \Delta_i^p - \mu_\optarm^p + \frac12 \Delta_i^p}_+^2  }{ 8 }  } \\
     \leq & \exp \del{ - 
    \frac{ n_\optarm^p(t-1) (\Delta_i^p)^2  }{ 8 (16) }  } \\
    \leq & \frac{1}{T+1},
\end{align*}
where the first inequality uses Lemma~\ref{lem:phi-bounds}; 
the second inequality uses the observations that, when $\Ecal_t$ and $\cbr{n_\optarm^p(t-1) \geq J}$ happen:
\begin{enumerate}
    \item $\wbar{n_\optarm^p}(t-1) \ge n_\optarm^p(t-1) \ge J = \frac{640 \ln T}{(\Delta_i^p)^2}$,
    
    \item $\indmu_{\optarm}^p(t-1) \ge \mu_{\optarm}^p - \sqrt{ \frac{10 \ln T}{\wbar{n_\optarm^p}(t-1)}} \ge \mu_{\optarm}^p - \frac{1}{4}\Delta_i^p$ (see Definition~\ref{def:event_e}), and
    
    \item $z_i^p = \mu_{\optarm}^p - \frac12 \Delta_i^p$;
\end{enumerate}
the third inequality is by algebra; 
and the last inequality follows because when $\cbr{n_\optarm^p(t-1) \geq J}$ happens, $n_\optarm^p(t-1) \ge \frac{640 \ln T}{(\Delta_i^p)^2} \ge \frac{320 \ln (T+1)}{(\Delta_i^p)^2}$ for $T > 1$. 

It follows that, when $\Ecal_t$ and $\cbr{n_\optarm^p(t-1) \geq J}$ happen, $\psi_{i,t}^p \ge \frac{T}{T+1}$ and $\frac{1 - \psi_{i,t}^p}{\psi_{i,t}^p} \le \frac{1}{T}$. Hence, $(e1.2) \le 1$.
\end{proof}

We now consider term $(e2)$. Recall that
\[
(e2) = \sum_{t: p \in \Pcal_t} \EE \sbr{  {\rbr{\frac{1 - \psi_{i,t}^p}{\psi_{i,t}^p}} \indic \cbr{i_t^p = \optarm, \Ecal_t, \wbar{H_\optarm^p(t)}} }}
\]
\begin{lemma}
\label{lem:e2}
\[
(e2) \le \Ocal \rbr{ \frac{\ln T}{(\Delta_i^p)^2} + M}.
\]
\end{lemma}

\begin{proof}[Proof of Lemma~\ref{lem:e2}]

With foresight, let $Z = \frac{640 \ln T}{(\Delta_i^{p})^2} + M$. We have
\begin{align*}
(e2) & = \underbrace{\sum_{t: p \in \Pcal_t} \EE \sbr{  {\rbr{\frac{1 - \psi_{i,t}^p}{\psi_{i,t}^p}} \indic \cbr{ i_t^p = \optarm, \Ecal_t, \wbar{H_\optarm^p(t)}, m_\optarm^p(t-1) < Z} }}}_{(e2.1)}
+ \\
& \hspace{100pt} \underbrace{\sum_{t: p \in \Pcal_t} \EE \sbr{  {\rbr{\frac{1 - \psi_{i,t}^p}{\psi_{i,t}^p}} \indic \cbr{ i_t^p = \optarm, \Ecal_t, \wbar{H_\optarm^p(t)}, m_\optarm^p(t-1) \ge Z} }}}_{(e2.2)}.
\end{align*}

The proof follows straightforwardly from Lemma~\ref{lem:e2.1} and Lemma~\ref{lem:e2.2} which we present subsequently.
\end{proof}

\begin{lemma}
\label{lem:e2.1}
\[
(e2.1) \le \Ocal \rbr{ \frac{\ln T}{(\Delta_i^{p})^2} + M}.
\]
\end{lemma}

\begin{proof}[Proof of Lemma~\ref{lem:e2.1}]
We have
\begin{align*}
(e2.1) 
& \le \EE \sbr{ \sum_{t: p \in \Pcal_t} { \rbr{\frac{1}{\psi_{i,t}^p}-1} \indic \cbr{i_t^p = \optarm, \wbar{H_\optarm^p(t)}, m_\optarm^p(t-1) < Z} }} \\
& \leq \EE \sbr{ \sum_{t: p \in \Pcal_t} {\frac{1}{\psi_{i,t}^p} \indic \cbr{i_t^p = \optarm, \wbar{H_\optarm^p(t)}, m_\optarm^p(t-1) < Z} }},
\end{align*}
where we drop $\Ecal_t$ and use the observation that $\frac{1}{\psi_{i,t}^p}-1 \le \frac{1}{\psi_{i,t}^p}$.

We now focus on sum inside the expectation. 
We denote $\tau_k(\optarm)$ by $\tau_k$ and the player that makes the $k$'s pull of $\optarm$ by $p_k := p_k(\optarm)$. Recall that we use $\wbar{m_\optarm^p}(t-1)$ to denote $\rbr{m_\optarm^p(t-1) - M} \vee 1$. 
We have
\begin{align}
& \sum_{t: p \in \Pcal_t} \frac{1}{\psi_{i,t}^p} \indic \cbr{i_t^p = \optarm, \wbar{H_\optarm^p(t)}, m_\optarm^p(t-1) < Z} \nonumber \\
= & \sum_{t: p \in \Pcal_t} \frac{1}{\bPhi\rbr{ \rbr{ z_i^p - \aggmu_\optarm^p( t-1 )} \sqrt{ \wbar{m_\optarm^p}(t-1)/4} } } \indic \cbr{i_t^p = \optarm, \wbar{H_\optarm^p(t)}, m_\optarm^p(t-1) < Z} 
\nonumber \\
\leq & \sum_{t: p \in \Pcal_t} \frac{1}{\bPhi\rbr{ \rbr{ \mu_\optarm^p - \aggmu_\optarm^p( t-1 )} \sqrt{ \wbar{m_\optarm^p}(t-1)/4} } } \indic \cbr{i_t^p = \optarm, m_\optarm^p(t-1) < Z} 
\label{eqn:e_2_1_f_p} \\
\leq & \sum_{t \in [T]} \sum_{q \in \Pcal_t} \frac{1}{\bPhi\rbr{ \rbr{ \mu_\optarm^q - \aggmu_\optarm^q( t-1 )} \sqrt{ \wbar{m_\optarm^q}(t-1)/4} } } \indic \cbr{i_t^q = \optarm, m_\optarm^q(t-1) < Z},
\label{eqn:e_2_1_f}
\end{align}

where the first equality uses Remark~\ref{rem:psi_explicit};
the first inequality drops $\wbar{H_\optarm^p(t)}$ and uses the observation that $z_i^p \le \mu_\optarm^p$ (see Definition~\ref{def:z_and_W}), along with the monotonic increasing property of $z \mapsto \frac{1}{\bPhi(z)}$.;
the second inequality adds similar terms for other players $q \neq p$. 

Now, define 
$\cbr{f_t^q: t \in [T], q \in [M]}$ where
$f_t^q = \frac{1}{\bPhi\rbr{ \rbr{ \mu_\optarm^q - \aggmu_\optarm^q( t-1 )} \sqrt{ \wbar{m_\optarm^q}(t-1)/4} } }$; recall from Example~\ref{ex:m-n-invariant} that $\cbr{\aggmu_\optarm^q(t-1): t \in [T]}$ and $\cbr{m_\optarm^q(t-1): t \in [T]}$ both satisfy the invariant property with respect to $(\optarm, q)$; therefore, $\cbr{f_t^q: t \in [T], q \in [M]}$ satisfies the invariant property with respect to $\optarm$. Applying Lemma~\ref{lem:auxiliary_sum_f_indic} to it, we have that

\[
\eqref{eqn:e_2_1_f}
\le \sum_{q \in [M]} \frac{1}{\bPhi \rbr{0}} + \sum_{k=1}^{Z-1} \frac{1}{ \bPhi \rbr{ \rbr{\mu_\optarm^{p_k} - \aggmu_\optarm^{p_k}(\tau_k) \sqrt{\wbar{m_\optarm^{p_k}} (\tau_k) /4}}    }} \indic \cbr{\tau_k \le T}.
\]

Since $\sum_{q \in [M]} \frac{1}{\bPhi \rbr{0}} \le \Ocal \rbr{M}$, it then suffices to show that for every $k \in \NN$,
\begin{align}
\EE\sbr{ \frac{1}{\bPhi\rbr{ ( \mu_\optarm^{p_k} - \aggmu_\optarm^{p_k}( \tau_k ) ) \sqrt{ \wbar{m_\optarm^{p_k}}(\tau_k)/4} } } \indic\cbr{ \tau_k \leq T } } 
\leq 
\order\rbr{ 1 }. \label{eqn:aggreg_posterior_nonsubpar_eqn}
\end{align}

Note that $\wbar{m_\optarm^{p_k}}(\tau_k) = \rbr{n_\optarm(\tau_k) - M} \vee 1$. Directly applying Corollary~\ref{col:agg_concentration_alternative} and Lemma~\ref{lem:sg-anticonc} with $X = ( \aggmu_\optarm^{p_k}( \tau_k ) - \mu_\optarm^{p_k} ) \sqrt{ \wbar{m_\optarm^{p_k}}(\tau_k)/4}$ and $E = \cbr{ \tau_k \leq T }$ proves Eq.~\eqref{eqn:aggreg_posterior_nonsubpar_eqn}.
\end{proof}

\begin{remark}
In the above proof, we relaxed Eq.~\eqref{eqn:e_2_1_f_p} to Eq.~\eqref{eqn:e_2_1_f} by adding the corresponding terms for all other players $q \neq p$. 
Alternatively, we could use the observation that $n_\optarm^p(t-1) \leq m_\optarm^p(t-1)$
to bound Eq.~\eqref{eqn:e_2_1_f_p} by 
\[
\sum_{t: p \in \Pcal_t} \frac{1}{\bPhi\rbr{ \rbr{ \mu_\optarm^p - \aggmu_\optarm^p( t-1 )} \sqrt{ \wbar{m_\optarm^p}(t-1)/4} } } \indic \cbr{i_t^p = \optarm, n_\optarm^p(t-1) < Z},
\]
and apply Lemma~\ref{lem:auxiliary_sum_g_indic} and subsequently Lemma~\ref{lem:sg-anticonc}. However, right now, we do not have tight-enough concentration inequalities for $\aggmu_\optarm^p( \pi_k(\optarm,p) )$---the best known inequality here is Freedman's inequality, which incurs an undesirable extra $\order\rbr{\ln T}$ factor in the bound for $(e2.1)$.
\end{remark}

\begin{lemma}
\label{lem:e2.2}
\[
(e2.2) \le \Ocal \rbr{ 1}.
\]
\end{lemma}

\begin{proof}[Proof of Lemma~\ref{lem:e2.2}]
Recall that
\begin{align*}
(e2.2) & = \sum_{t: p \in \Pcal_t} \EE \sbr{  {\rbr{\frac{1 - \psi_{i,t}^p}{\psi_{i,t}^p}} \indic \cbr{ i_t^p = \optarm, \Ecal_t, \wbar{H_\optarm^p(t)}, m_\optarm^p(t-1) \ge Z} }}.
\end{align*}

Recall that $\wbar{m_\optarm^p}(t-1) = \rbr{m_\optarm^p(t-1) - M} \vee 1$.
When $\Ecal_t, \wbar{H_{\optarm}^p(t)}$ and $\cbr{m_\optarm^p(t-1) \ge Z}$ happen simultaneously,  
\begin{align*}
    & 1 - \psi_{i,t}^p \\
    = & \Pr \rbr{ \theta_{\optarm}^p(t) \le z_i^p \mid \Fcal_{t-1}} \\
    = & \Phi\rbr{ \rbr{z_i^p - \aggmu_\optarm^p(t-1)} \sqrt{ \wbar{m_\optarm^p}(t-1) / 4 } } \\
    \leq & \exp \del{ - 
    \frac{ \wbar{m_\optarm^p}(t-1) \rbr{\aggmu_\optarm^p(t-1) - z_i^p}_+^2  }{ 8 }  } \\
     \leq & \exp \del{ - 
    \frac{ \wbar{m_\optarm^p}(t-1) \rbr{\mu_\optarm^p - \frac14 \Delta_i^p - \mu_\optarm^p + \frac12 \Delta_i^p}_+^2  }{ 8 }  } \\
     \leq & \exp \del{ - 
    \frac{ \wbar{m_\optarm^p}(t-1) (\Delta_i^p)^2  }{ 8(16) }  } \\
    \leq & \frac{1}{T+1},
\end{align*}
where the first inequality uses Lemma~\ref{lem:phi-bounds}; 
the second inequality uses the observations that when $\Ecal_t, \wbar{H_{\optarm}^p(t)}$ and $\cbr{m_\optarm^p(t-1) \ge Z}$ happen:
\begin{enumerate}
    \item $\wbar{m_\optarm^p}(t-1) \ge m_\optarm^p(t-1) - M \ge Z - M \ge \frac{640\ln T}{(\Delta_i^{p})^2}$,
    
    \item $\aggmu_{\optarm}^p(t-1) \ge \mu_{\optarm}^p - \sqrt{ \frac{10 \ln T}{\wbar{m_\optarm^p}(t-1)}} \ge \mu_{\optarm}^p - \frac{1}{4}\Delta_i^p$ (see Definition~\ref{def:event_e}), and
    
    \item $z_i^p = \mu_{\optarm}^p - \frac12 \Delta_i^p$ (see Definition~\ref{def:z_and_W});
\end{enumerate}
the third inequality is by algebra; 
and the fourth inequality is by the fact that when $m_\optarm^p(t-1) \ge Z$, $\wbar{m_\optarm^p}(t-1)\geq Z-M = \frac{640\ln T}{(\Delta_i^{p})^2} \ge \frac{320 \ln(T+1)}{(\Delta_i^p)^2}$ for $T > 1$.

It follows that, when $\Ecal_t, \wbar{H_{\optarm}^p(t)}$ and $\cbr{m_\optarm^p(t-1) \ge Z}$ happen, $\psi_{i,t}^p \ge \frac{T}{T+1}$ and $\frac{1 - \psi_{i,t}^p}{\psi_{i,t}^p} \le \frac{1}{T}$. As a result, $(e2.2) \le \Ocal (1)$.
\end{proof}

\subsection{Concluding the proofs of Theorems~\ref{thm:ts_gap_dep_ub} and~\ref{thm:ts_gap_indep_ub}}
\label{sec:appendix_concluding_proofs}
\begin{lemma}
\label{lem:arm-pull-reg}
Let a generalized $\epsilon$-MPMAB problem instance and $\alpha > 0$ be such that for all $i \in \Ical_\alpha$ and all $p \in [M]$,  $\Delta_i^p \leq 2 \Delta_i^{\min}$.
If algorithm $\Acal$  guarantees that when interacting with this problem instance:
\begin{enumerate}
    \item For any arm $i \in \Ical_{\alpha}$,
    \begin{equation}
    \EE \sbr{n_i(T)} \le \Ocal \rbr{\frac{\ln T}{(\Delta_i^{\min})^2} + M};
    \label{eqn:i-eps}
    \end{equation}
    \item For any arm $i \in \Ical_{\alpha}^C$ and player $p \in [M]$,
    \begin{equation}
    \EE \sbr{n_i^p(T)} \le \Ocal \rbr{\frac{\ln T}{(\Delta_i^p)^2} + C},
    \label{eqn:i-eps-c}
    \end{equation}
\end{enumerate}
for some $C \geq 0$,
then it has the following regret bounds simultaneously:
\begin{enumerate}
    \item gap-dependent regret bound:
    \begin{equation}
    \Reg(T) \le \Ocal \vast( \frac 1 M \sum_{i \in \Ical_{\alpha}} \sum_{p \in [M]: \Delta_i^p>0} \frac{\ln T}{\Delta_i^p} + \sum_{i \in \Ical_{\alpha}^C} \sum_{\substack{p \in [M]: \Delta_i^p>0}} \frac{\ln T}{\Delta_i^p} + MK(1+C) \vast),
    \label{eqn:gap-dept}
    \end{equation}
    \item gap-independent regret bound:
    \begin{equation}
    \Reg(T) \le \tilde{\Ocal} \vast(  \sqrt{ |\Ical_{\alpha}| P } + \sqrt{ M \rbr{|\Ical_{\alpha}^C| -1 } P } + MK(1+C)   \vast),
    \label{eqn:gap-indept}
    \end{equation}
    where we recall that $\rounds = \sum_{t=1}^T \abr{ \Pcal_t }$.
\end{enumerate}
\end{lemma}
\begin{proof}
We prove the two items respectively. Recall that $\Delta_i^{\min} = \min_{p \in [M]} \Delta_i^p$. 
\begin{enumerate}
\item Note that for all $i \in \Ical_\alpha$ and all $p \in [M]$,  $\Delta_i^p \leq 2 \Delta_i^{\min}$, and $\sum_{p=1}^M \EE\sbr{n_i^p(T)} = \EE\sbr{n_i(T)}$; as a consequence,
\begin{equation}
\Reg(T) 
= 
\sum_{p=1}^M \sum_{i=1}^K \EE\sbr{n_i^p(T)} \Delta_i^p
= 
O
\rbr{
\sum_{i \in \Ical_\alpha} \EE\sbr{n_i(T)} \Delta_i^{\min}
+
\sum_{i \in \Ical_\alpha^C} \sum_{p \in [M]: \Delta_i^p > 0} \EE\sbr{n_i^p(T)} \Delta_i^p
}.
\label{eqn:reg-decomp}
\end{equation}
Using Eq.~\eqref{eqn:i-eps}, the first term can be bounded by: 
\[
\sum_{i \in \Ical_\alpha} \EE\sbr{n_i(T)} \Delta_i^{\min} 
\leq 
\order\rbr{ \sum_{i \in \Ical_\alpha}  \frac{\ln T}{\Delta_i^{\min}} + M K }
\leq 
\order\rbr{ \frac 1 M \sum_{i \in \Ical_\alpha} \sum_{p \in [M]: \Delta_i^p>0} \frac{\ln T}{\Delta_i^p} + M K },
\]
where the second inequality follows from the assumption that for all $i \in \Ical_\alpha$ and $p \in [M]$,  $\Delta_i^p \leq 2 \Delta_i^{\min}$.

Using Eq.~\eqref{eqn:i-eps-c}, the second term can be bounded by:
\[
\sum_{i \in \Ical_\alpha^C} \sum_{p \in [M]: \Delta_i^p > 0} \EE\sbr{n_i^p(T)} \Delta_i^p
\leq 
\order \rbr{ \sum_{i \in \Ical_{\alpha}^C} \sum_{\substack{p \in [M]: \Delta_i^p>0}} \frac{\ln T}{\Delta_i^p} + MKC }.
\]
Combining the above two bounds yields Eq.~\eqref{eqn:gap-dept}.

\item As with the proof of Eq.~\eqref{eqn:gap-indept}, we continue from Eq.~\eqref{eqn:reg-decomp}, but look at the two terms respectively. For the first  term,
\begin{align}
    \sum_{i \in \Ical_\alpha} \EE\sbr{n_i(T)} \Delta_i^{\min}
    \leq &  \order \rbr{ \sum_{i \in \Ical_\alpha} \min\rbr{ \EE\sbr{n_i(T)}, \frac{\ln T}{(\Delta_i^{\min})^2}  + M  } \Delta_i^{\min} } \nonumber \\
    \leq &  \order \rbr{ \sum_{i \in \Ical_\alpha} \min\rbr{ \EE\sbr{n_i(T)} \Delta_i^{\min}, \frac{\ln T}{\Delta_i^{\min}}  } + MK } \nonumber \\
    \leq & \order\rbr{ \sum_{i \in \Ical_\alpha} \sqrt{ \EE\sbr{n_i(T)} \ln T }  + MK } \nonumber \\
    \leq & \order\rbr{ \sqrt{ |\Ical_\alpha| \rounds \ln T }  + MK }
    \label{eqn:term1}
\end{align}
where the first inequality is from Eq.~\eqref{eqn:i-eps}; the second inequality is by algebra; the third inequality is from the elementary fact that $\min(A,B) \leq \sqrt{AB}$; the last inequality is from Jensen's inequality and the concavity of function $x \mapsto \sqrt{x}$, which implies that $\sum_{i \in \Ical_\alpha} \sqrt{ \EE\sbr{n_i(T)} } \leq \sqrt{ \abr{\Ical_\alpha} \rbr{\sum_{i \in \Ical_\alpha} \EE\sbr{n_i(T)} } }$, and the fact that $\sum_{i \in \Ical_\alpha} \EE\sbr{n_i(T)} \leq \sum_{i=1}^M \EE\sbr{n_i(T)} \leq \rounds$. 

For the second term in Eq.~\eqref{eqn:gap-indept}, first observe that if $\abr{\Ical_\alpha^C} = 1$, then let $i^*$ be the only element in $\Ical_\alpha^C$; it must be the case that for all $p \in [M]$, $i^*$ is the optimal arm for player $p$. As a consequence,
$\sum_{i \in \Ical_\alpha^C} \sum_{p=1}^M \EE\sbr{n_i^p(T)} \Delta_i^p = 0 = \order(\sqrt{ M (\abr{\Ical_\alpha^C} - 1) \rounds })$.

Otherwise, $\abr{\Ical_\alpha^C} \geq 2$. In this case,
\begin{align*}
    \sum_{p \in [M]} \sum_{i \in \Ical_\alpha^C} \EE\sbr{n_i^p(T)} \Delta_i^p
    \leq & 
    \order\rbr{ \sum_{p \in [M]} \sum_{i \in \Ical_\alpha^C} \min\rbr{ \EE\sbr{n_i^p(T)}, \frac{\ln T}{ (\Delta_i^p)^2 }  } \Delta_i^p + MKC } \\
    \leq & 
    \order\rbr{ \sum_{p \in [M]} \sum_{i \in \Ical_\alpha^C} \min\rbr{ \EE\sbr{n_i^p(T)} \Delta_i^p , \frac{\ln T}{ \Delta_i^p }  } + MKC } \\
    \leq &  
    \order\rbr{ \sum_{p \in [M]} \sum_{i \in \Ical_\alpha^C} \sqrt{\EE\sbr{n_i^p(T)} \ln T} + MKC } \\
    \leq &
    \order\rbr{ \sqrt{M \abr{\Ical_\alpha^C} \rounds \ln T} + MKC } \\
    \leq & 
    \order\rbr{ \sqrt{M \rbr{\abr{\Ical_\alpha^C} - 1} \rounds \ln T} + MKC }. 
\end{align*}
where the first inequality is by Eq.~\eqref{eqn:i-eps-c} and algebra; the second inequality is by algebra; the third inequality is from the elementary fact that $\min(A,B) \leq \sqrt{A B}$; the fourth inequality is from Jensen's inequality and the concavity of function $x \mapsto \sqrt{x}$, which implies that 
$\sum_{i \in \Ical_\alpha} \sqrt{ \EE\sbr{n_i(T)} } \leq \sqrt{ \abr{\Ical_\alpha} \rbr{\sum_{i \in \Ical_\alpha} \EE\sbr{n_i(T)} } }$, and the fact that $\sum_{i \in \Ical_\alpha} \EE\sbr{n_i(T)} \leq \sum_{i=1}^M \EE\sbr{n_i(T)} \leq P$; the last inequality is from the simple observation that $\abr{\Ical_\alpha^C} \leq 2(\abr{\Ical_\alpha^C} - 1)$ when $\abr{\Ical_\alpha^C} \geq 2$. 

In summary, $\sum_{p=1}^M \sum_{i \in \Ical_\alpha^C} \EE\sbr{n_i^p(T)} \Delta_i^p \leq \order\rbr{ \sqrt{M \rbr{\abr{\Ical_\alpha^C} - 1} \rounds \ln T}} + MKC$. Combining this with Eq.~\eqref{eqn:term1}, this concludes the proof of  Eq.~\eqref{eqn:gap-indept}.
\qedhere
\end{enumerate}
\end{proof}

\begin{proof}[Proofs of Theorems~\ref{thm:ts_gap_dep_ub} and~\ref{thm:ts_gap_indep_ub}]
Combining Lemmas~\ref{lem:i-eps-robustagg-ts}, ~\ref{lem:i-eps-c-robustagg-ts},~\ref{lem:arm-pull-reg} with $C = M$ and $\alpha=10\epsilon$, Theorems~\ref{thm:ts_gap_dep_ub} and~\ref{thm:ts_gap_indep_ub} follow immediately.
\end{proof}

\subsection{Auxiliary Lemmas}
\label{sec:appendix_auxiliary}

Recall that we denote by 
$\wbar{\Phi}(x) = \int_{x}^\infty \frac{1}{\sqrt{2\pi}} e^{-\frac{z^2}{2}} dz$ the complementary CDF of the standard normal distribution.

\begin{lemma}
$\bPhi$ is monotonically decreasing. In addition,
for $z \geq 0$,
\[
\frac{1}{\sqrt{2\pi}} \frac{z}{z^2+1}
\exp\rbr{-\frac{z^2}{2}}
\leq 
\wbar{\Phi}(z)
\leq 
\exp\rbr{-\frac{z^2}{2}},
\]
where the first inequality (anti-concentration) is from \cite{gordon1941millsratio}.
In addition, for any $z \in \RR$,
\[
\wbar{\Phi}(z)
\leq 
\exp\rbr{-\frac{(z)_+^2}{2}},
\quad
\Phi(z)
\leq 
\exp\rbr{-\frac{(-z)_+^2}{2}},
\]
where we recall that $(z)_+ = \max(z, 0)$.
\label{lem:phi-bounds}
\end{lemma}

The following lemma is useful in bounding 
$(b2.1)$, $(e1.1)$, $(e2.1)$; 
it can also be used to provide a simplified proof of the first case of~\citet[][Lemma 2.13]{an17}. 
Roughly speaking, the lemma shows that a random variable $X$ with a light lower probability tail must have a small value of $\EE\sbr{ \frac{1}{ \wbar{\Phi}(-X) } }$; it crucially uses the lower bound on $\bPhi$ (Gaussian anti-concentration) given in Lemma~\ref{lem:phi-bounds}. 

\begin{lemma}
\label{lem:sg-anticonc}
There exists some absolute constants $c_1, c_2 > 0$ such that the following holds.
Given a random variable $X$, an event $E$ and some $C_1 > 0$; 
if, for every $z \ge 1$, $\PP(X \leq -z, E) \leq C_1 \exp(-2 z^2)$,  such that
\[
\EE\sbr{ \frac{1}{ \wbar{\Phi}(-X) } \indic \cbr{E} } 
\leq c_1 C_1 + c_2.
\]
\end{lemma}

\begin{proof}
Define $Y = -X$; we have $\PP(Y \geq z, E) \leq C_1 \exp(-2 z^2)$ for all $z \ge 1$.
\begin{align*}
    & \EE\sbr{ \frac{1}{ \wbar{\Phi}(-X) } \indic \cbr{E} } \\
    = & \EE\sbr{ \frac{1}{ \wbar{\Phi}(-X) } \indic \cbr{E, X \leq -1} } + \EE\sbr{ \frac{1}{ \wbar{\Phi}(-X) } \indic \cbr{E, X \geq -1} } \\
    \leq & 
    \EE\sbr{ \frac{1}{ \wbar{\Phi}(Y) } \indic \cbr{ E, Y \geq 1} }
    +
    \frac{1}{\bPhi(1)} \\
    \leq & 
    8 \sqrt{2\pi} \cdot  \EE\sbr{ e^{Y^2} \indic \cbr{ E, Y \geq 1} } + \frac{1}{\bPhi(1)}.
\end{align*}
where 
the first inequality follows due to the fact that $\frac{1}{\bPhi(z)}$ increases monotonically as $z$ increases;
and the second inequality is based on the observation that for $y \geq 1$, $\frac{1}{\bPhi(y)} \leq \sqrt{2\pi} \frac{y^2 + 1}{y} \exp(\frac{y^2}{2})\leq 8 \sqrt{2\pi} e^{y^2}$ (see Lemma~\ref{lem:phi-bounds}).

It suffices to show that $\EE\sbr{ e^{Y^2} \indic \cbr{E, Y \geq 1} }$ is bounded by some constant, given the assumption on $Y$. Define $W = e^{Y^2} \indic \cbr{E, Y \geq 1}$. We have that for any $w \geq e$, 
\[ 
\PP(W \geq w) 
=
\PP(E, Y \geq \sqrt{\ln w})
\leq 
\frac{C_1}{w^2}. 
\]
As a result,
\begin{align*}
\EE\sbr{W}
= &
\int_0^\infty \PP(W \geq w) \dif w \\
= & 
\int_0^e \PP(W \geq w) \dif w + \int_e^\infty \PP(W \geq w) \dif w \\
\leq & e + \int_e^\infty \frac{C_1}{w^2} \dif w \\
\leq & e + \frac{C_1}{e},
\end{align*}

Therefore, the lemma holds by taking
$c_1 = \frac{8\sqrt{2\pi}}{e}$ and 
$c_2 = 8\sqrt{2\pi} e + \frac{1}{\bPhi(1)}$.
\end{proof}

The following two lemmas are useful in bounding $(e1.1)$ (Lemma~\ref{lem:auxiliary_sum_g_indic}), as well as 
$(b2.1)$ and $(e2.1)$ (Lemma~\ref{lem:auxiliary_sum_f_indic}), respectively.

\begin{lemma}
\label{lem:auxiliary_sum_g_indic}
Fix any arm $i \in [K]$ and player $p \in [M]$. Let $N \in \NN^+$. 
Suppose $\cbr{g_t: t \in [T]}$ satisfies the invariant property with respect to $(i,p)$ (Definition~\ref{def:invariant}). 
Then,
\[
\sum_{t: p \in \Pcal_t} g_t
\indic\cbr{ i_t^p = i, n_i^p(t-1) < N }
\le
g_1
+
\sum_{k=1}^{N-1} g_{\pi_k+1} \indic \cbr{\pi_k \le T},
\]
where $\pi_k = \pi_k(i,p)$ denotes the round associated with the $k$-th pull of arm $i$ by  player $p$.
\end{lemma}
\begin{proof}
Let $h_t = g_t \indic\cbr{ n_i^p(t-1) < N }$. As seen in Example~\ref{ex:m-n-invariant}, $\cbr{n_i^p(t-1): t \in [T]}$ satisfies the invariant property with respect to $(i,p)$. This, combined with the assumption that $\cbr{g_t: t \in [T]}$ satisfies the invariant property with respect to $(i,p)$, implies that $\cbr{h_t: t \in [T]}$ is also invariant with respect to $(i,p)$. Applying Lemma~\ref{lem:auxiliary_sum_h} to the above $\cbr{h_t: t \in [T]}$, we have
\begin{align*}
\sum_{t: p \in \Pcal_t} g_t
\indic\cbr{ i_t^p = i, n_i^p(t-1) < N } 
= & \sum_{t: p \in \Pcal_t} h_t
\indic\cbr{ i_t^p = i } \\
\leq & h_1 
+
\sum_{k=1}^{T} h_{\pi_k+1} \indic \cbr{\pi_k \le T} \\
= & g_1 \indic\cbr{ n_i^p(0) < N } 
+
\sum_{k=1}^{T} g_{\pi_k+1}
\indic \cbr{n_i^p(\pi_k) < N}
\indic \cbr{\pi_k \le T} \\
= & 
g_1
+
\sum_{k=1}^{T} g_{\pi_k+1}
\indic \cbr{k < N}
\indic \cbr{\pi_k \le T} \\
= & 
g_1
+
\sum_{k=1}^{N-1} g_{\pi_k+1}
\indic \cbr{\pi_k \le T},
\end{align*}
where the first inequality is by Equation~\eqref{eqn:auxiliary_sum_h_1} in Lemma~\ref{lem:auxiliary_sum_h}; 
the second equality is by expanding the definition of $h_t$'s; the third equality is from that $n_i^p(0) = 0$ and $n_i^p(\pi_k) = k$; and the last eqaulity is by algebra.
\end{proof}

\begin{lemma}
\label{lem:auxiliary_sum_f_indic}
Fix any arm $i \in [K]$ and let $N \in \NN^+$. 
Suppose $\cbr{f_t^p: t \in [T], p \in [M]}$ satisfies the invariant property with respect to arm $i$ (Definition~\ref{def:invariant}),
then,
\[
\sum_{t \in [T]} \sum_{p \in \Pcal_t} f_t^p
\indic\cbr{ i_t^p = i, m_i^p(t-1) < N }
\le
\sum_{p \in [M]} f_1^p 
+
\sum_{k=1}^{N-1} f_{\tau_k+1}^{p_k} \indic \cbr{\tau_k \le T},
\]
where $(\tau_k, p_k) = (\tau_k(i), p_k(i))$ denote the round and player associated with the $k$-th pull of arm $i$ by all players.
\end{lemma}

\begin{proof}[Proof of Lemma~\ref{lem:auxiliary_sum_f_indic}]
First, consider any fixed player $p \in [M]$; let $h_t = f_t^p \indic\cbr{ m_i^p(t-1) < N }$. As seen in Example~\ref{ex:m-n-invariant}, $\cbr{m_i^p(t-1): t \in [T]}$ satisfies the invariant property with respect to $(i,p)$. This, combined with the assumption that $\cbr{f_t^p: t \in [T]}$ satisfies the invariant property with respect to $(i,p)$, implies that $\cbr{h_t: t \in [T]}$ is also invariant with respect to $(i,p)$. Applying Lemma~\ref{lem:auxiliary_sum_h} to the above $\cbr{h_t: t \in [T]}$, we have
\begin{align}
\sum_{t: p \in \Pcal_t} {f_t^p \indic \cbr{i_t^p = i, m_i^p (t-1) < N} }
= & 
\sum_{t: p \in \Pcal_t} h_t \indic \cbr{i_t^p = i} \nonumber
\\
\leq & 
h_1 + \sum_{t: p \in \Pcal_t} h_{t+1} \indic \cbr{i_t^p = i} \nonumber
\\
= & 
f_1^p 
+ 
\sum_{t: p \in \Pcal_t} {f_{t+1}^p \indic \cbr{i_t^p = i, m_i^p (t) < N}} \nonumber \\
= & 
f_1^p 
+ 
\sum_{t: p \in \Pcal_t} {f_{t+1}^p \indic \cbr{i_t^p = i, n_i (t) < N}}
\label{eqn:rewriting_f_player_p}
\end{align}
where the first inequality is from Equation~\eqref{eqn:auxiliary_sum_h_2} of Lemma~\ref{lem:auxiliary_sum_h}; the second equality is by expanding the definition of $h_t$ and noting that $h_1 = \indic\cbr{ m_i^p(0) < N } f_1^p = \indic\cbr{ 0 < N } f_1^p = f_1^p$; the third equality is from the observation that, if $i_t^p = i$ and $u_i^p(t) = t$, then $m_i^p(t) = n_i( u_i^p(t) ) = n_i(t)$.

Now, summing Equation~\eqref{eqn:rewriting_f_player_p} over all players $p \in [M]$, we have
\begin{align*}
    & \sum_{t \in [T]} \sum_{p \in \Pcal_t} f_t^p
      \indic\cbr{ i_t^p = i, m_i^p(t-1) < N } \\
    \leq & 
    \sum_{p \in [M]} f_1^p +
    \sum_{p \in [M]} \sum_{t: p \in \Pcal_t}  f_{t+1}^p \indic \cbr{i_t^p = i, n_i (t) < N} \\
    \leq & \sum_{p \in [M]} f_1^p +
    \sum_{k=1}^{N-1}  f_{\tau_k+1}^{p_k} \indic \cbr{\tau_k \leq T},
\end{align*}
where the second inequality is from the observation that for every $t \in [T]$, $p \in \Pcal_t$ such that $i_t^p = i$ and $n_i(t) < N$, there must exists some unique $k \in [N-1]$ such that $\tau_k = t$ and $p_k = p$. 
\end{proof}

The following auxiliary lemma facilitates the proofs of Lemmas~\ref{lem:auxiliary_sum_g_indic} and~\ref{lem:auxiliary_sum_f_indic}.
\begin{lemma}
\label{lem:auxiliary_sum_h}
Fix any arm $i \in [K]$ and player $p \in [M]$. Suppose $\cbr{h_t: t \in [T]}$ satisfies the invariant property with respect to $(i,p)$ (Definition~\ref{def:invariant}). 
Then,
\begin{align}
\sum_{t \in [T]: p \in \Pcal_t} h_t
\indic\cbr{ i_t^p = i}
\le & 
h_1 
+
\sum_{k=1}^{T} h_{\pi_k+1} \indic \cbr{\pi_k \le T} 
\label{eqn:auxiliary_sum_h_1}
\\
= &
h_1
+
\sum_{t \in [T]: p \in \Pcal_t} h_{t+1} \indic \cbr{i_t^p = i}
,
\label{eqn:auxiliary_sum_h_2}
\end{align}
where $\pi_k = \pi_k(i,p)$ denotes the round associated with the $k$-th pull of arm $i$ by  player $p$.
\end{lemma}
\begin{proof}
\begin{align*}
\sum_{t \in [T]: p \in \Pcal_t} h_t
\indic\cbr{ i_t^p = i}  
 = &
\sum_{k=1}^T h_{\pi_k} \indic \cbr{\pi_k \le T}  \\
 = &
\sum_{k=1}^T h_{\pi_{k-1}+1} \indic \cbr{\pi_k \le T}  \\
 \leq &
h_1 + \sum_{k=2}^T h_{\pi_{k-1}+1} \indic \cbr{\pi_k \le T}  \\
 = &
h_1 + \sum_{k=1}^{T-1} h_{\pi_{k}+1} \indic \cbr{\pi_{k+1} \le T} \\
\leq & 
h_1 + \sum_{k=1}^{T-1} h_{\pi_{k}+1} \indic \cbr{\pi_k \le T}  \\
 = &  
h_1
+
\sum_{t \in [T]: p \in \Pcal_t} h_{t+1} \indic \cbr{i_t^p = i},
\end{align*}
where the first equality uses the definition of $\pi_k$; the second equality uses the invariant property, specifically, $h_{\pi_k} = h_{\pi_{k-1}+1}$; the first inequality uses the observation that the first term $h_{\pi_0+1} \indic \cbr{\pi_1 \le T} = h_{1} \indic \cbr{\pi_1 \le T} \leq  h_1$; the third equality shifts the indices in the sum by $1$; the second inequality uses the observation that $\pi_{k+1} \leq T \implies \pi_k \leq T$; and the last equality is again by the definition of $\pi_k$.
\end{proof}

The following lemma is largely inspired by~\citet[][Lemma 2.8]{an17}; here we generalize it to the multi-task setting, for reducing bounding $(B)$ and $(E)$ to bounding $(B*)$ and $(E*)$ respectively. 
\begin{lemma}
\label{lem:change-arm}
For any player $p \in [M]$, time step $t \in [T]$, and arm $i \in [K]$, we have for any arm $l \in [K]$ and any threshold $z \in \RR$: 
\[
\Pr \rbr{i_t^p = i, \theta_i^p(t) \leq z \mid \Fcal_{t-1}}
\leq 
\frac{\Pr \rbr{\theta_l^p(t) \leq z \mid \Fcal_{t-1}}}{\Pr \rbr{ \theta_l^p(t) > z \mid \Fcal_{t-1}} } \cdot \Pr \rbr{i_t^p = l \mid \Fcal_{t-1}}.
\]
\end{lemma}
\begin{proof}
First, 
\begin{align*}
    & \Pr \rbr{i_t^p = i, \wbar{Q_i^p(t)} \mid \Fcal_{t-1}} \\
    \le & \Pr \rbr{ \forall j \in [K],\ \theta_j^p(t) \le z \mid \Fcal_{t-1} }  \\
    = & \Pr \rbr{\theta_{l}^p(t) \le z \mid \Fcal_{t-1}} 
    \cdot
    \Pr \rbr{\forall j \neq l,\ \theta_j^p(t) \le z \mid \Fcal_{t-1}},
\end{align*}
where
the first inequality follows because the event $\cbr{i_t^p = i, \wbar{Q_i^p(t)}}$ happens only if $\forall j \in [K],\ \theta_j^p(t) \le z$; 
and the second equality follows because conditional on $\Fcal_{t-1}$, the draws $\theta_j^p(t)$'s and $\theta_{l}^p(t)$ are independent.

Now, observe that
\begin{align*}
    & \Pr \rbr{\forall j \neq l,\ \theta_j^p(t) \le z \mid \Fcal_{t-1}} \\
    = & \frac{\Pr \rbr{\theta_{l}^p(t) > z, \text{ and } \ \forall j \neq l,\ \theta_j^p(t) \le z \mid \Fcal_{t-1} } }{\Pr \rbr{\theta_{l}^p(t) > z \mid \Fcal_{t-1}}} 
    \\
    \leq & \frac{\Pr \rbr{ i_t^p = l  \mid \Fcal_{t-1} }}{\Pr \rbr{\theta_{l}^p(t) > z \mid \Fcal_{t-1}}} 
\end{align*}
where the equality follows, again, by the conditional independence of $\cbr{ \theta_j^p(t): j \neq l }$ and $\theta_{l}^p(t)$; and the inequality follows because the event $\cbr{\theta_{l}^p(t) > z, \ \forall j \neq l,\ \theta_j^p(t) \le y_i^p}$ implies that $\cbr{i_t^p = l}$ happens. The lemma follows from combining the above two inequalities.
\end{proof}

\newpage
\section{Theoretical Guarantees of Baselines}
\label{sec:appendix_baselines}

\subsection{\inducb and \indts in the generalized $\epsilon$-MPMAB setting}

\begin{theorem}
The expected collective regret of $\inducb$ and $\indts$ after $T$ rounds satisfies the following two upper bounds simultaneously:
\begin{align}
    \Reg(T) \le &  \order \rbr{
     \sum_{p \in [M]}
     \sum_{i \in [K]: \Delta_i^p > 0} \frac{\ln T}{\Delta_i^p} }
    \label{eqn:ind-gap-dept}
    \\
    \Reg(T) \le & \tilde{\order} \rbr{  \sqrt{ M K \rounds} },
    \label{eqn:ind-gap-indept}
\end{align}
where we recall that $\rounds = \sum_{t=1}^T \abr{ \Pcal_t }$.
\label{thm:ind-reg}
\end{theorem}
\begin{proof}[Proof sketch]
For Eq.~\eqref{eqn:ind-gap-dept}, we note that both \inducb and \indts guarantees that for every $p \in [M]$,
\[
\Reg^p(T) \leq \order\rbr{ \sum_{i \in [K]: \Delta_i^p > 0} \frac{\ln T}{\Delta_i^p} };
\]
summing over $p$ yields Eq.~\eqref{eqn:ind-gap-dept}. 

For Eq.~\eqref{eqn:ind-gap-indept}, we note that for every $p \in [M]$,
\[
\Reg^p(T) \leq \otil \rbr{ \sqrt{K \abr{\cbr{t: p \in \Pcal_t}} } }. 
\]
Summing over all $p \in [M]$, we have
\[
\Reg(T) 
=
\sum_{p=1}^M \Reg^p(T)
\leq 
\otil \rbr{ \sum_{p=1}^M \sqrt{K \abr{\cbr{t: p \in \Pcal_t}} } }
\leq 
\otil \rbr{ \sqrt{M  K \sum_{p=1}^M  \abr{\cbr{t: p \in \Pcal_t}} } }
= 
\otil \rbr{ \sqrt{M K  \rounds } }.
\qedhere
\]
\end{proof}

\subsection{$\robustagg(\epsilon)$ and its regret analysis in the generalized $\epsilon$-MPMAB setting}
\label{sec:robustagg-ucb}

\citet{wzsrc21} study a special case of  $\epsilon$-MPMAB problem, which can be viewed as $\epsilon$-MPMAB problem defined in Section~\ref{sec:background}, with active sets of players $\Pcal_t \equiv [M]$. In this specialized setting, they propose  $\robustagg(\epsilon)$, a UCB-based algorithm that achieves a gap-dependent and gap-independent regret of 
\begin{equation}
\Ocal \rbr{
\frac{1}{M} \sum_{i \in \Ical_{5\epsilon}} \sum_{p \in [M]: \Delta_i^p>0} \frac{\ln T}{\Delta_i^p} +
\sum_{i \in \Ical^C_{5\epsilon}}
\sum_{p \in [M]: \Delta_i^p > 0}
\frac{\ln T}{\Delta_i^p} + MK},
\label{eqn:robustagg-original-gap-dept}
\end{equation}
and
\begin{equation}
\tilde{\Ocal} \rbr{ \sqrt{ M |\Ical_{5\epsilon}| T } + M \sqrt{ |\Ical_{5\epsilon}^C| T } + MK },
\label{eqn:robustagg-original-gap-indept}
\end{equation}
respectively. In this section, we show that, with a few small modifications, their algorithm and analysis can be used in our (more general) $\epsilon$-MPMAB setting, where the active sets $\Pcal_t$ can change over time.

\begin{algorithm}[H]
\begin{algorithmic}[1]
 \STATE {\bfseries Input:} Dissimilarity parameter $\epsilon \in [0,1]$
 
 \STATE \textbf{Initialization:} Set $n^p_i = 0$ for all $p \in [M]$ and all $i \in [K]$.\\
 
 \FOR{$t = 1, 2 \ldots, T$}
 
    \STATE Receive active set of players $\Pcal_t$
 
    \FOR{$p \in \Pcal_t$}
    \label{line:active-player-only-ucb}
    
        \FOR{$i \in [K]$}
        
            \STATE Let $m^p_i = \sum_{q \in [M]: q \neq p} n^q_i$\;
            \label{line:ucb-start}
            
            \STATE Let $\wbar{n^p_i} = n^p_i \vee 1$ and $\wbar{m^p_i} = m^p_i \vee 1$\; 
                        
            \STATE Let
            $$\zeta_i^p(t) = \frac{1}{\wbar{n^p_i}} \sum_{\substack{s < t:  \\ p \in \Pcal_s, i^p_s = i}} r^p_s,
            \ \eta_i^p(t) =  \frac{1}{\wbar{m^p_i}}
            \sum_{s < t} \sum_{\substack{q \in \Pcal_s: \\ q \neq p, i_s^q = i}} r^q_s, 
            \text{and}\ \kappa_i^p(t, \lambda) = \lambda \zeta_i^p(t) + (1-\lambda) \eta_i^p(t);
            $$
            \label{line:emp-mean}
            
            \STATE Let $F(\wbar{n^p_i}, \wbar{m^p_i}, \lambda, \epsilon ) = 8\sqrt{13\ln T\left[\frac{\lambda^2}{\wbar{n^p_i}} + \frac{(1-\lambda)^2}{\wbar{m^p_i}}\right]} + (1-\lambda)\epsilon$\; 
            \label{line:dev-bound}
            
            \STATE Compute $\lambda^* = \argmin_{\lambda \in [0,1]} F(\wbar{n^p_i},\wbar{m^p_i}, \lambda, \epsilon)$\;
            \label{line:opt-lambda}

            \STATE \label{line:ucb-end} Compute an upper confidence bound of the reward of arm $i$ for player $p$:
            $$\UCB^p_i(t) = \kappa_i^p(t, \lambda^*) + F(\wbar{n^p_i}, \wbar{m^p_i}, \lambda^*, \epsilon).$$ 
        \ENDFOR
        \STATE Let $i^p_t = \text{argmax}_{i \in [K]} \UCB^p_i(t)$\;
        \STATE Player $p$ pulls arm $i^p_t$ and observes reward $r^p_t$\;
   \ENDFOR
   \FOR{active players $p \in \Pcal_t$}
   \label{line:active-player-only-count}
       \STATE Let $i = i^p_t$ and set $n^p_{i} \gets n^p_{i} + 1$.
   \ENDFOR
 \ENDFOR
 \end{algorithmic}
 \caption{$\robustagg(\epsilon)$ for the generalized $\epsilon$-MPMAB setting}
 \label{alg:robustagg}
\end{algorithm}

Specifically, Algorithm~\ref{alg:robustagg} is our modified version of $\robustagg(\epsilon)$. 
Recall that $\robustagg(\epsilon)$ performs an UCB-based exploration~\cite{acf02}: for every player and every arm, it constructs high-probability UCBs on the expected rewards (line~\ref{line:ucb-start} to~\ref{line:ucb-end}); to this end, it makes careful use of both the player and other players' data, and construct a series of UCBs parameterized by $\lambda$ (line~\ref{line:dev-bound}), and selects the tightest one (line~\ref{line:opt-lambda} and~\ref{line:ucb-end}). 
Compared to $\robustagg(\epsilon)$, for every round $t$, Algorithm~\ref{alg:robustagg} only computes expected reward UCBs for active players $p \in \Pcal_t$ (line~\ref{line:active-player-only-ucb}), and updates arm pull counts on active players (line~\ref{line:active-player-only-count}). 

We show that Algorithm~\ref{alg:robustagg}, when applied to our $\epsilon$-MPMAB setting, has regret guarantees that recover and generalize $\robustagg(\epsilon)$'s original guarantees. Specifically, in the specialized $\epsilon$-MPMAB setting where $\Pcal_t \equiv [M]$, we recover the regret guarantees of $\robustagg(\epsilon)$ (Equations~\eqref{eqn:robustagg-original-gap-dept} and~\eqref{eqn:robustagg-original-gap-indept}). 

\begin{theorem}
The expected collective regret of $\robustagg(\epsilon)$ after $T$ rounds satisfies the following two upper bounds simultaneously:
\begin{align}
    \Reg(T) \le &  \Ocal \vast( \frac 1 M \sum_{i \in \Ical_{5\epsilon}} \sum_{p \in [M]: \Delta_i^p>0} \frac{\ln T}{\Delta_i^p}  + \sum_{i \in \Ical_{5\epsilon}^C} \sum_{\substack{p \in [M]: \Delta_i^p>0}} \frac{\ln T}{\Delta_i^p} +  M K \vast),
    \label{eqn:robustagg-gap-dept}
    \\
    \Reg(T) \le & \tilde{\Ocal} \vast(  \sqrt{ |\Ical_{5\epsilon}| \rounds } + \sqrt{ M \rbr{|\Ical_{5\epsilon}^C| -1 } \rounds }  + MK \vast),
    \label{eqn:robustagg-gap-indept}
\end{align}
where we recall that $\rounds = \sum_{t=1}^T \abr{ \Pcal_t }$.
\label{thm:robustagg-reg}
\end{theorem}

\begin{proof}[Proof sketch]

Even in the general setting where $\Pcal_t$ is not necessarily $[M]$, Freedman's inquality can still be applied to establish the high-probability concentration of the empirically averaged rewards $\zeta_i^p(t)$ and $\eta_i^p(t)$; therefore, Lemma 17 of~\citet{wzsrc21} still holds in the general setting. As a result, Lemmas 20 and 21 of~\citet{wzsrc21} carries over; hence, for all $i \in \Ical_{5\epsilon}$, Algorithm~\ref{alg:robustagg} still satisfies that
\begin{equation}
\mathbb{E} [n_i(T)] \le \order \left(\frac{\ln T}{(\Delta_i^{\min})^2}  + M\right),
\label{eqn:robustagg-i-eps}
\end{equation}
and for all $i \in \Ical_{5\epsilon}^C$ and all $p \in [M]$, 
\begin{equation}
\mathbb{E} [n^p_i(T)] \le \order \left( \frac{\ln T}{(\Delta^p_i)^2} \right).
\label{eqn:robustagg-i-eps-c}
\end{equation}
Equations~\eqref{eqn:robustagg-gap-dept} and~\eqref{eqn:robustagg-gap-indept} now follows directly from applying Lemma~\ref{lem:arm-pull-reg} with $C=0$ and $\alpha=5\epsilon$. 
\end{proof}

\section{Additional Experimental Results}
\label{sec:appendix_exp}

In this section, we present the rest of the experimental results. Figures~\ref{figure:cumulative_allplots}, \ref{figure:perc_allplots}, and \ref{figure:regret_perc_allplots} compare the average performance of $\robustaggTS(0.15)$, $\robustagg(0.15)$, \inducb, and \indts in randomly generated $0.15$-MPMAB problem instances with different numbers of subpar arms.

Note that, when $\abr{\Ical_{5\epsilon}} = 9$, we have $\abr{\Ical_{5\epsilon}^C} = 1$ which means that there exists one arm that is optimal to all the players and the other arms are all subpar. In this favorable special case, $\robustaggTS(0.15)$ and $\robustagg(0.15)$ perform significantly better than the baseline algorithms without transfer, as expected.

Furthermore, when $\abr{\Ical_{5\epsilon}} = 0$, i.e., there is no subpar arm and all the arms have relatively small suboptimality gaps. In this unfavorable special case, $\robustaggTS(0.15)$'s performance is still very competitive in comparison with $\indts$, which demonstrates the robustness of our proposed algorithm.

\subsection{Empirical Comparison with $\robustaggTSV(\epsilon)$}
\label{sec:comp_v}
We empirically evaluated a variant of Algorithm~\ref{alg:robustaggTS}, which we refer to as $\robustaggTSV(\epsilon)$. $\robustaggTSV(\epsilon)$ differs from $\robustaggTS(\epsilon)$ (Algorithm~\ref{alg:robustaggTS}) in one way: in each round, instead of only updating the posteriors associated with each active player and its pulled arm (i.e., delayed update, line~\ref{line:invariant} of Algorithm~\ref{alg:robustaggTS}), $\robustaggTSV(\epsilon)$ updates the posteriors associated with every arm and player. Note that this change only affects the aggregate posteriors, as the individual posteriors associated with a player and an arm remains the same if the player does not pull the arm in this round. 

Figure~\ref{figure:v_allplots} compares the average cumulative regret of $\robustaggTS(0.15)$, $\robustaggTSV(0.15)$, $\robustagg(0.15)$, \inducb, and \indts in randomly generated $0.15$-MPMAB problem instances with different numbers of subpar arms. The instances were generated following the same procedure as the other experiments. Observe that $\robustaggTSV(0.15)$'s empirical performance is on par with that of $\robustaggTS(0.15)$. However, our analysis in this paper takes advantages of the design choice made for $\robustaggTS(\epsilon)$, i.e., delayed update which leads to the invariant property (Definition~\ref{def:invariant} and Examples~\ref{ex:h-invariant}, \ref{ex:m-n-invariant} and \ref{ex:mu-var-invariant}). It is unclear whether $\robustaggTSV(\epsilon)$ enjoys similar near-optimal guarantees.

\begin{figure}[htbp]
    \centering
    \begin{subfigure}{0.32\textwidth}
        \centering
        \includegraphics[height=0.85\linewidth]{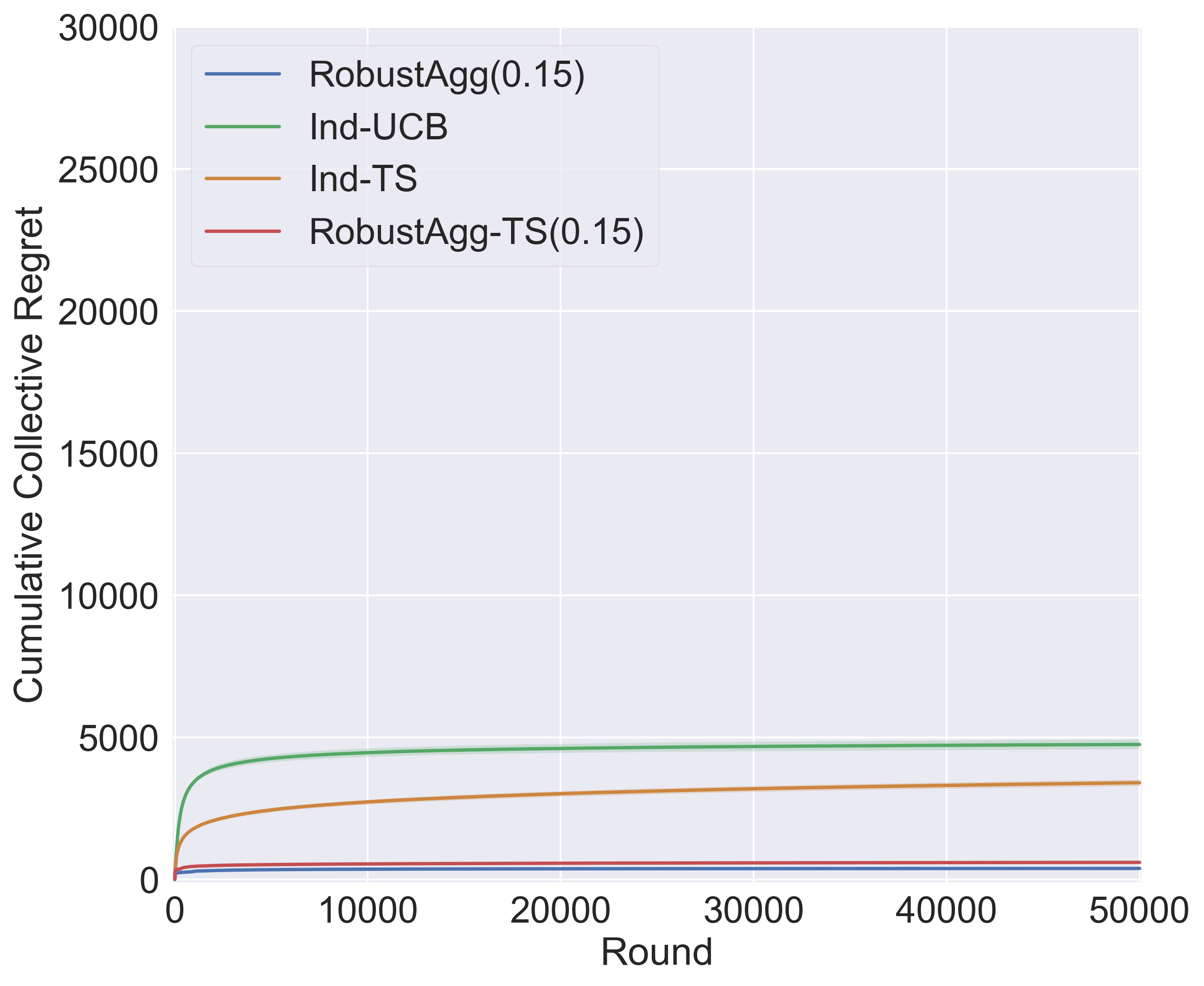}
        \caption{$|\Ical_{5\epsilon}| = 9$}
        \label{figure:exp1_9}
    \end{subfigure}
    \begin{subfigure}{0.32\textwidth}
        \centering
        \includegraphics[height=0.85\linewidth]{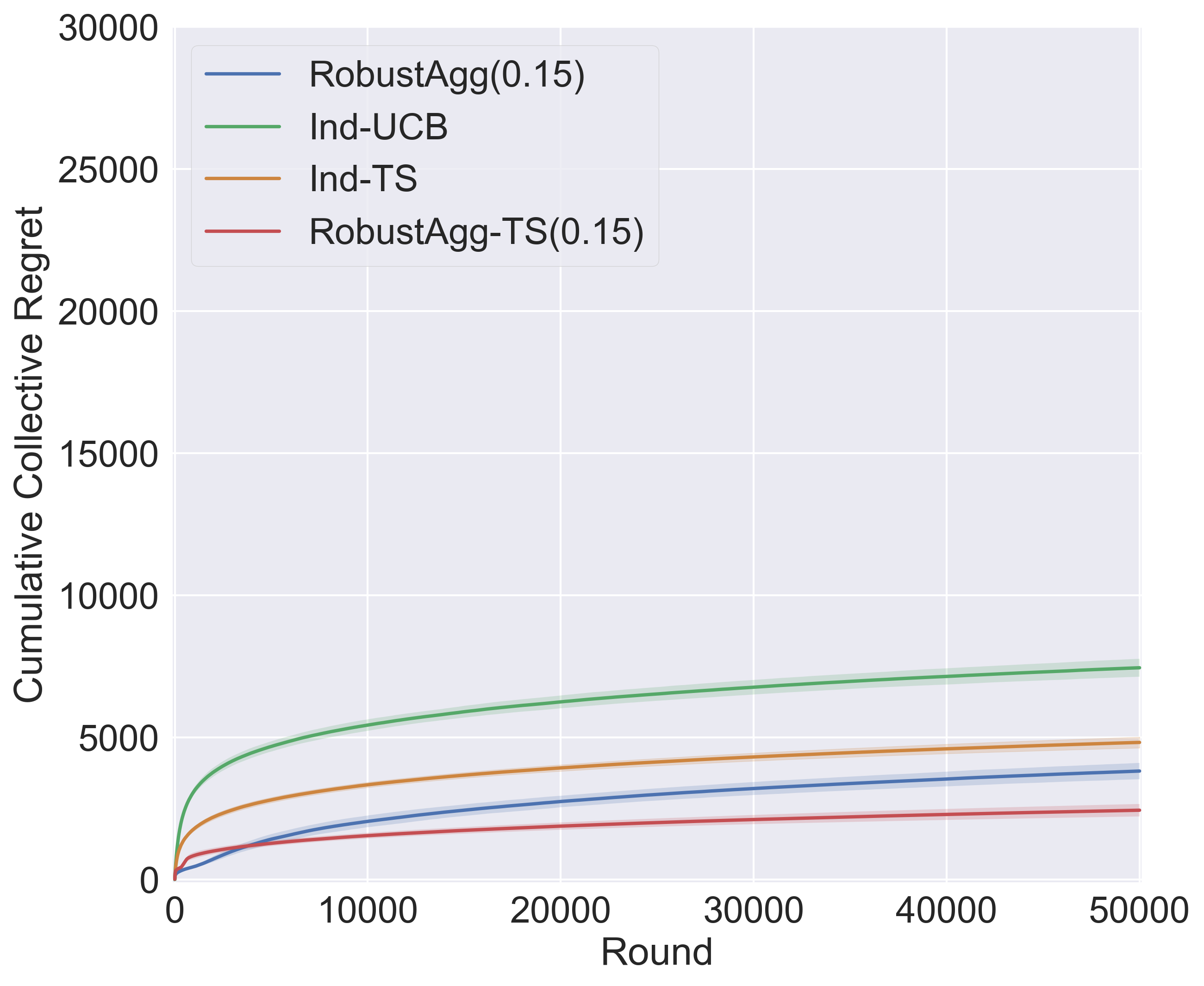}
        \caption{$|\Ical_{5\epsilon}| = 8$}
        \label{figure:exp1_8}
    \end{subfigure}
    \begin{subfigure}{0.32\textwidth}
        \centering
        \includegraphics[height=0.85\linewidth]{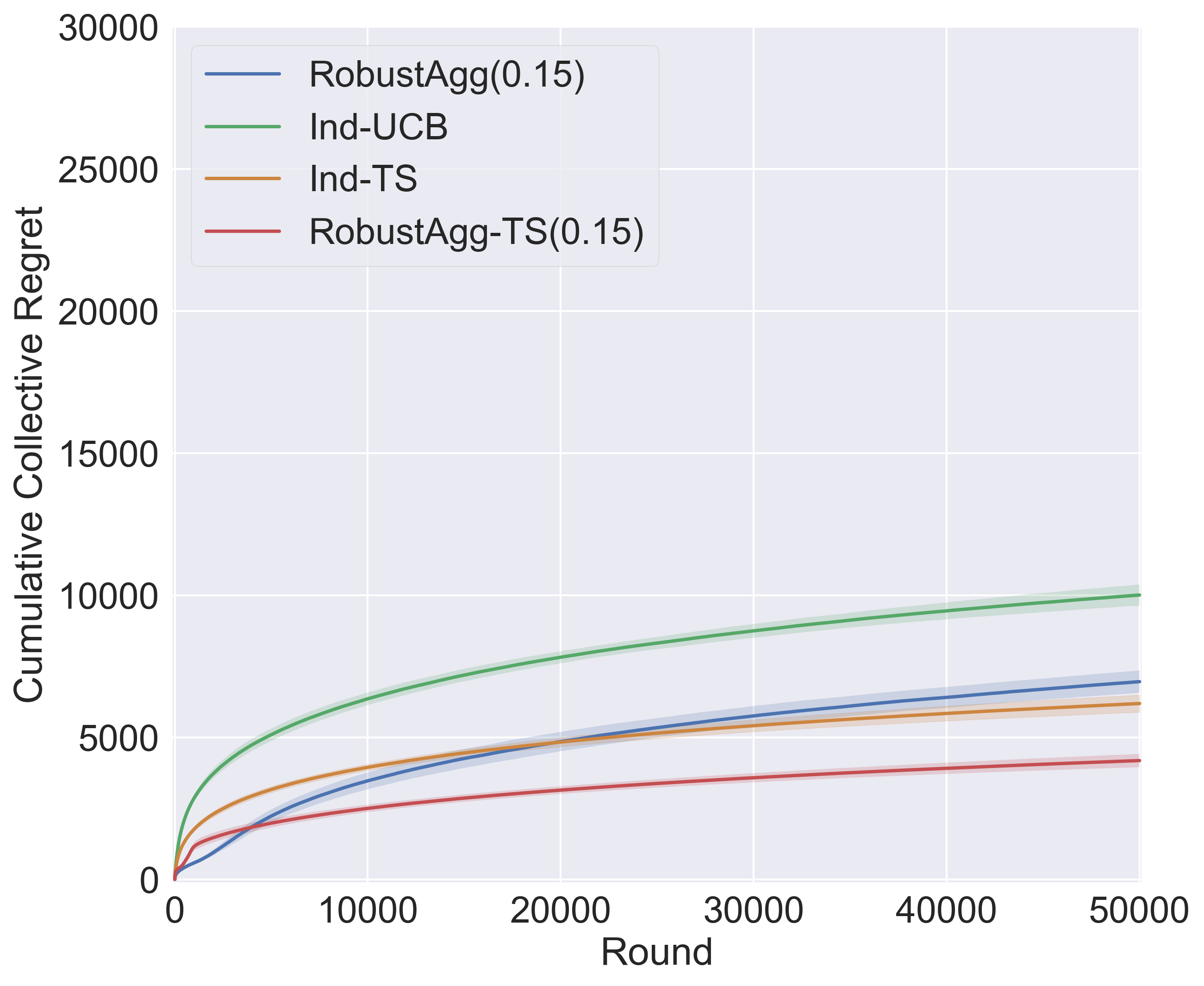}
        \caption{$|\Ical_{5\epsilon}| = 7$}
        \label{figure:exp1_7}
    \end{subfigure}
        \begin{subfigure}{.32\textwidth}
        \centering
        \includegraphics[height=0.85\linewidth]{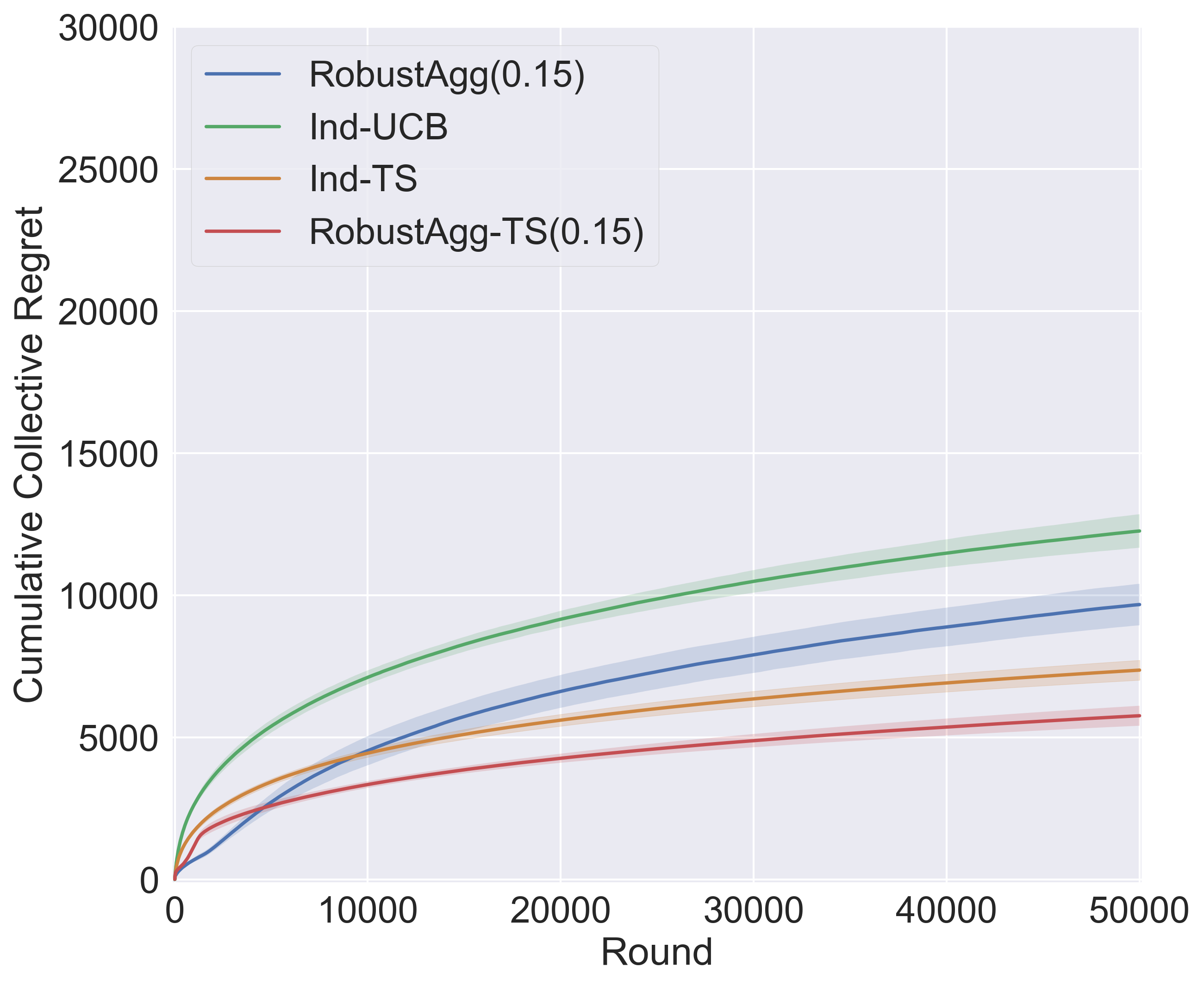}
        \caption{$|\Ical_{5\epsilon}| = 6$}
        \label{figure:exp1_6}
    \end{subfigure}
    \begin{subfigure}{0.32\textwidth}
        \centering
        \includegraphics[height=0.85\linewidth]{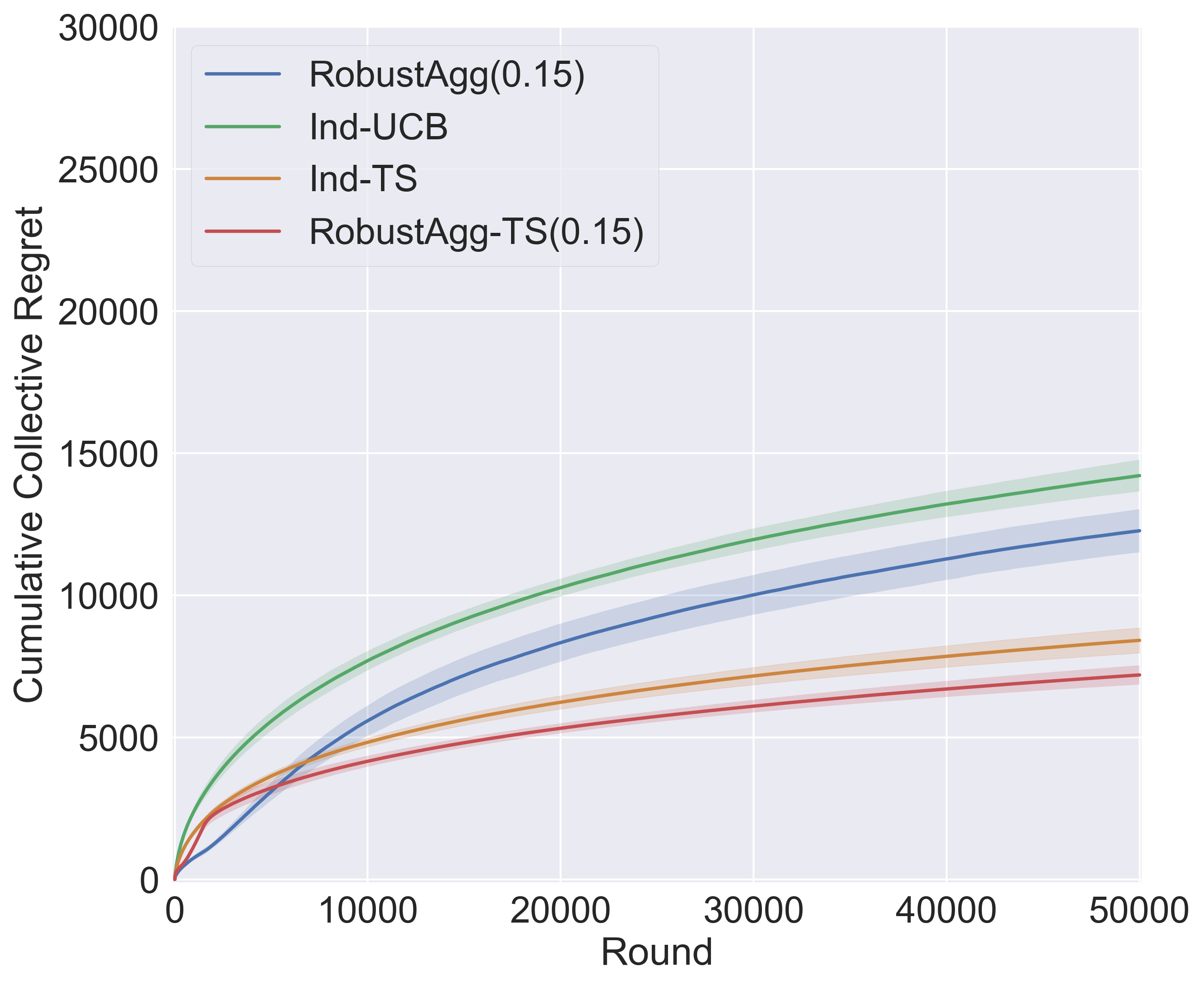}
        \caption{$|\Ical_{5\epsilon}| = 5$}
        \label{figure:exp1_5}
    \end{subfigure}
    \begin{subfigure}{0.32\textwidth}
        \centering
        \includegraphics[height=0.85\linewidth]{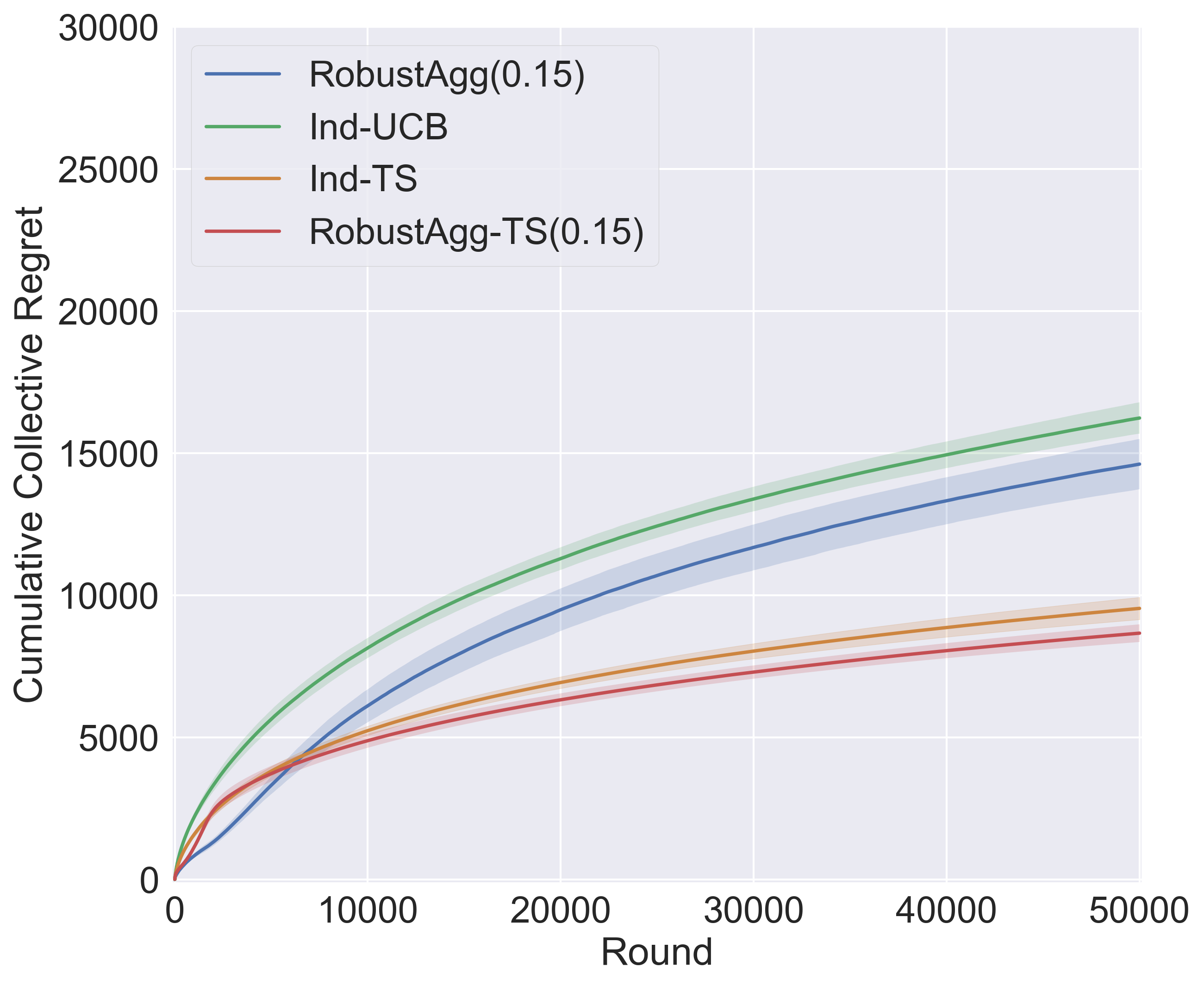}
        \caption{$|\Ical_{5\epsilon}| = 4$}
        \label{figure:exp1_4}
    \end{subfigure}
        \begin{subfigure}{.32\textwidth}
        \centering
        \includegraphics[height=0.85\linewidth]{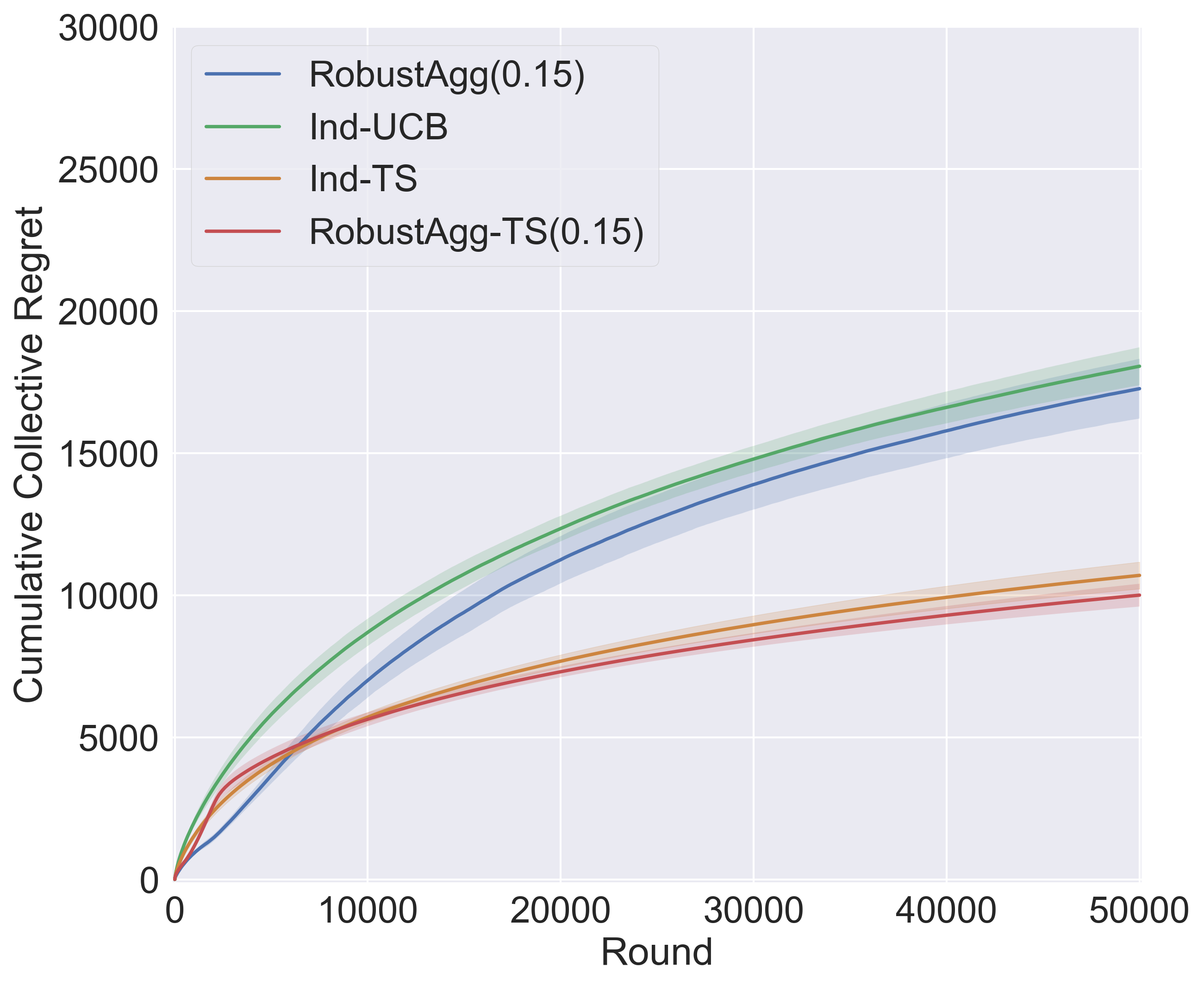}
        \caption{$|\Ical_{5\epsilon}| = 3$}
        \label{figure:exp1_3}
    \end{subfigure}
    \begin{subfigure}{0.32\textwidth}
        \centering
        \includegraphics[height=0.85\linewidth]{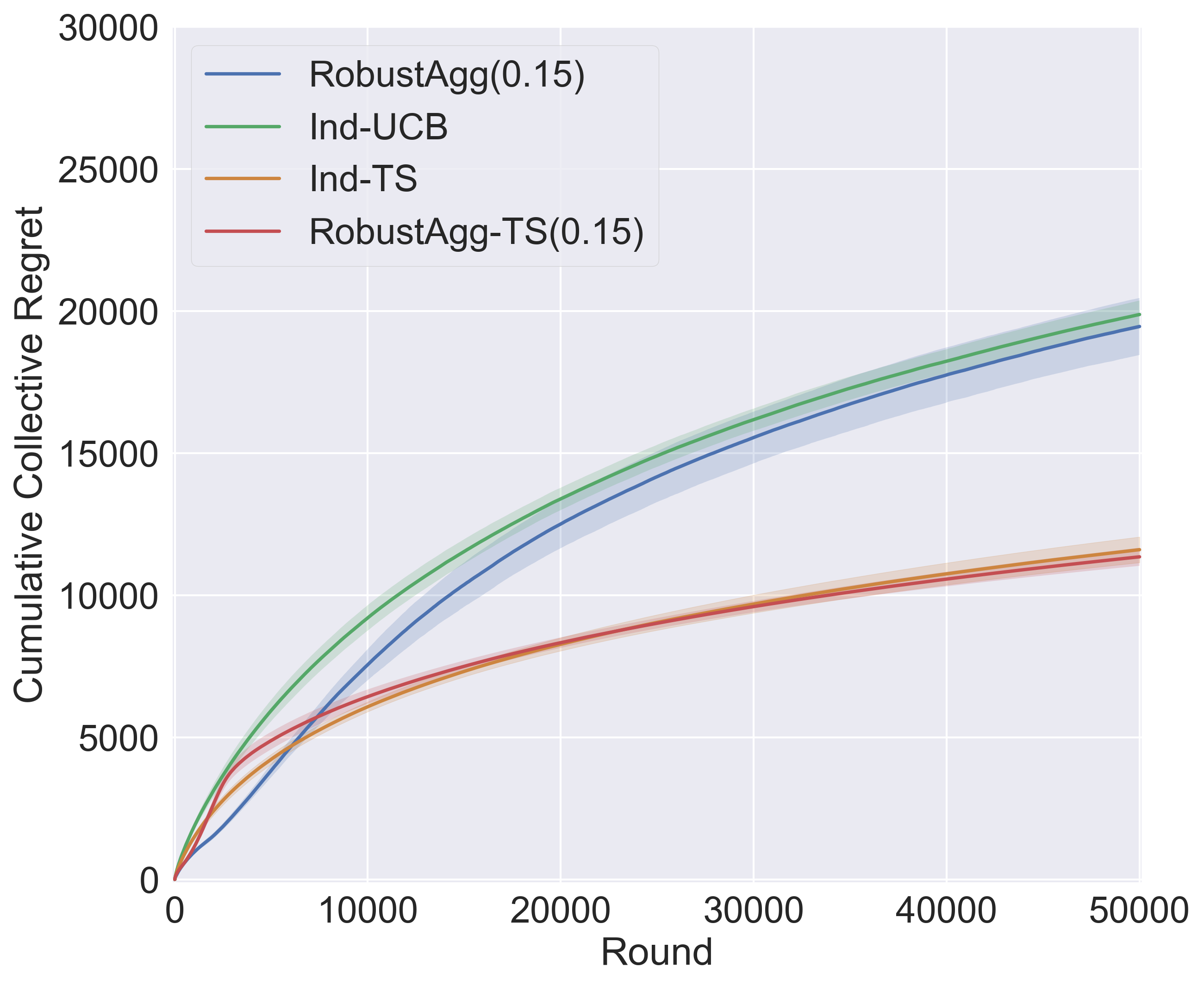}
        \caption{$|\Ical_{5\epsilon}| = 2$}
        \label{figure:exp1_2}
    \end{subfigure}
    \begin{subfigure}{0.32\textwidth}
        \centering
        \includegraphics[height=0.85\linewidth]{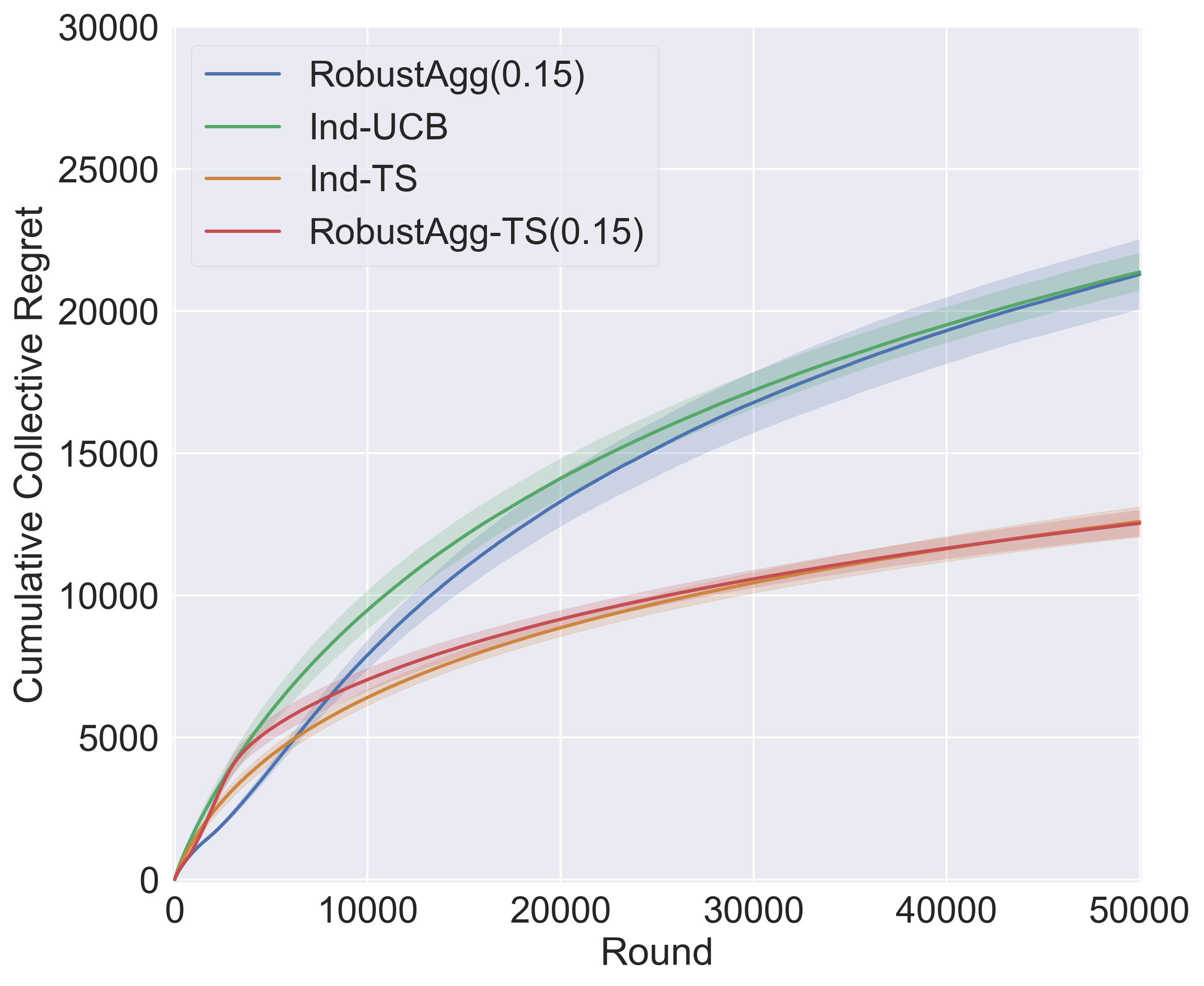}
        \caption{$|\Ical_{5\epsilon}| = 1$}
        \label{figure:exp1_1}
    \end{subfigure}
    \begin{subfigure}{0.32\textwidth}
        \centering
        \includegraphics[height=0.85\linewidth]{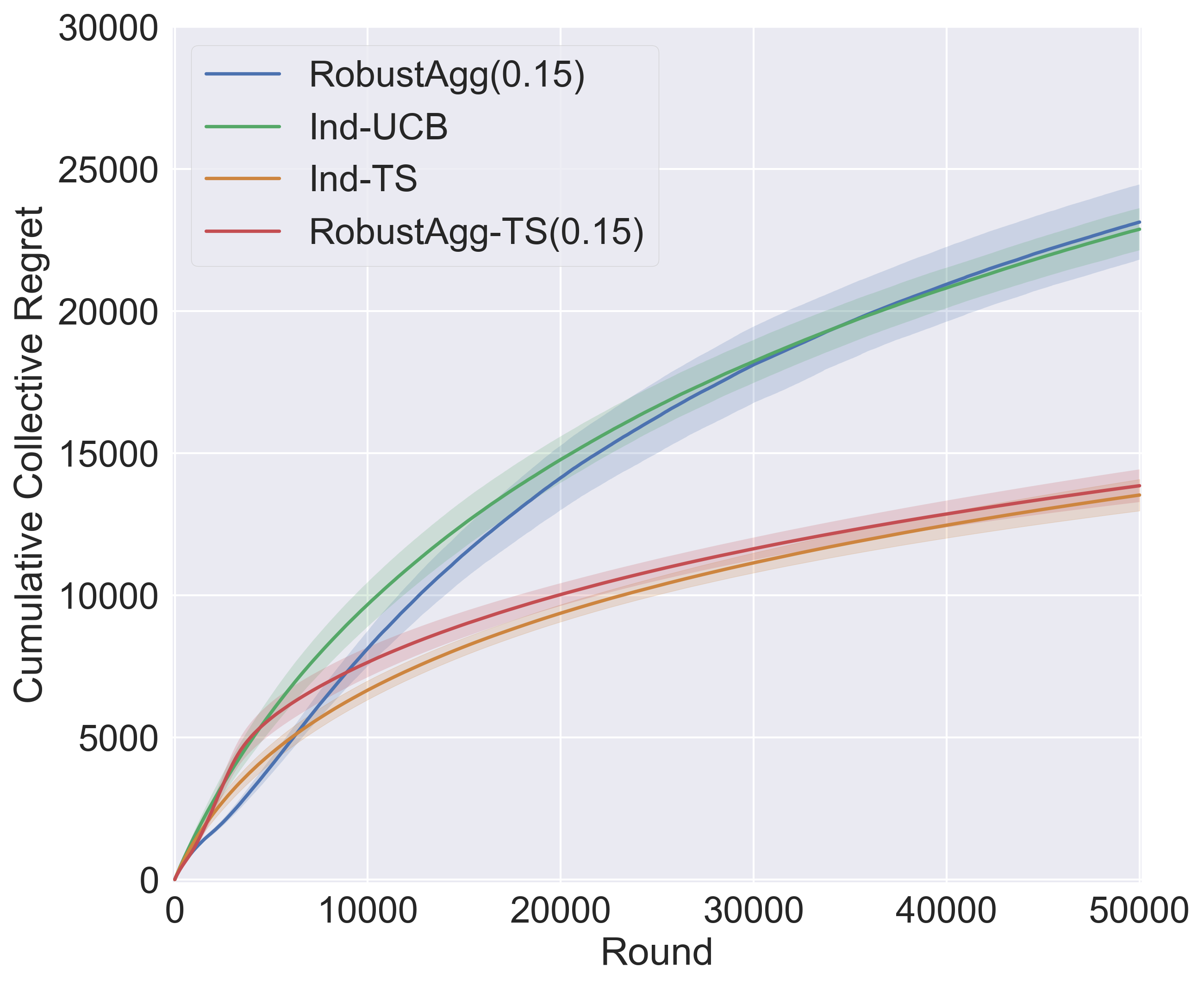}
        \caption{$|\Ical_{5\epsilon}| = 0$}
        \label{figure:exp1_0}
    \end{subfigure}
    \caption{Compares the cumulative collective regret of the 4 algorithms over a horizon of $T = 50,000$ rounds.}
    \label{figure:cumulative_allplots}
\end{figure}

\begin{figure}[htbp]
    \centering
    \begin{subfigure}{0.32\textwidth}
        \centering
        \includegraphics[height=0.85\linewidth]{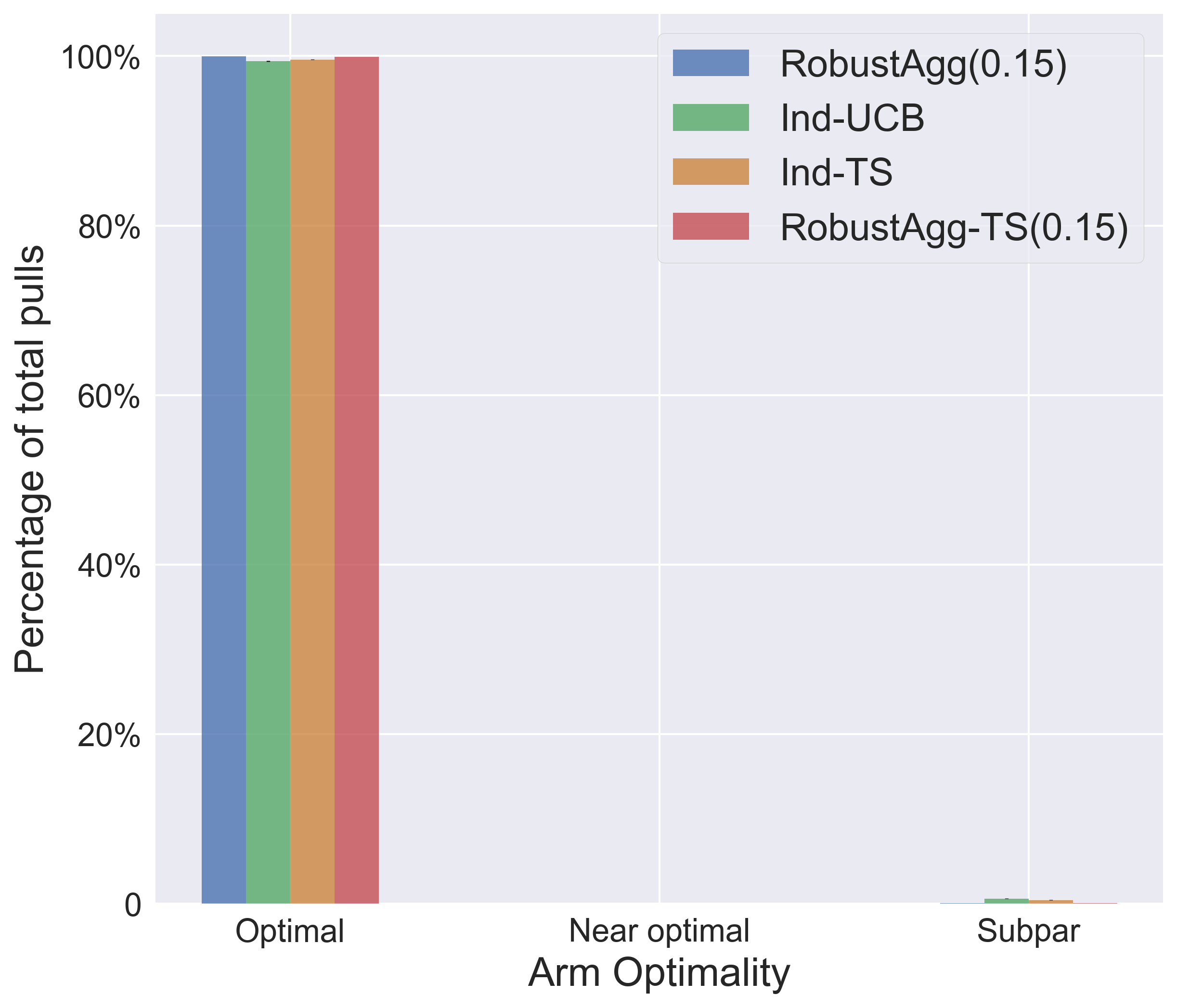}
        \caption{$|\Ical_{5\epsilon}| = 9$}
        \label{figure:exp2_9}
    \end{subfigure}
    \begin{subfigure}{0.32\textwidth}
        \centering
        \includegraphics[height=0.85\linewidth]{plots/exp1_perc_ical=8_50000x30.png}
        \caption{$|\Ical_{5\epsilon}| = 8$}
        \label{figure:exp2_8}
    \end{subfigure}
    \begin{subfigure}{0.32\textwidth}
        \centering
        \includegraphics[height=0.85\linewidth]{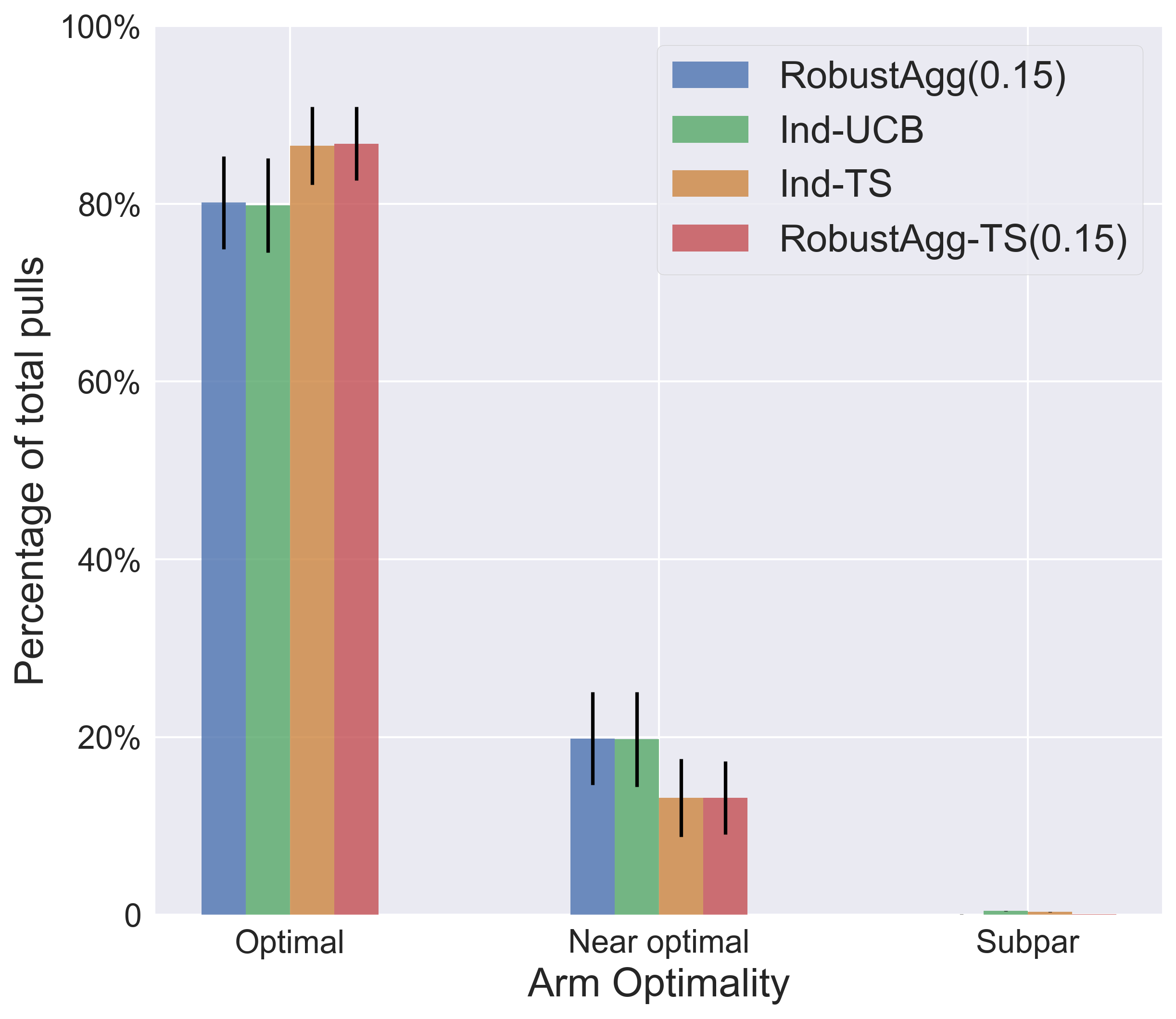}
        \caption{$|\Ical_{5\epsilon}| = 7$}
        \label{figure:exp2_7}
    \end{subfigure}
        \begin{subfigure}{.32\textwidth}
        \centering
        \includegraphics[height=0.85\linewidth]{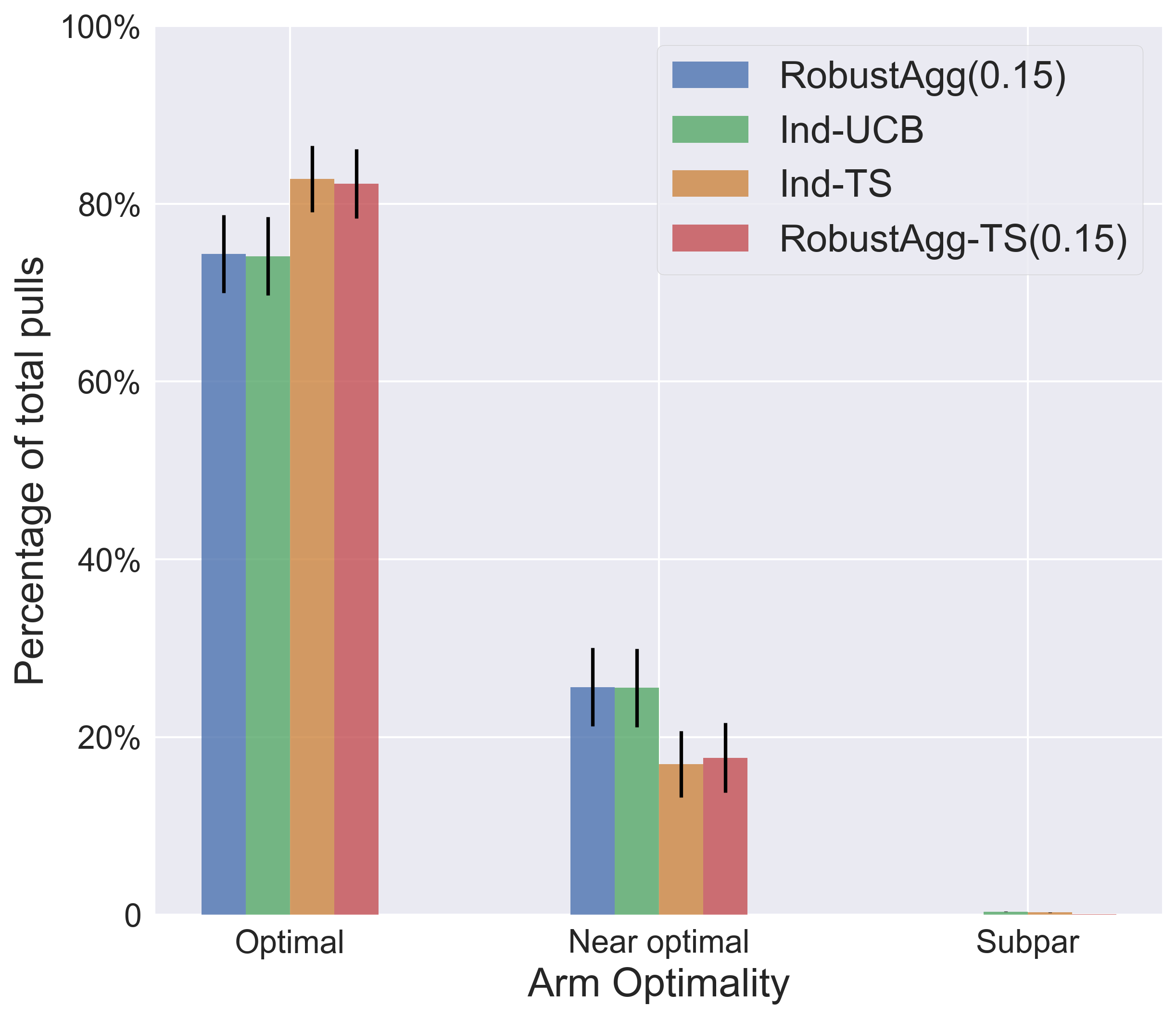}
        \caption{$|\Ical_{5\epsilon}| = 6$}
        \label{figure:exp2_6}
    \end{subfigure}
    \begin{subfigure}{0.32\textwidth}
        \centering
        \includegraphics[height=0.85\linewidth]{plots/exp1_perc_ical=5_50000x30.png}
        \caption{$|\Ical_{5\epsilon}| = 5$}
        \label{figure:exp2_5}
    \end{subfigure}
    \begin{subfigure}{0.32\textwidth}
        \centering
        \includegraphics[height=0.85\linewidth]{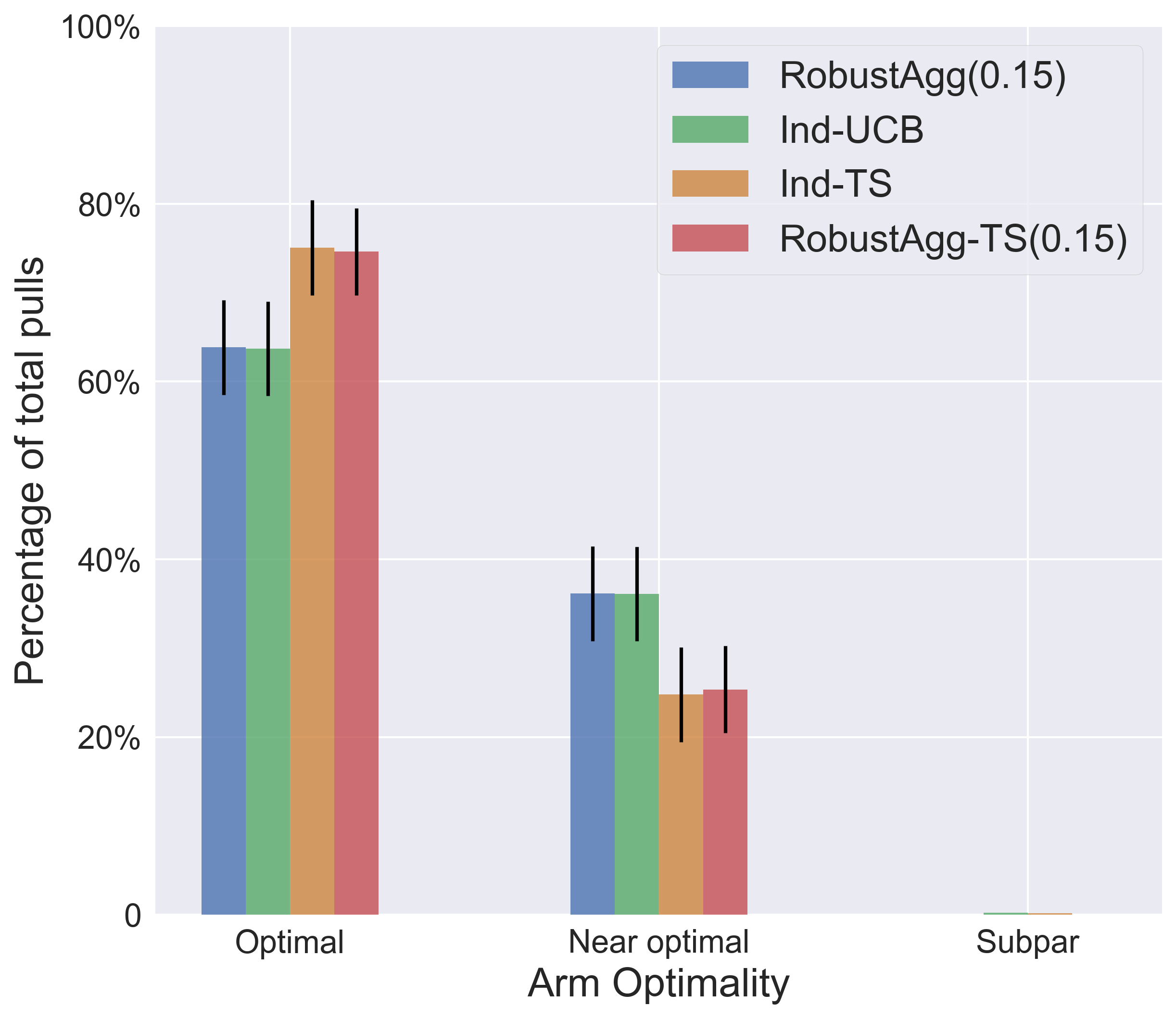}
        \caption{$|\Ical_{5\epsilon}| = 4$}
        \label{figure:exp2_4}
    \end{subfigure}
        \begin{subfigure}{.32\textwidth}
        \centering
        \includegraphics[height=0.85\linewidth]{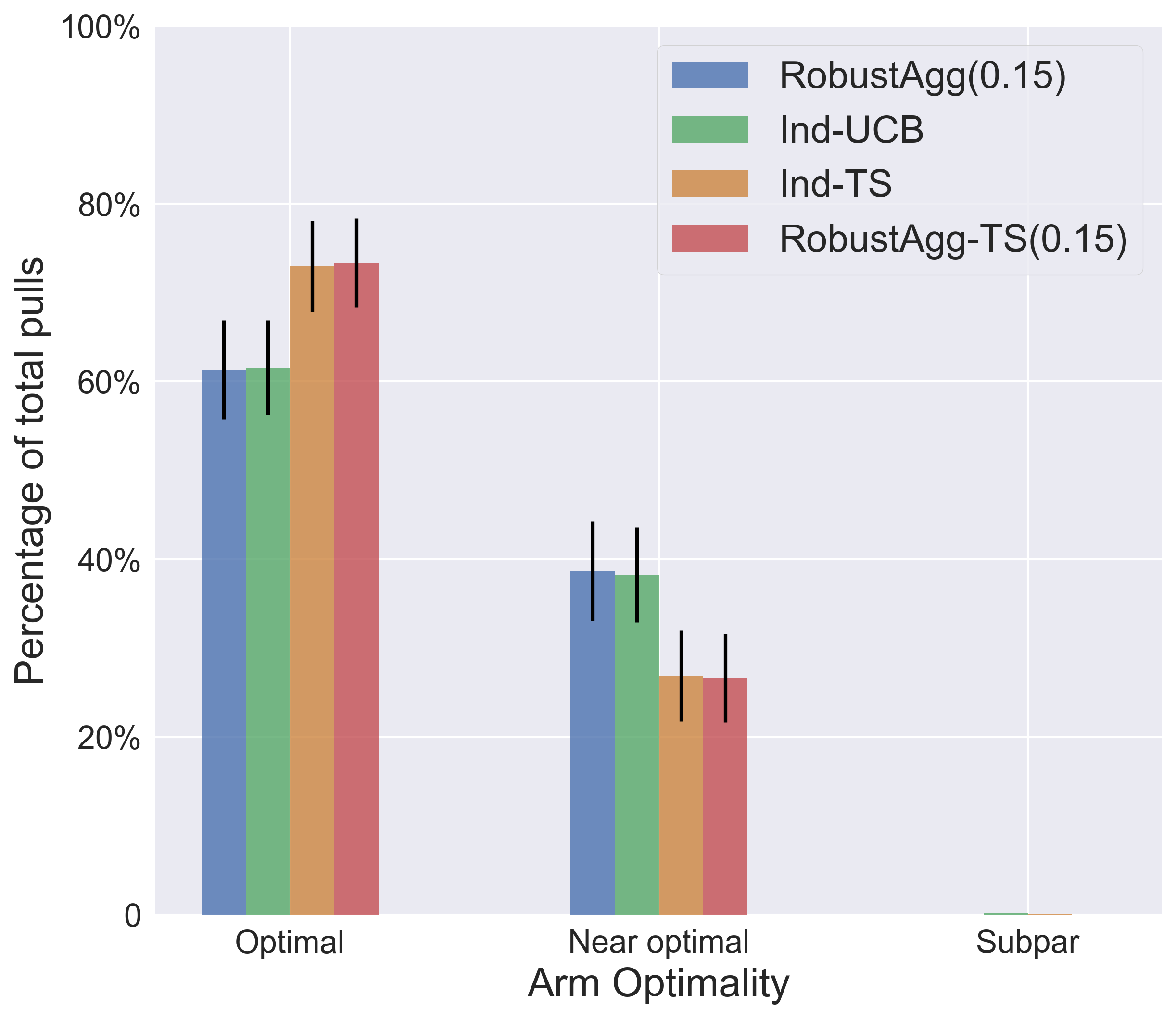}
        \caption{$|\Ical_{5\epsilon}| = 3$}
        \label{figure:exp2_3}
    \end{subfigure}
    \begin{subfigure}{0.32\textwidth}
        \centering
        \includegraphics[height=0.85\linewidth]{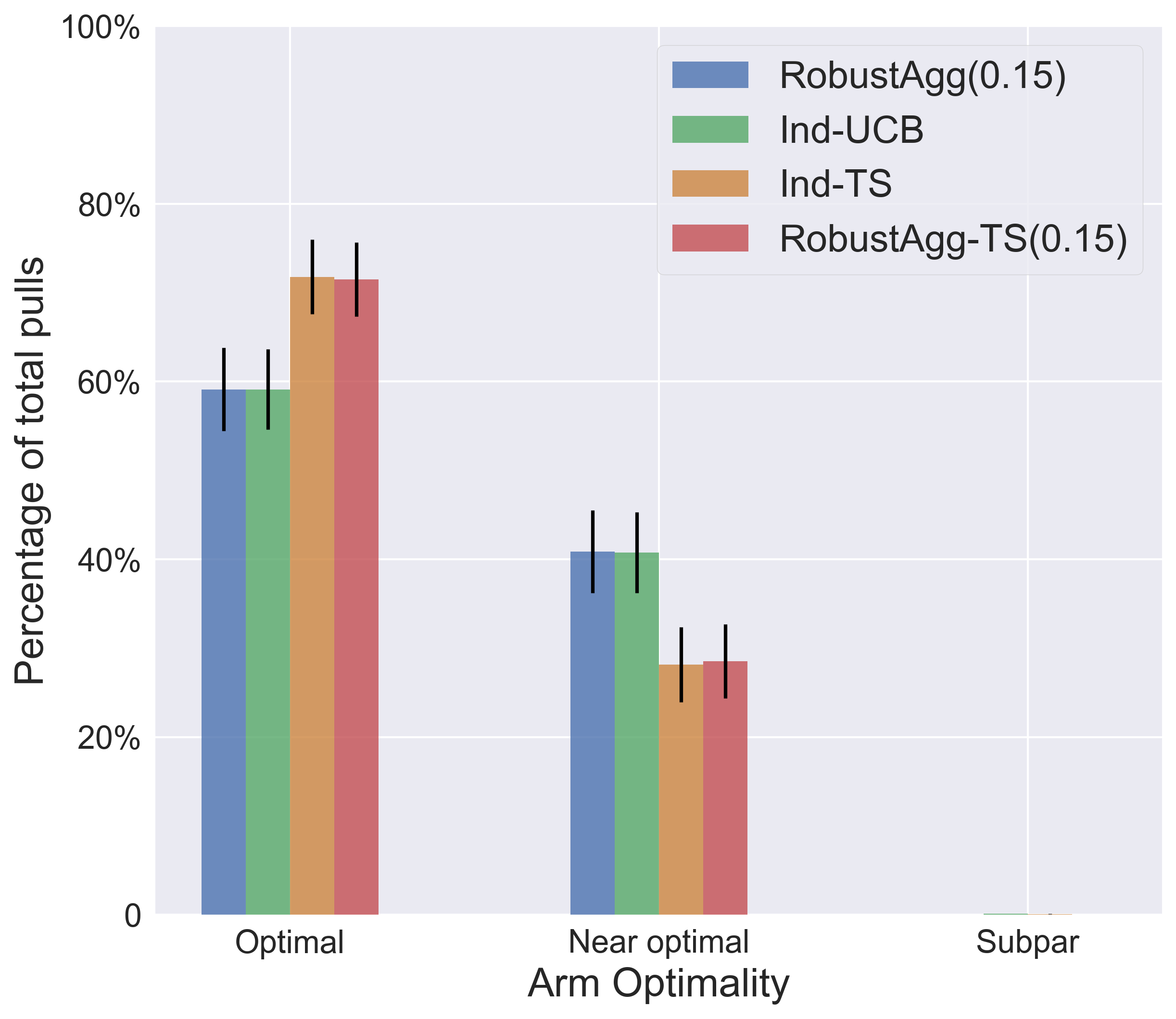}
        \caption{$|\Ical_{5\epsilon}| = 2$}
        \label{figure:exp2_2}
    \end{subfigure}
    \begin{subfigure}{0.32\textwidth}
        \centering
        \includegraphics[height=0.85\linewidth]{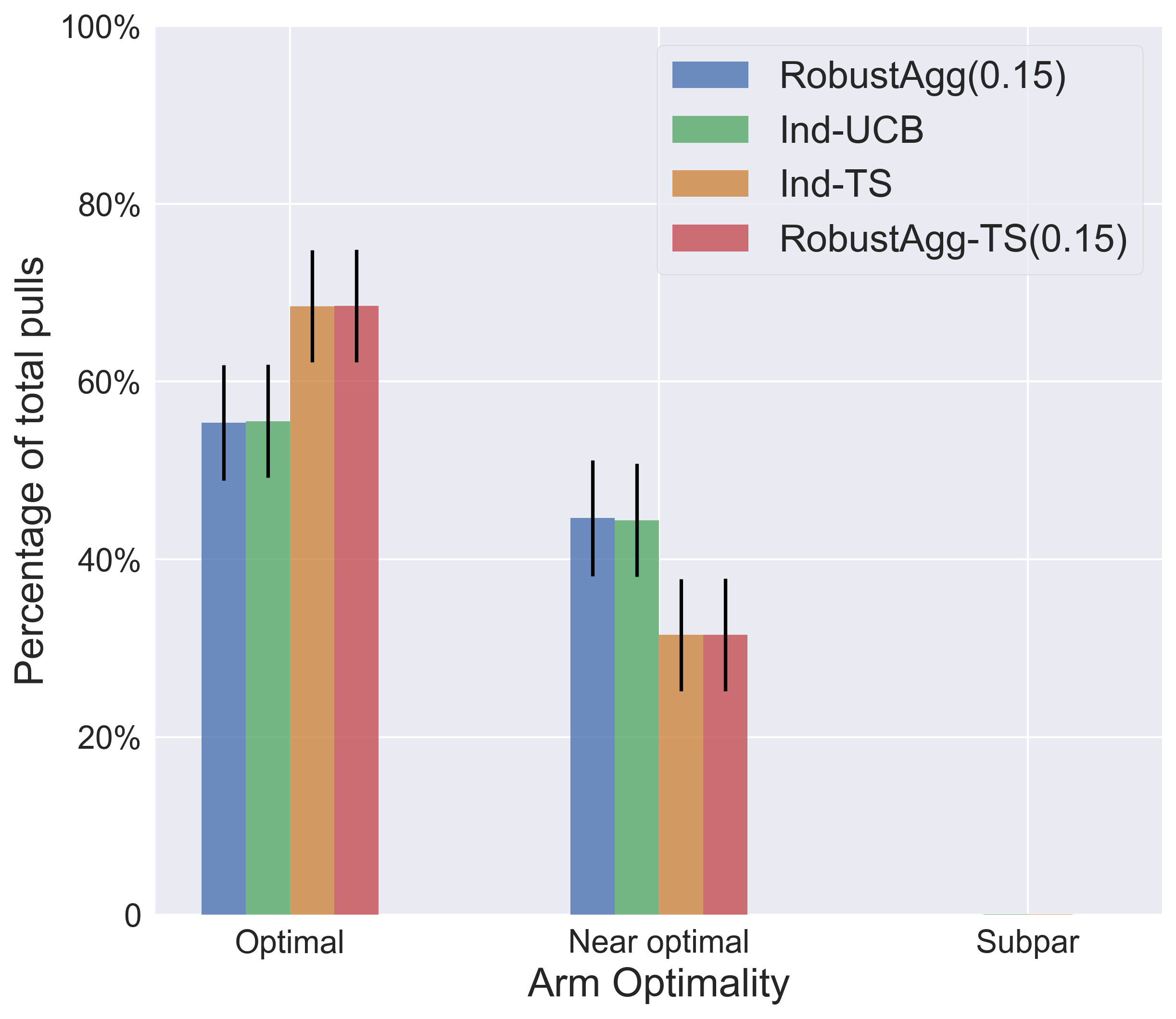}
        \caption{$|\Ical_{5\epsilon}| = 1$}
        \label{figure:exp2_1}
    \end{subfigure}
    \begin{subfigure}{0.32\textwidth}
        \centering
        \includegraphics[height=0.85\linewidth]{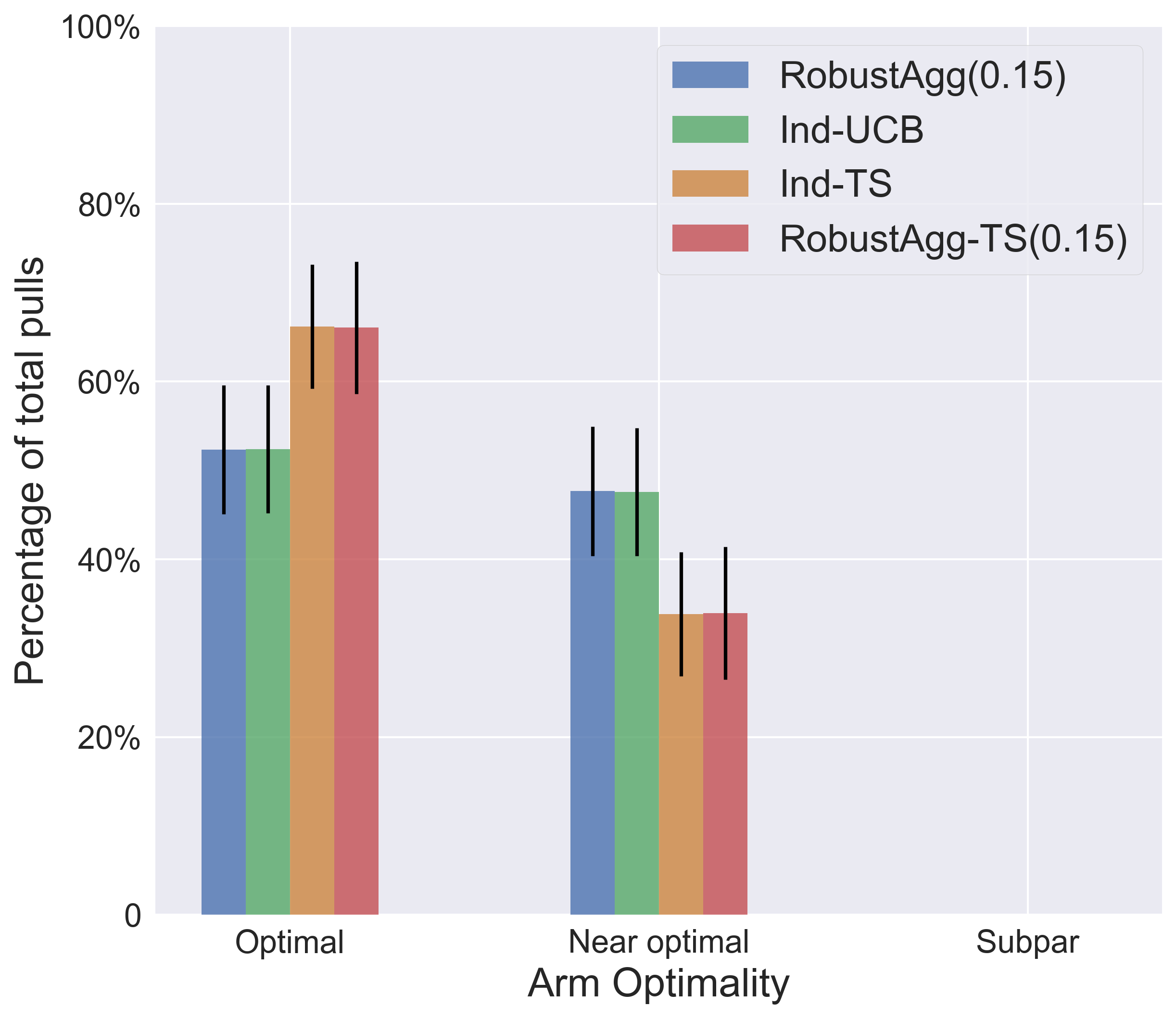}
        \caption{$|\Ical_{5\epsilon}| = 0$}
        \label{figure:exp2_0}
    \end{subfigure}
    \caption{Compares the percentage of arm pulls by arm optimality for the 4 algorithms in $T = 50,000$ rounds.}
    \label{figure:perc_allplots}
\end{figure}

\begin{figure}[htbp]
    \centering
    \begin{subfigure}{0.32\textwidth}
        \centering
        \includegraphics[height=0.85\linewidth]{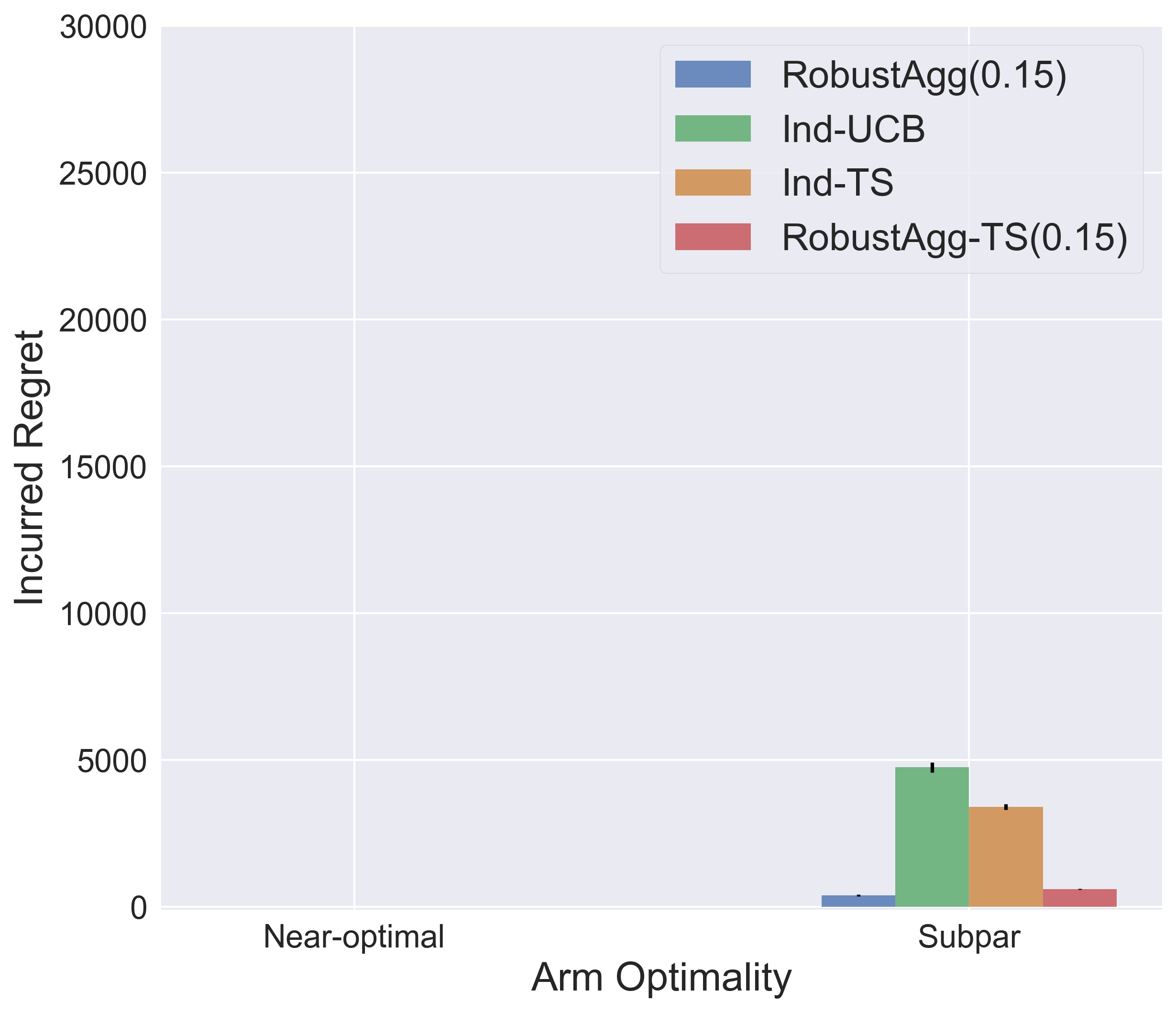}
        \caption{$|\Ical_{5\epsilon}| = 9$}
        \label{figure:exp3_9}
    \end{subfigure}
    \begin{subfigure}{0.32\textwidth}
        \centering
        \includegraphics[height=0.85\linewidth]{plots/exp1__regret_perc_ical=8_50000x30.png}
        \caption{$|\Ical_{5\epsilon}| = 8$}
        \label{figure:exp3_8}
    \end{subfigure}
    \begin{subfigure}{0.32\textwidth}
        \centering
        \includegraphics[height=0.85\linewidth]{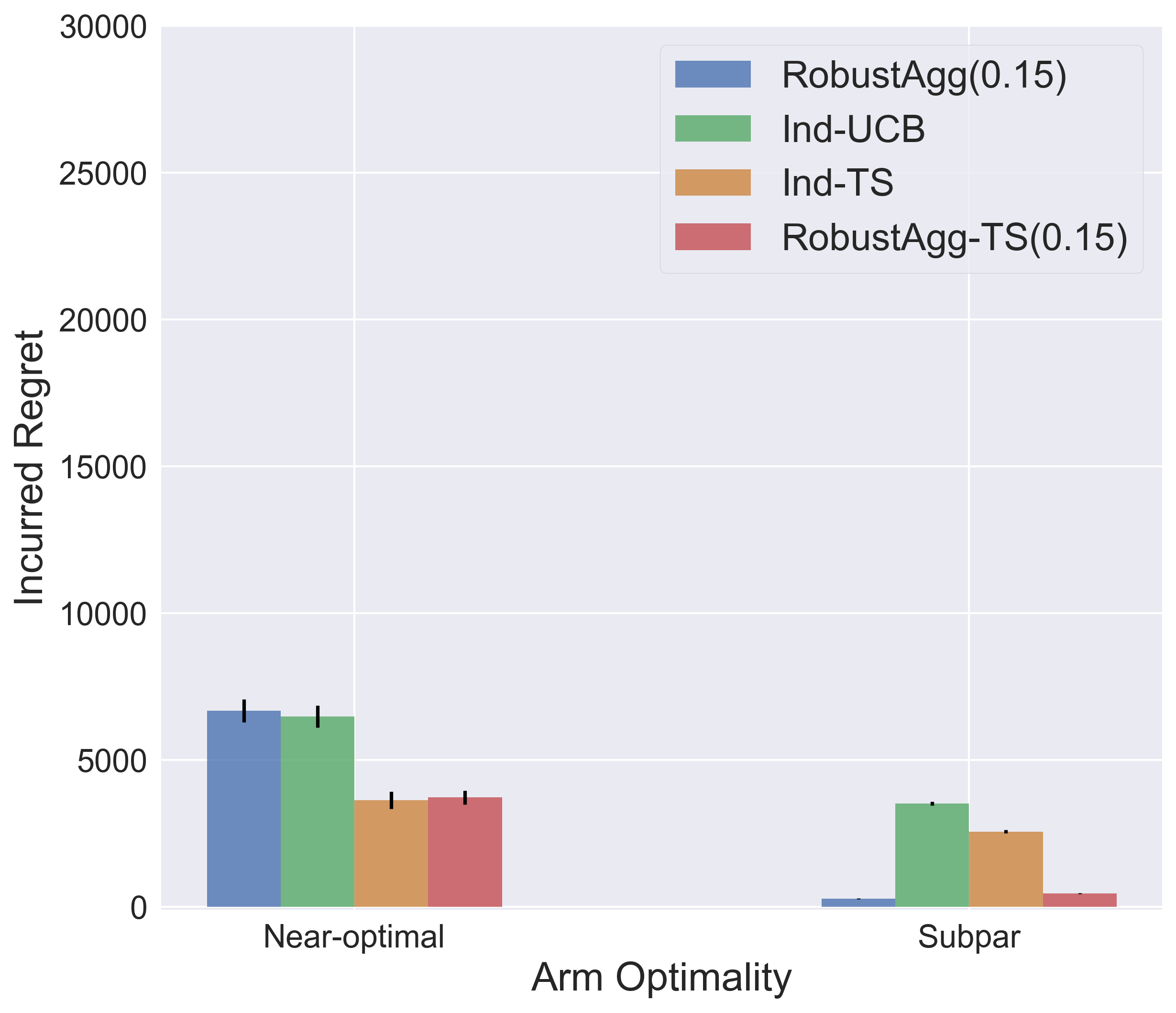}
        \caption{$|\Ical_{5\epsilon}| = 7$}
        \label{figure:exp3_7}
    \end{subfigure}
        \begin{subfigure}{.32\textwidth}
        \centering
        \includegraphics[height=0.85\linewidth]{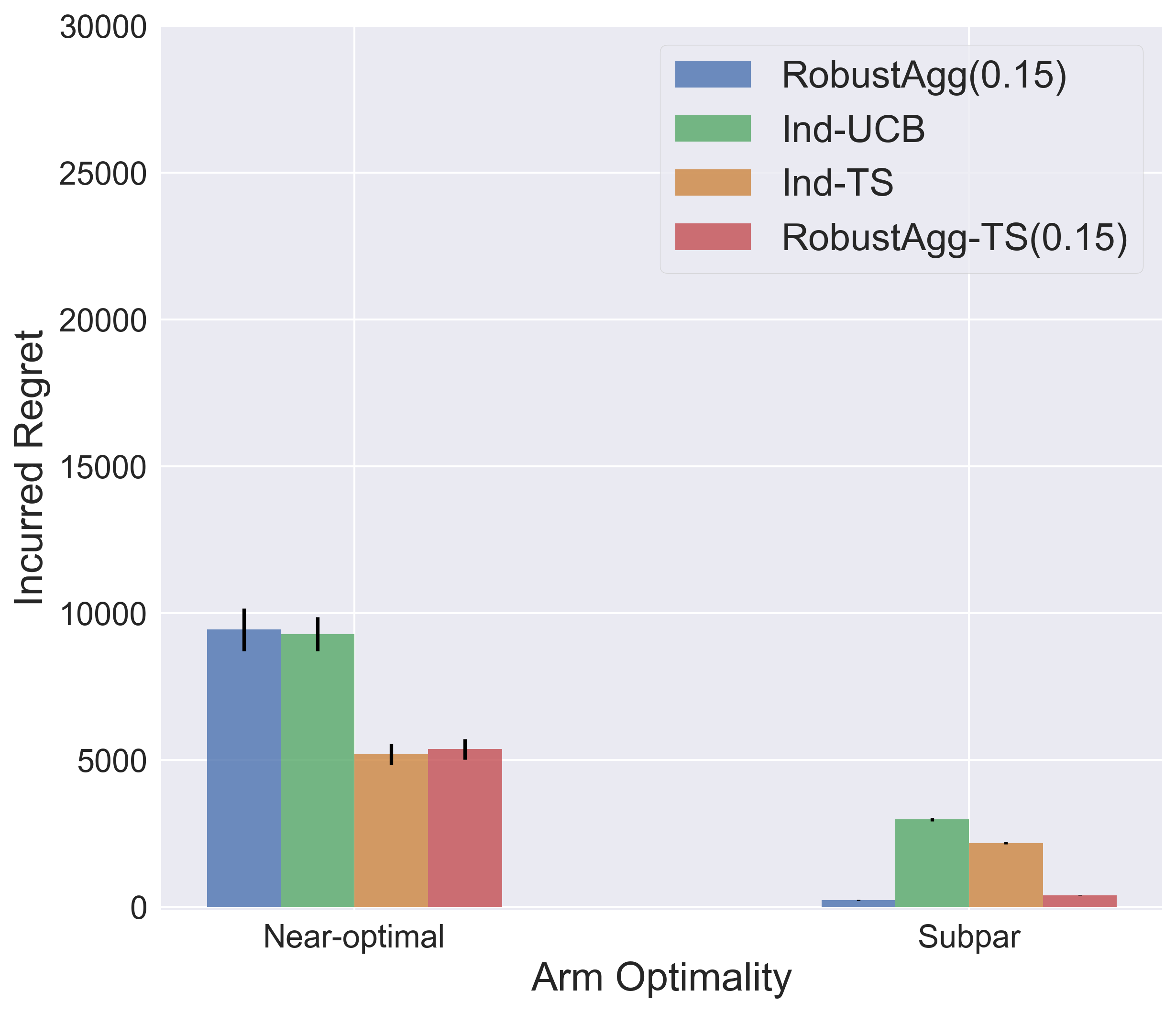}
        \caption{$|\Ical_{5\epsilon}| = 6$}
        \label{figure:exp3_6}
    \end{subfigure}
    \begin{subfigure}{0.32\textwidth}
        \centering
        \includegraphics[height=0.85\linewidth]{plots/exp1__regret_perc_ical=5_50000x30.png}
        \caption{$|\Ical_{5\epsilon}| = 5$}
        \label{figure:exp3_5}
    \end{subfigure}
    \begin{subfigure}{0.32\textwidth}
        \centering
        \includegraphics[height=0.85\linewidth]{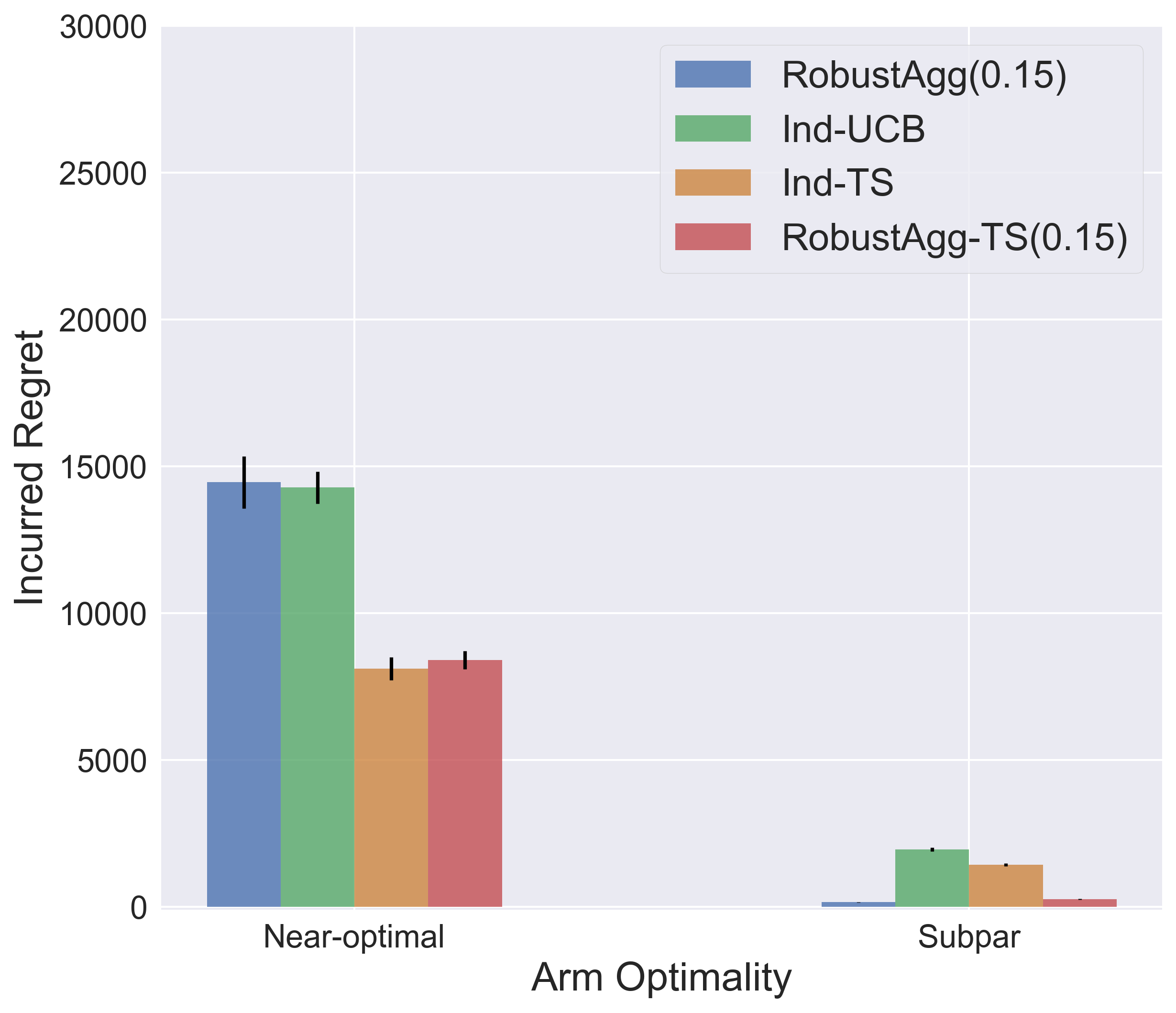}
        \caption{$|\Ical_{5\epsilon}| = 4$}
        \label{figure:exp3_4}
    \end{subfigure}
        \begin{subfigure}{.32\textwidth}
        \centering
        \includegraphics[height=0.85\linewidth]{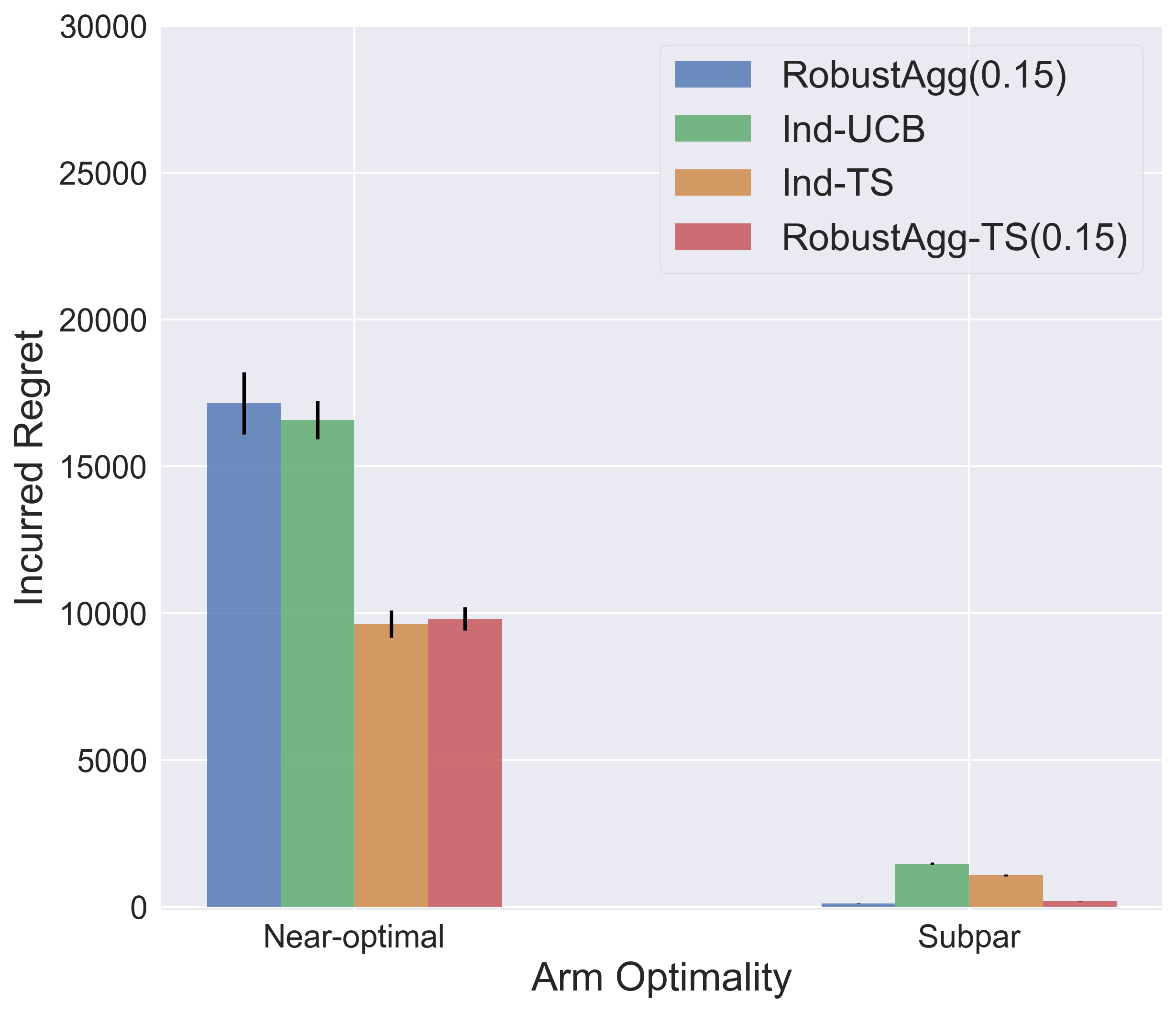}
        \caption{$|\Ical_{5\epsilon}| = 3$}
        \label{figure:exp3_3}
    \end{subfigure}
    \begin{subfigure}{0.32\textwidth}
        \centering
        \includegraphics[height=0.85\linewidth]{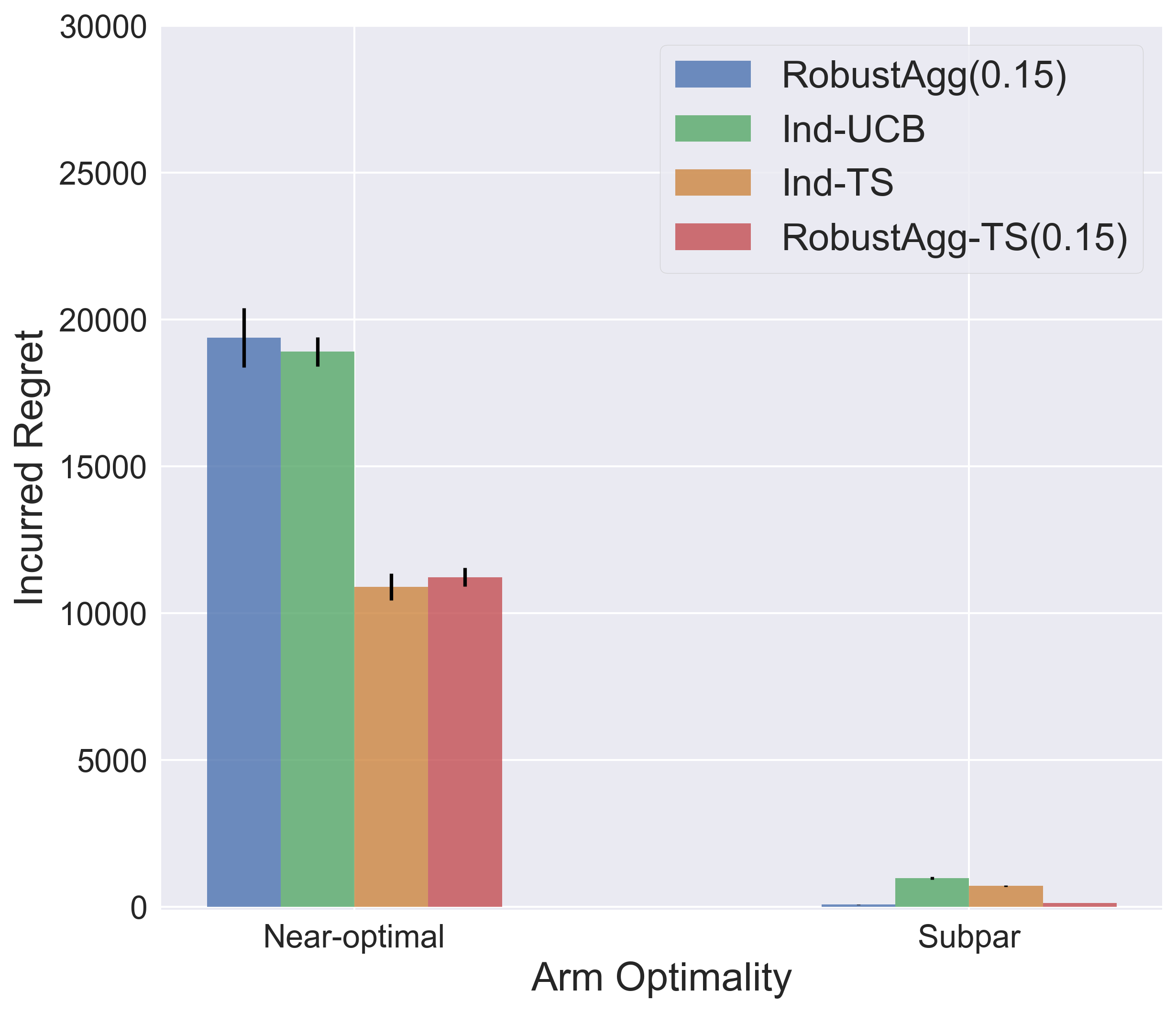}
        \caption{$|\Ical_{5\epsilon}| = 2$}
        \label{figure:exp3_2}
    \end{subfigure}
    \begin{subfigure}{0.32\textwidth}
        \centering
        \includegraphics[height=0.85\linewidth]{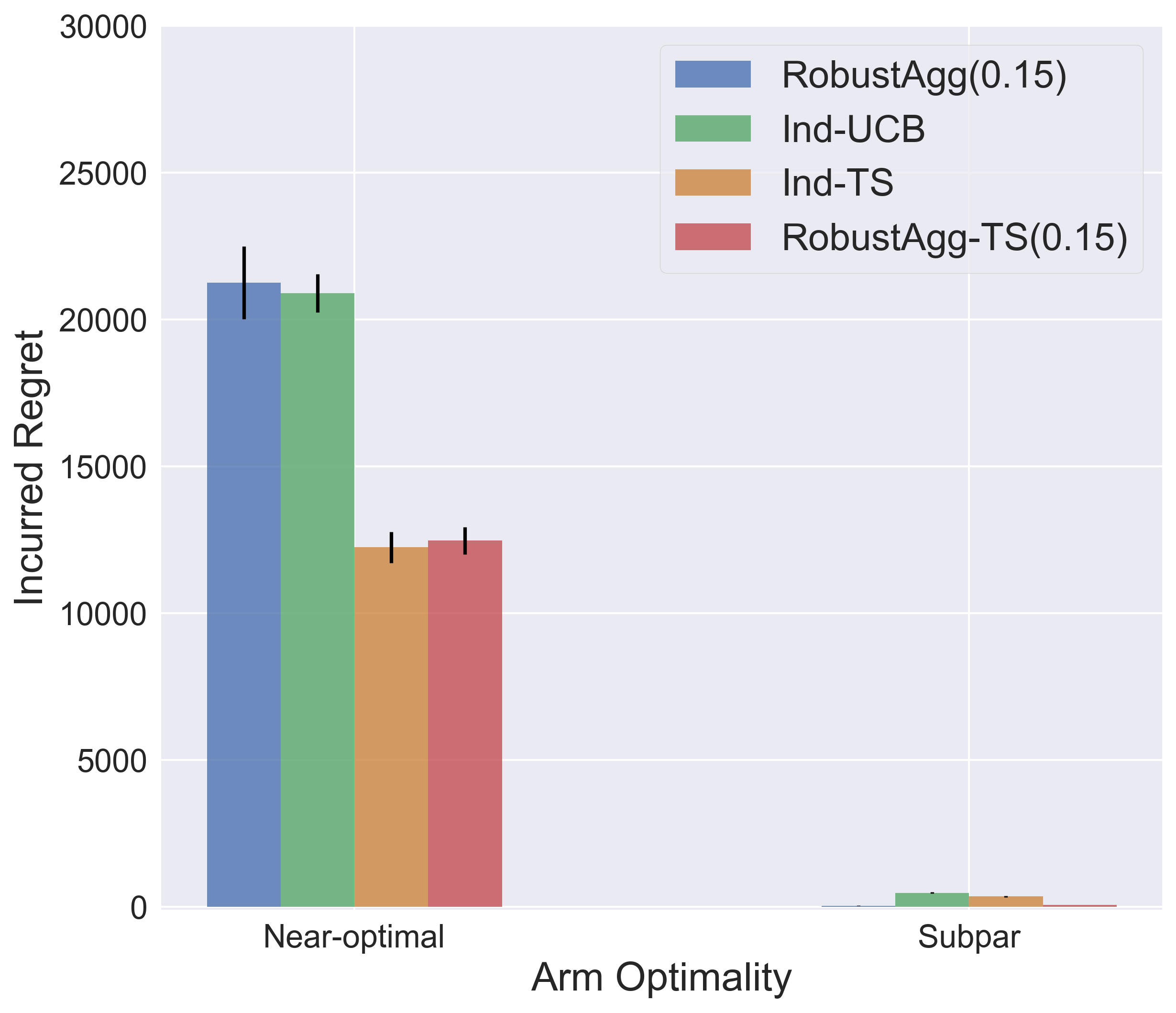}
        \caption{$|\Ical_{5\epsilon}| = 1$}
        \label{figure:exp3_1}
    \end{subfigure}
    \begin{subfigure}{0.32\textwidth}
        \centering
        \includegraphics[height=0.85\linewidth]{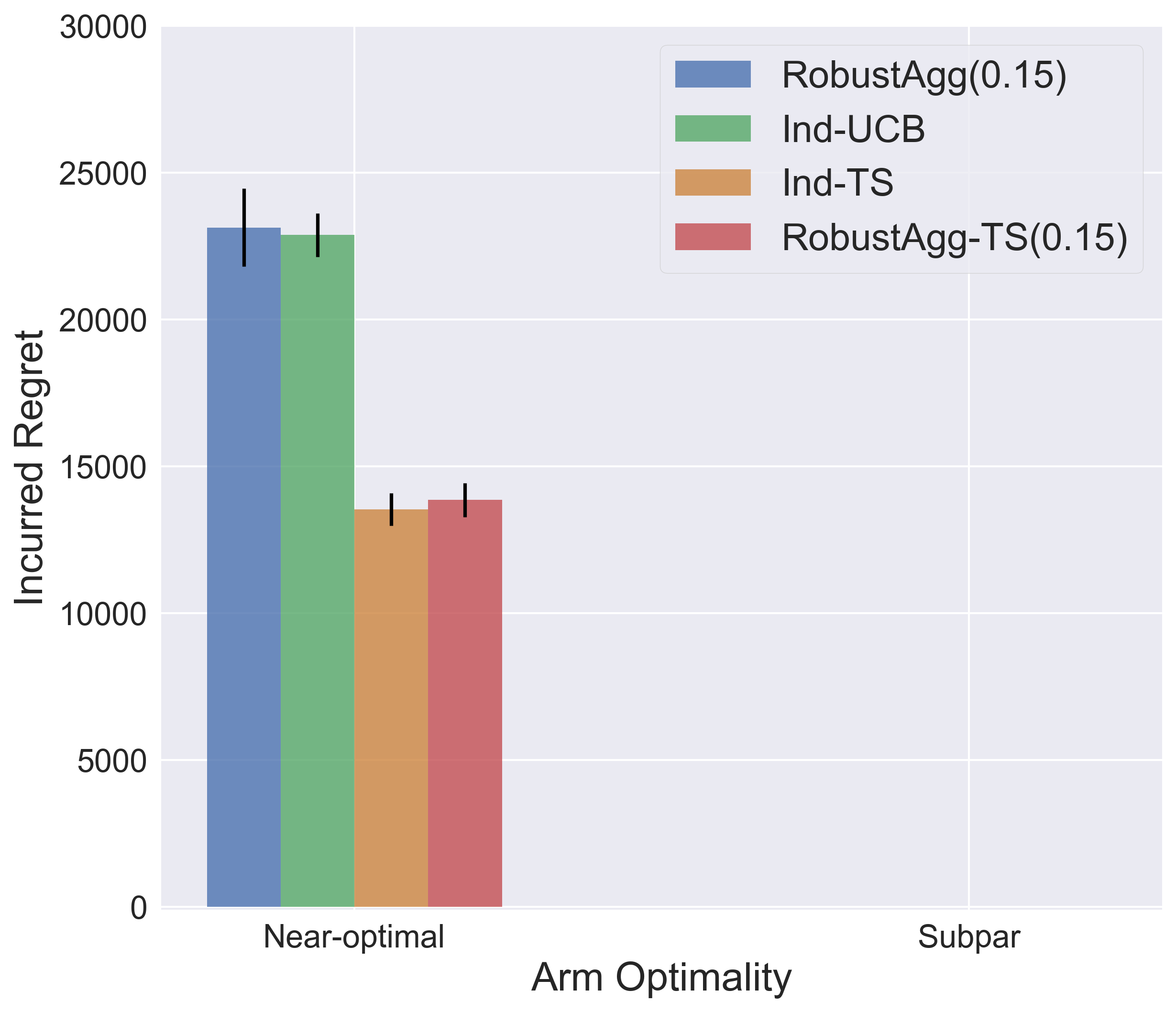}
        \caption{$|\Ical_{5\epsilon}| = 0$}
        \label{figure:exp3_0}
    \end{subfigure}
    \caption{Compares the cumulative collective regret incurred by arm optimality for the 4 algorithms in $T = 50,000$ rounds.}
    \label{figure:regret_perc_allplots}
\end{figure}

\begin{figure}[htbp]
    \centering
    \begin{subfigure}{0.32\textwidth}
        \centering
        \includegraphics[height=0.85\linewidth]{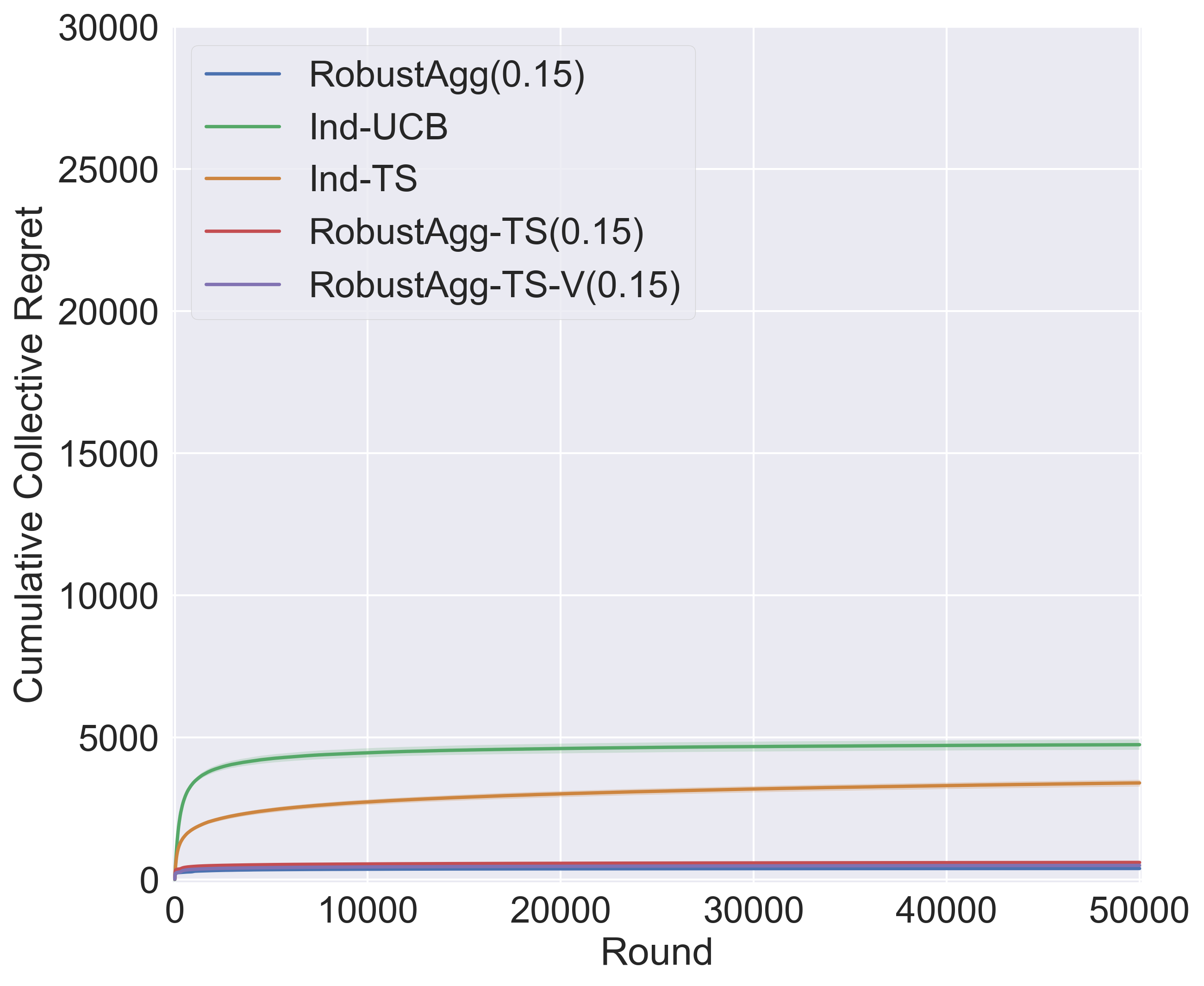}
        \caption{$|\Ical_{5\epsilon}| = 9$}
        \label{figure:exp4_9}
    \end{subfigure}
    \begin{subfigure}{0.32\textwidth}
        \centering
        \includegraphics[height=0.85\linewidth]{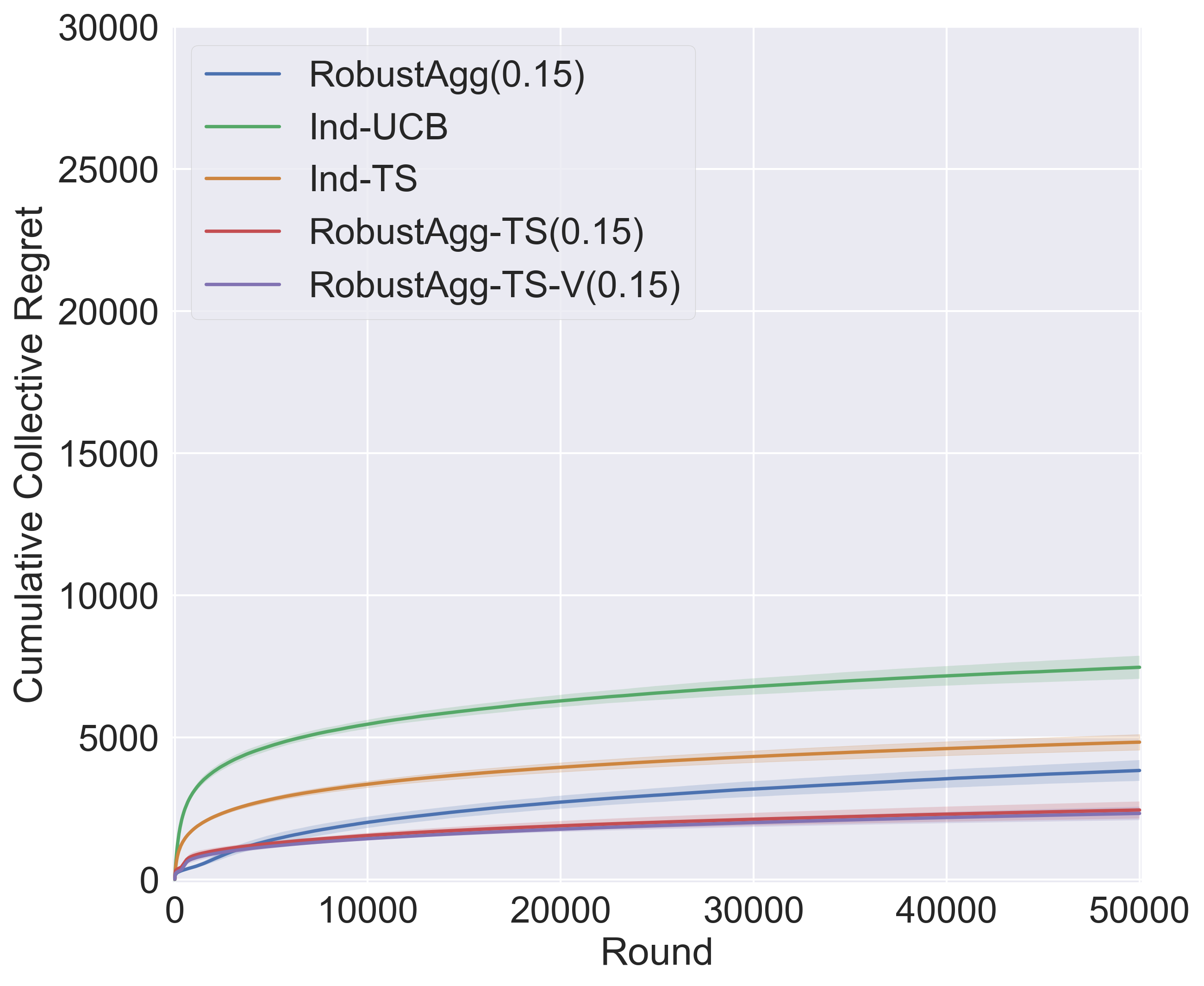}
        \caption{$|\Ical_{5\epsilon}| = 8$}
        \label{figure:exp4_8}
    \end{subfigure}
    \begin{subfigure}{0.32\textwidth}
        \centering
        \includegraphics[height=0.85\linewidth]{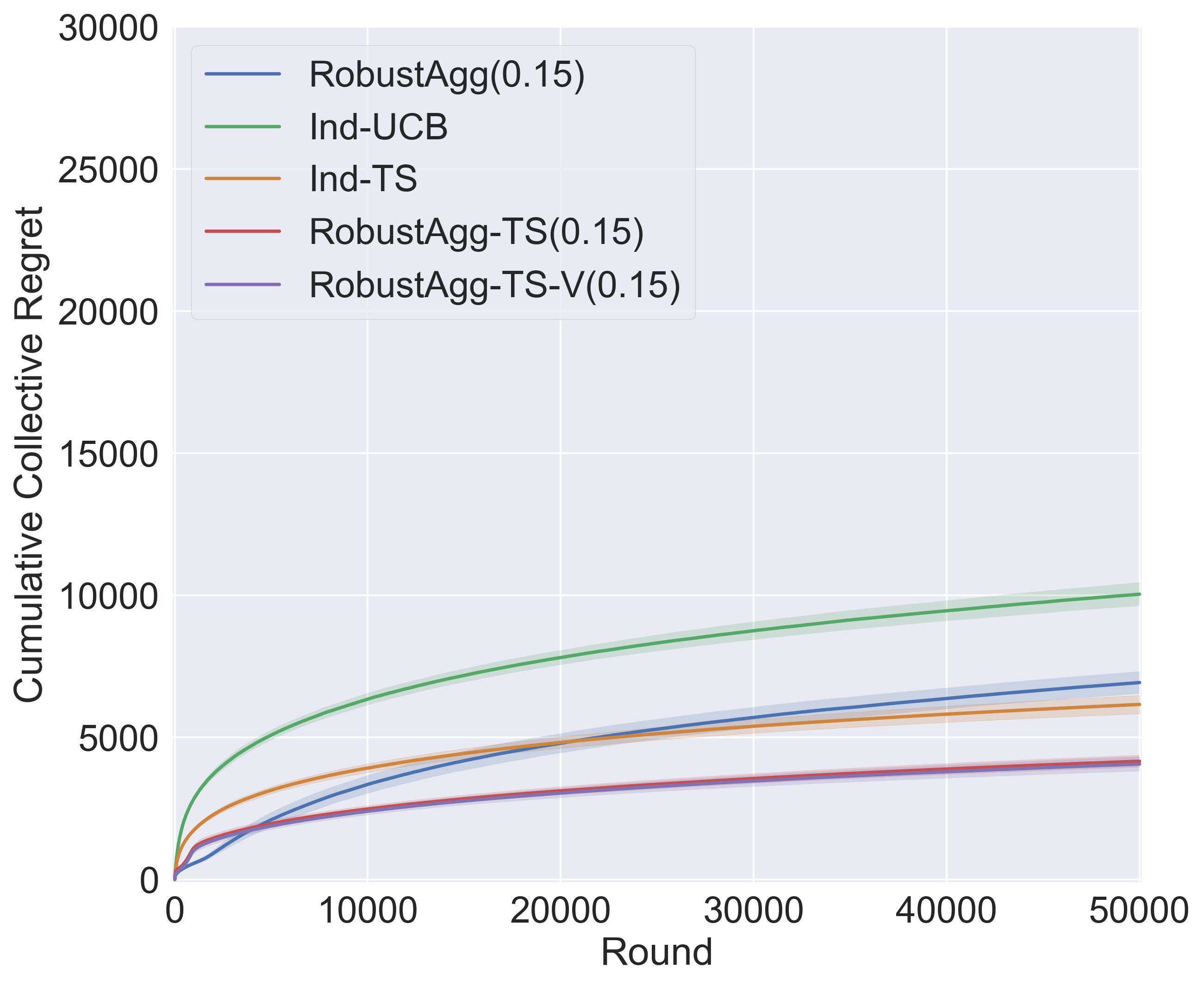}
        \caption{$|\Ical_{5\epsilon}| = 7$}
        \label{figure:exp4_7}
    \end{subfigure}
        \begin{subfigure}{.32\textwidth}
        \centering
        \includegraphics[height=0.85\linewidth]{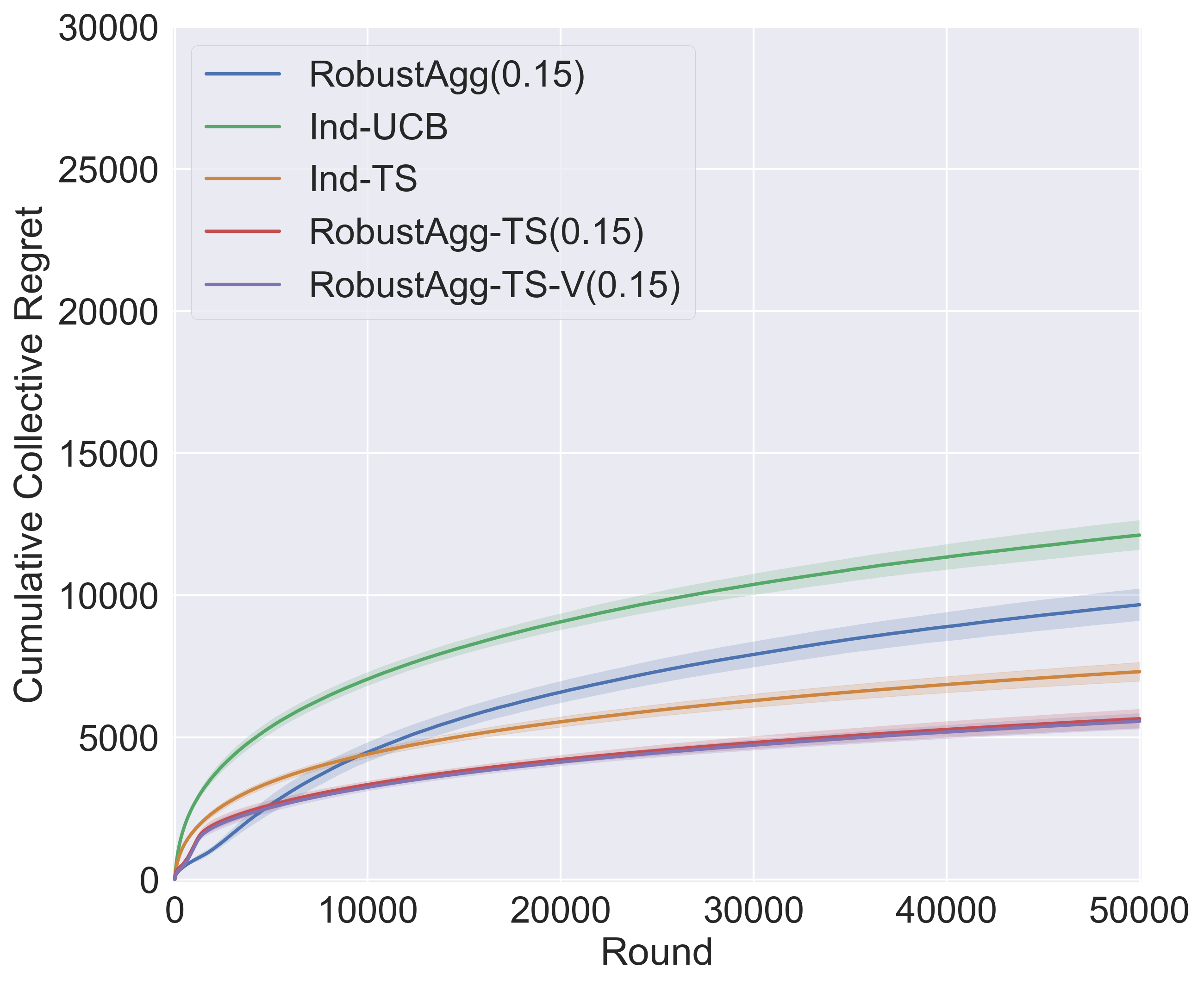}
        \caption{$|\Ical_{5\epsilon}| = 6$}
        \label{figure:exp4_6}
    \end{subfigure}
    \begin{subfigure}{0.32\textwidth}
        \centering
        \includegraphics[height=0.85\linewidth]{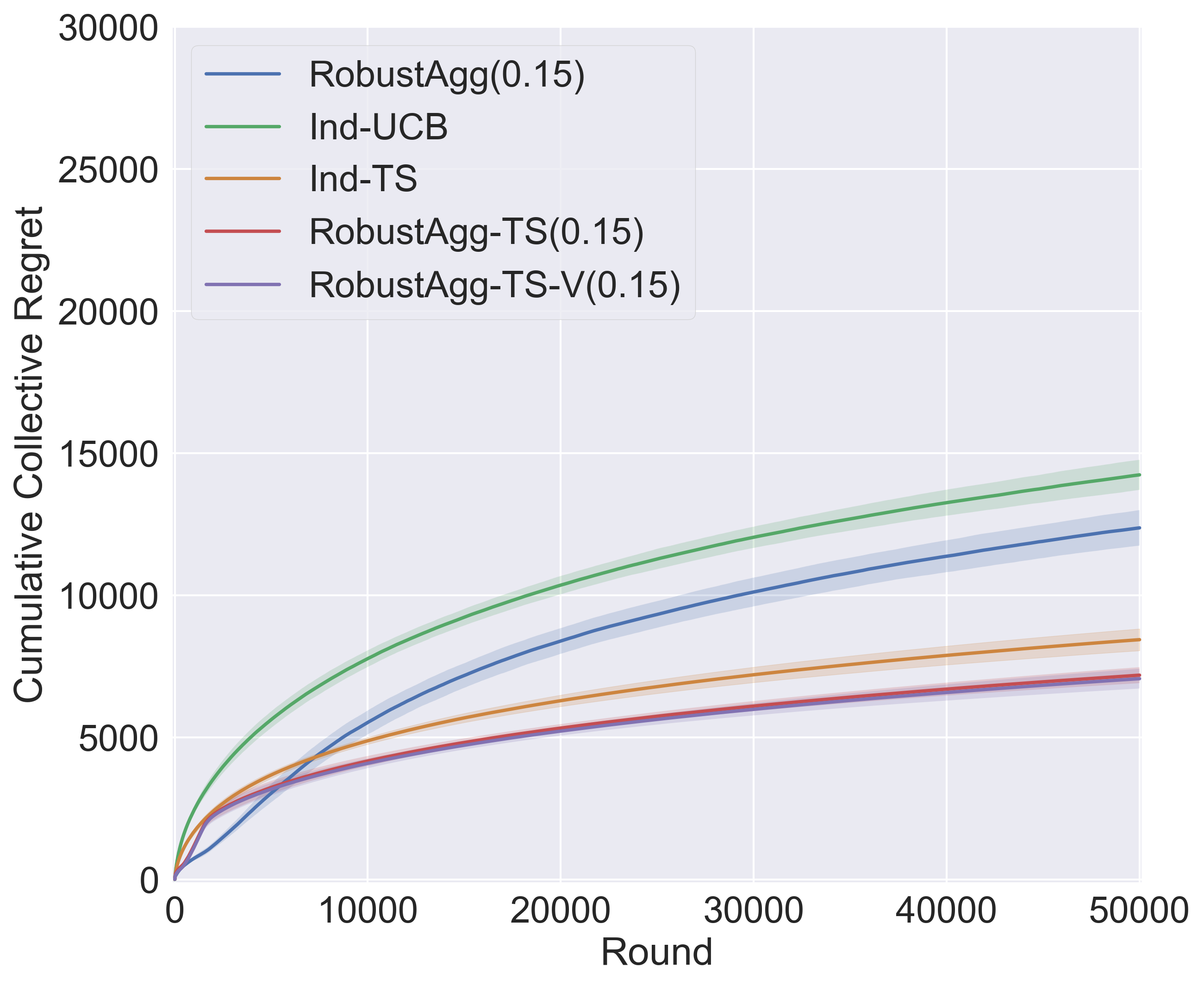}
        \caption{$|\Ical_{5\epsilon}| = 5$}
        \label{figure:exp4_5}
    \end{subfigure}
    \begin{subfigure}{0.32\textwidth}
        \centering
        \includegraphics[height=0.85\linewidth]{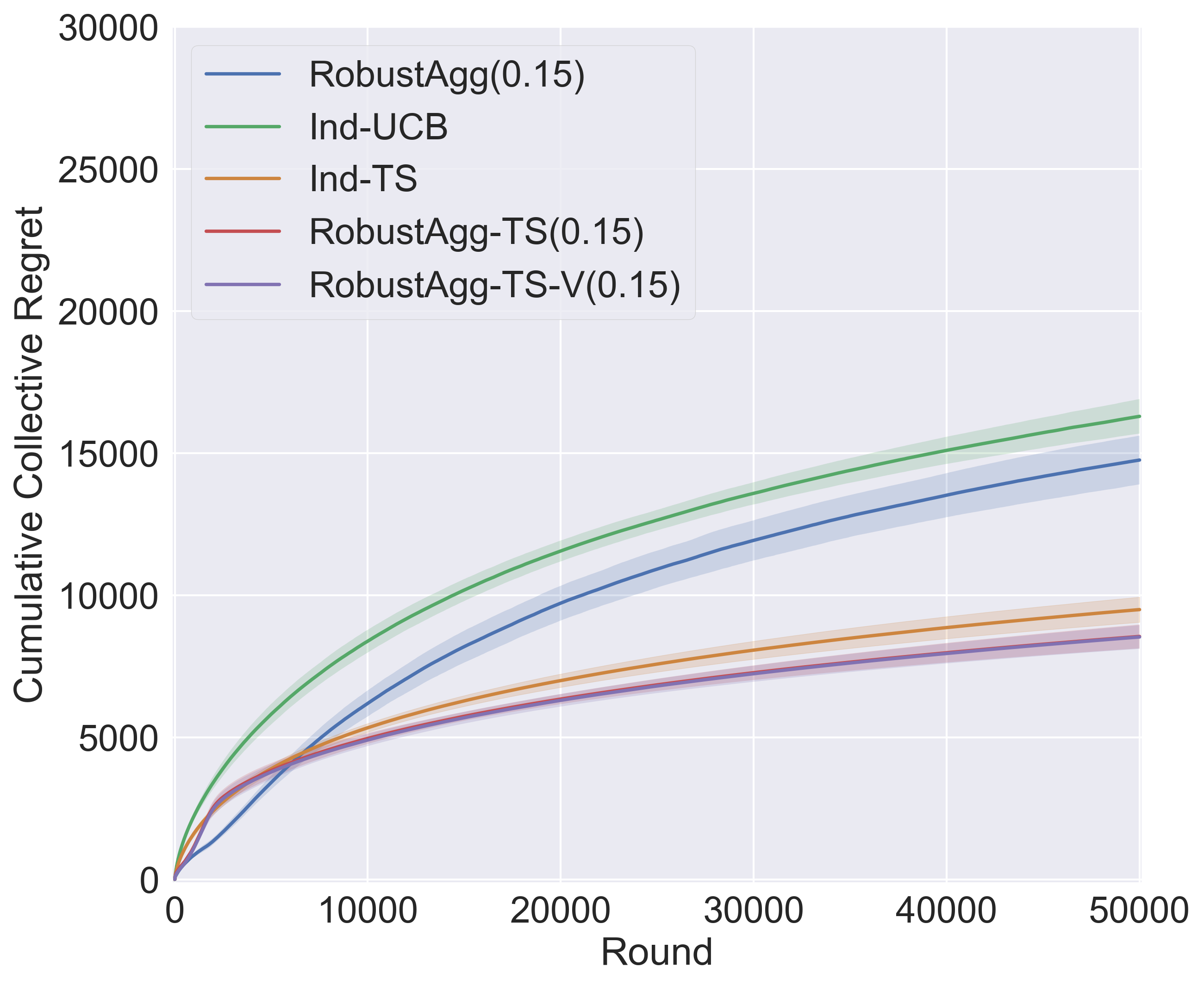}
        \caption{$|\Ical_{5\epsilon}| = 4$}
        \label{figure:exp4_4}
    \end{subfigure}
        \begin{subfigure}{.32\textwidth}
        \centering
        \includegraphics[height=0.85\linewidth]{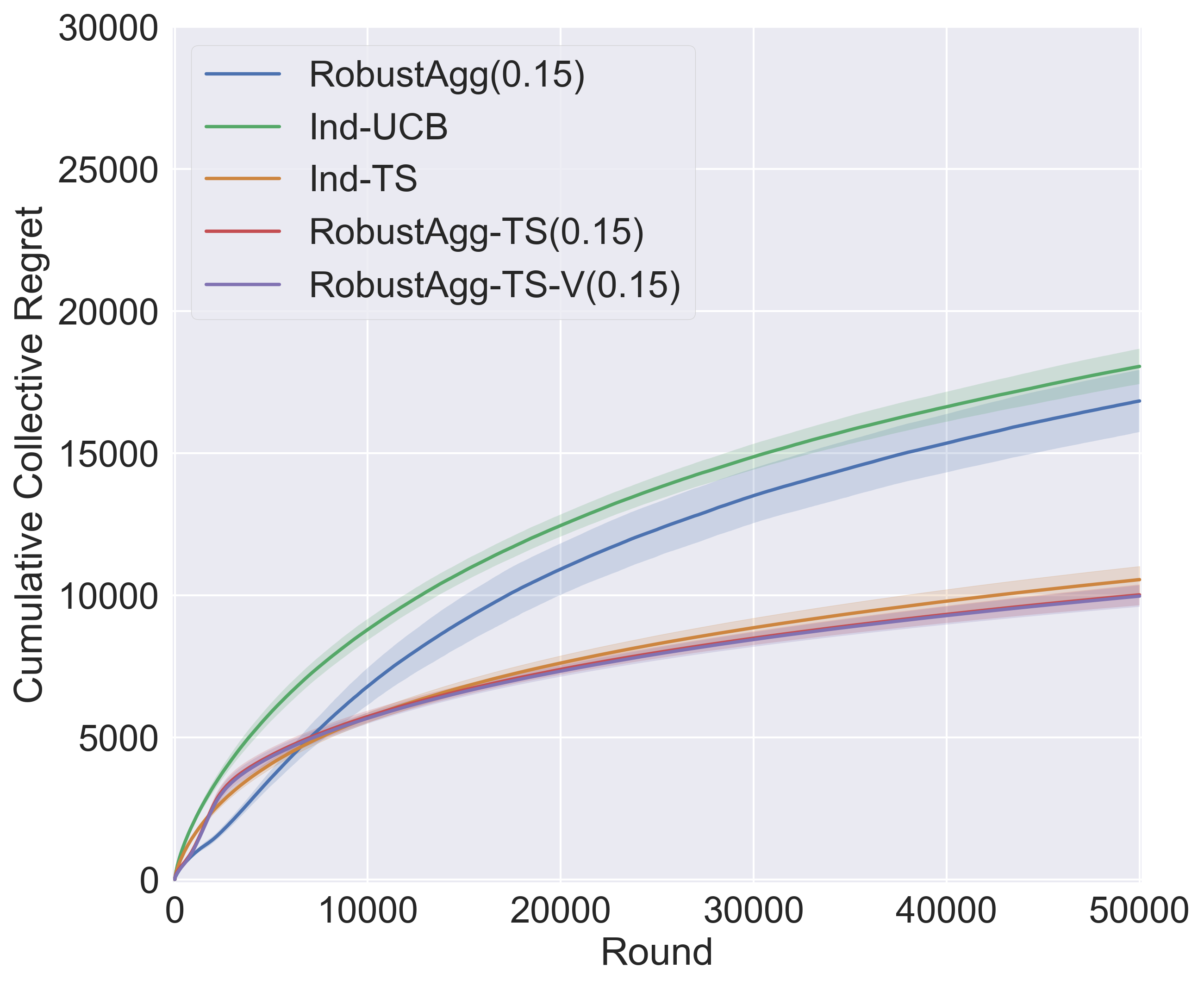}
        \caption{$|\Ical_{5\epsilon}| = 3$}
        \label{figure:exp4_3}
    \end{subfigure}
    \begin{subfigure}{0.32\textwidth}
        \centering
        \includegraphics[height=0.85\linewidth]{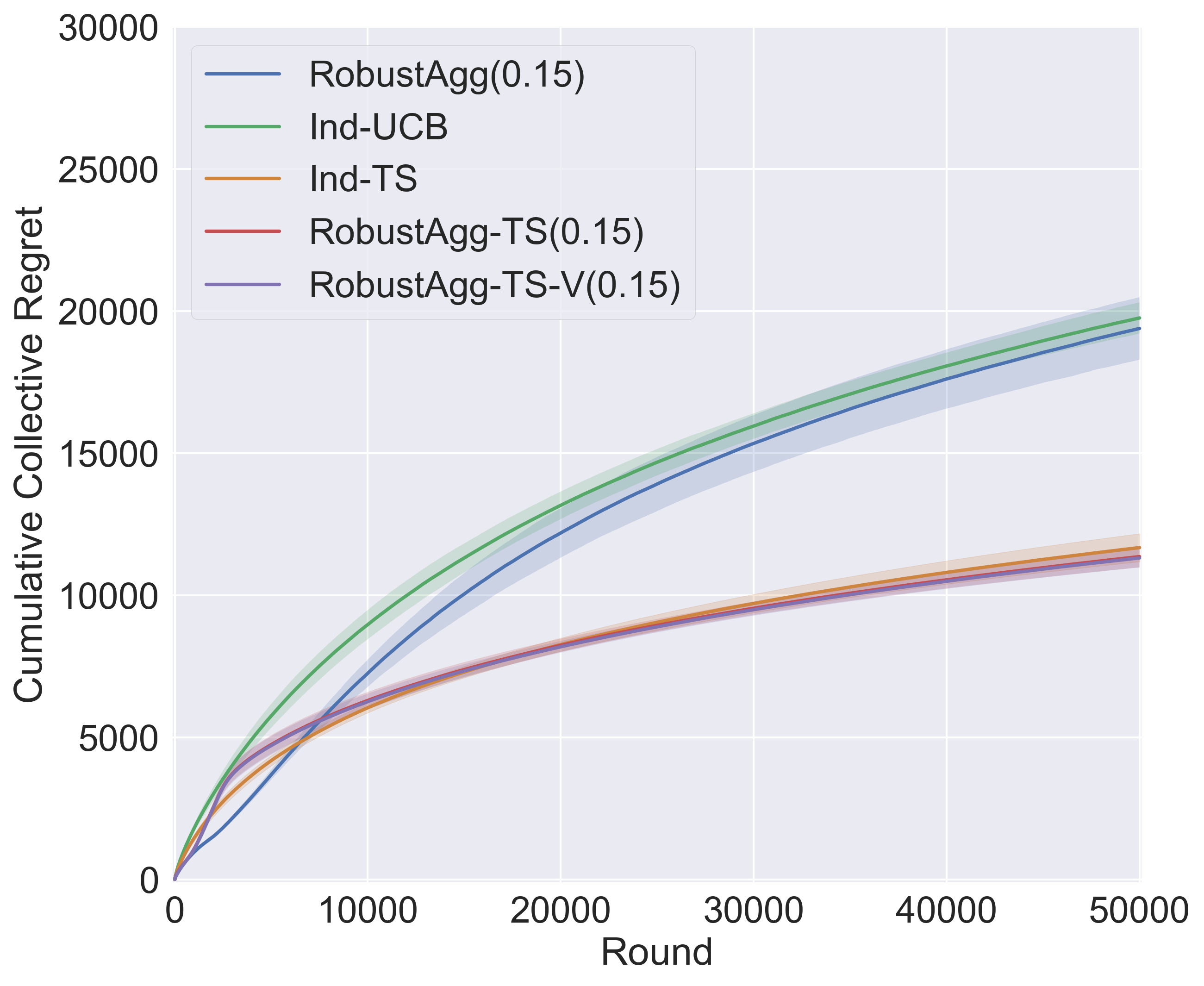}
        \caption{$|\Ical_{5\epsilon}| = 2$}
        \label{figure:exp4_2}
    \end{subfigure}
    \begin{subfigure}{0.32\textwidth}
        \centering
        \includegraphics[height=0.85\linewidth]{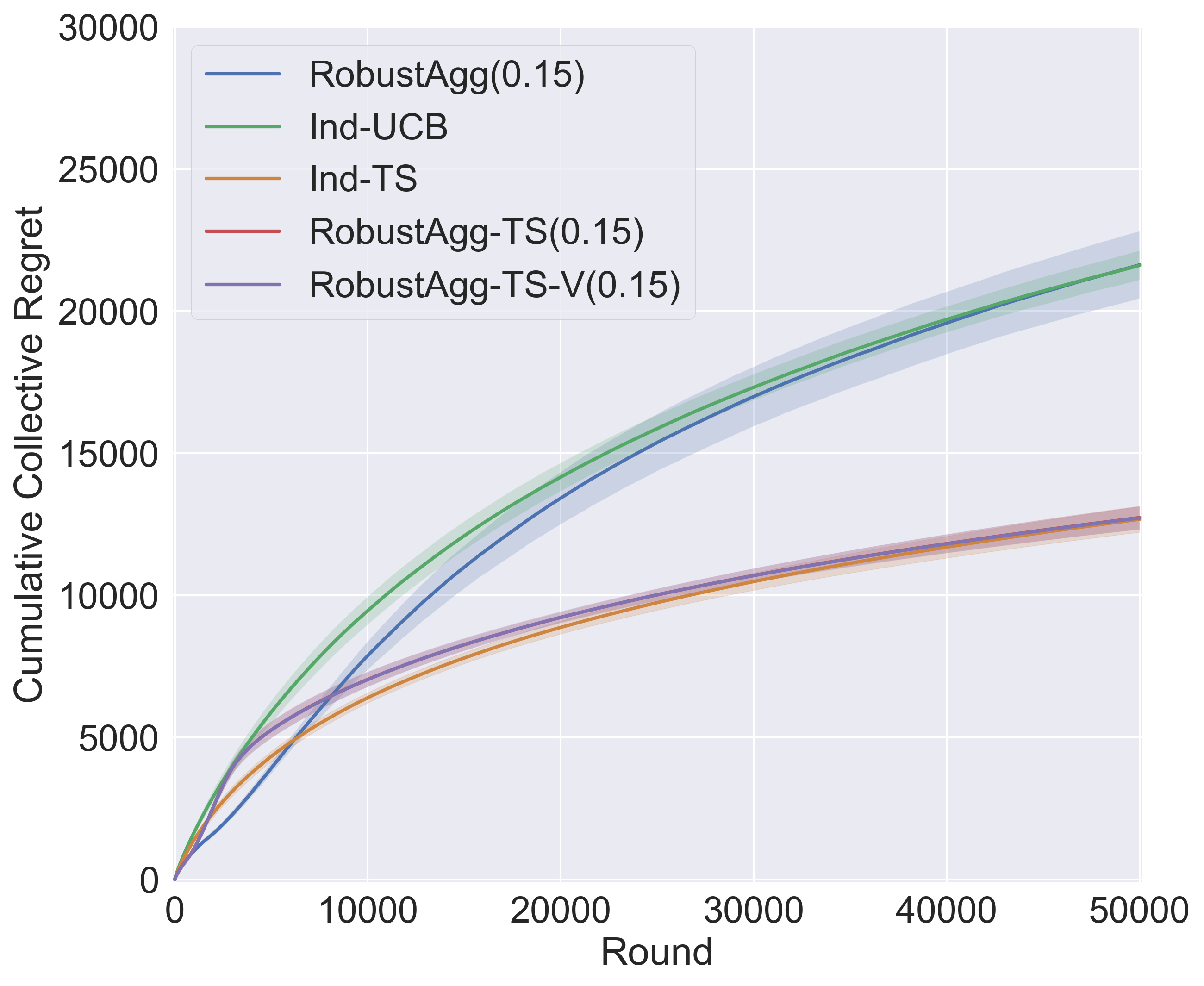}
        \caption{$|\Ical_{5\epsilon}| = 1$}
        \label{figure:exp4_1}
    \end{subfigure}
    \begin{subfigure}{0.32\textwidth}
        \centering
        \includegraphics[height=0.85\linewidth]{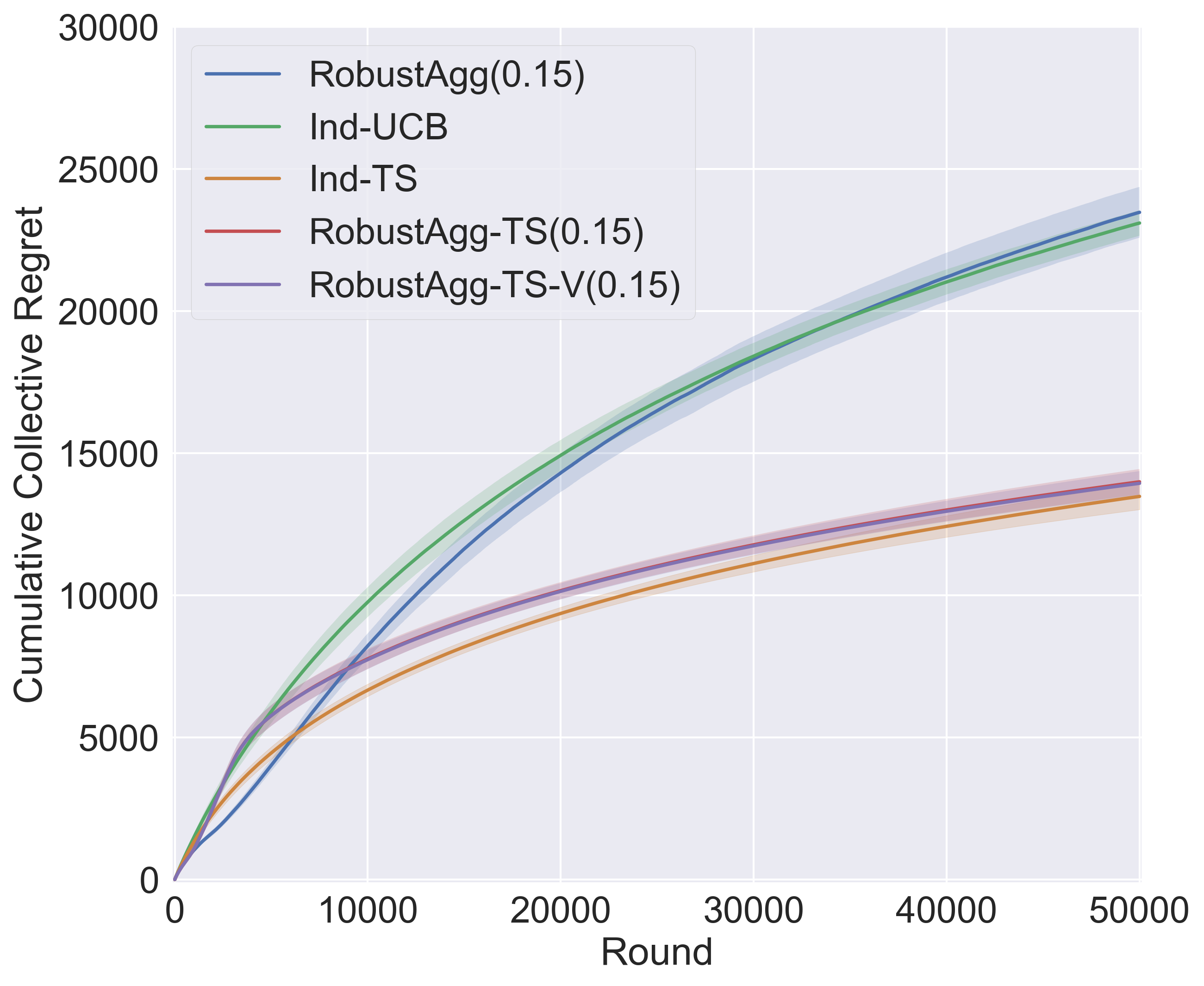}
        \caption{$|\Ical_{5\epsilon}| = 0$}
        \label{figure:exp4_0}
    \end{subfigure}
    \caption{Compares the cumulative collective regret of the 5 algorithms over a horizon of $T = 50,000$ rounds.}
    \label{figure:v_allplots}
\end{figure}

\end{document}